\documentclass[a4paper,11pt,oneside,reqno]{amsart}
\usepackage[utf8]{inputenc}
\usepackage[colorlinks]{hyperref}
\usepackage{amssymb,amsmath,amsthm,color, txfonts}
\usepackage{comment}
\usepackage{geometry}
\usepackage{xcolor}
\usepackage{graphicx}
\usepackage{makecell}
\usepackage{graphicx}
\usepackage{caption}
\usepackage{subcaption}
\usepackage[rightcaption]{sidecap}
\usepackage{float}
\usepackage{array}
\usepackage{tabularx}
\usepackage{bm}
\usepackage[foot]{amsaddr}

\newtheorem{theorem}{Theorem}[section]
\newtheorem{lemma}[theorem]{Lemma}
\newtheorem{proposition}[theorem]{Proposition}

\newtheorem{corollary}[theorem]{Corollary}
\usepackage{hyperref}
\newtheorem{rmk}[theorem]{Remark}
\newtheorem{obs}[theorem]{Observation}

\newtheorem{gp}{Suggested General Principle}
\theoremstyle{remark}
\newtheorem{ex}[theorem]{Example}

\newcommand{\ToVerTwo}[1]{}

\newcommand{\FromRan}[1]{\footnote{\textcolor{orange}{RT: #1}}}

\newcommand{\Tr}{{\rm Tr}}
\newcommand{\HD}{{\rm HD}}
\newcommand{\bscs}{{\rm bscs}}

\newcommand{\rank}{{\rm rank}}
\newcommand{\brank}{{\rm rank}_\square}
\newcommand{\argbrank}{{\rm argrank}_\square}
\newcommand{\mmin}{\textstyle\min_2}
\newif\iffullversion

\def\RT#1{\textcolor{orange}{[RT: #1]}}

\newcommand{\bbZ}{{\mathbb{Z}}}

\newcommand{\CF}{{\mathcal{CFC}}}
\newcommand{\bbC}{{\mathbb{C}}}
\newcommand{\bbR}{{\mathbb{R}}}
\newcommand{\bbN}{{\mathbb{N}}}
\newcommand{\cD}{\mathcal{D}}
\newcommand{\cS}{\mathcal{S}}
\newcommand{\cW}{\mathcal{W}}
\newcommand{\cB}{\mathcal{B}}

\begin{document}
\title{Learning words in groups: fusion algebras, tensor ranks and Grokking}
\author{Maor Shutman}
\email{maorshut@protonmail.com}
\author{Oren Louidor$^{1}$}
\email{oren.louidor@gmail.com}
\thanks{The work of O.L. is supported by ISF grant nos.~2870/21 and~3782/25
, and by the BSF award 2018330.}
\address{$^{1}$ Technion - Israel Institute of Technology}
\author{Ran Tessler$^{2}$}
\address{$^{2}$ Weizmann Institute}
\thanks{The work of R.T. is supported by ISF grant no.~1729/23.}
\email{ran.tessler@weizmann.ac.il}

\begin{abstract}
In this work, we demonstrate that a simple two-layer neural network with standard activation functions can learn an arbitrary word operation in any finite group, provided sufficient width is available and exhibits grokking while doing so. To explain the mechanism by which this is achieved, we reframe the problem as that of learning a particular $3$-tensor, which we show is typically of low rank. A key insight is that low-rank implementations of this tensor can be obtained by decomposing it along triplets of basic self-conjugate representations of the group and leveraging the fusion structure to rule out many components. Focusing on a phenomenologically similar but more tractable surrogate model, we show that the network is able to find such low-rank implementations (or approximations thereof), thereby using limited width to approximate the word-tensor in a generalizable way. In the case of the simple multiplication word, we further elucidate the form of these low-rank implementations, showing that the network effectively implements efficient matrix multiplication in the sense of Strassen. Our work also sheds light on the mechanism by which a network reaches such a solution under gradient descent. 
\end{abstract}

\maketitle

\section{Introduction, contribution}
\subsection{Background and motivation}
Studying the means by which statistical models learn and represent operations on finite sets is receiving quite a lot of attention these days. By an operation we refer to a bi-variate function $f: G \times G \to G$, where $G$ is a general finite set. While there are many real-world examples which fit into this framework, considerable effort is centered on studying the more tractable case when there is an explicit mathematical formulation which governs the result of the operation. Questions of interest here focus, as usual, on expressibility and interpretability, generalization and phenomenological aspects of the dynamics. As the literature shows, such algorithmic setups have proved to be a fruitful soil for reconstructing real-world phenomena, while keeping the overall complexity and analytic tractability of the problem low and high respectively.

Perhaps the most natural of such operations, at least from a mathematical point of view, is the multiplication operation of a mathematical group. The simplest case of a cyclic group of order $p$, a canonical representative of which is $\bbZ_p := \{0,\dots, p-1\}$ with the operation being addition modulo $p$, was studied by Power et al in~\cite{power2022grokking}. This worked showed that a simple
decoder-only transformer architecture is able to represent and, moreover, learn such an operation based on a a fraction of all $p^2$ examples. Interestingly, the authors observed that during training, both train and test accuracy transitioned very sharply from a trivial level to $100\%$, with the transition in the test-set lagging behind and occurring well beyond the interpolation threshold. This phenomenon, which the authors termed ``Grokking'' was later found to occur in many other architectures and learning tasks (see related work section).

In an effort to understand this phenomenon better, Gromov~\cite{gromov2023grokking} studied the simpler setup of a standard Two Layer Perceptron (henceforth TLP, i.e. an MLP with one hidden layer) and demonstrated that the network still exhibits Grokking given the same task of addition modulo $p$. Moreover, he proposed a  ``solution-ansatz'' for the weights of the network, composed of Fourier basis vectors, and showed that this solution achieves asymptotically zero test loss and perfect accuracy. Under his ansatz, the rows of the weight matrices are multiples of real-valued Fourier basis vectors whose frequency is the same for matching rows across all weight matrices. His work then verified empirically that the network converges to this solution under standard first order optimization algorithms, given a partial subset of all examples. Convincing evidence of the convergence to suitable variants of this solution have been presented for a simple transformer-based architectures as well~\cite{nanda2023progress} at around the same time.

The case of a general group was studied shortly after in~\cite{chughtai2023toy}. The authors showed that a similar architecture as that in~\cite{gromov2023grokking} is able to learn and generalize many other groups, including non-cyclic and non-abelian ones. Moreover, they proposed and empirically verified, a generalization of the Fourier-based solution for the general case, using (real versions of) irreducible representations, which are the analogs of the Fourier vectors from the cyclic case. Lastly, they showed that the system exhibits ``Grokking'' in the same sense as before.

\subsection{Contribution}

In this work we go a step further and generalize the class of bi-variate operations to that of \emph{group words}. Given a group $G$ with a multiplication operator $\cdot$, a word $w$ is a non-empty string of finite length over the literals $a,b,a^{-1},b^{-1}$, which represents an expression involving two arguments $a$ and $b$. For example, $w=aba^{-1}$ represents the expression $a \cdot b \cdot a^{-1}$. In what follows we identify a word with the bi-variate operation defined by the expression it represents, so that the word in the last example is also the operation $w(a,b) := a \cdot b \cdot a^{-1}$. This is clearly a natural extension of the usual group multiplication.

Using the same simple TLP model used by Gromov in~\cite{gromov2023grokking}, we first verify that the network is able to learn and generalize arbitrary words and groups and that grokking is still exhibited as before. The affirmative results are summarized in Figure~\ref{f:A1}. The required fraction of examples and how pronounced the grokking turns out to be, depends on the 
underlying group and word, as well as on the width of the model. 

Next, we turn to study this problem theoretically. Our analysis relies on {\em representation theory}, as in~\cite{chughtai2023toy}. However, we also appeal to two new mathematical notions: the \emph{fusion algebra} associated with the group $G$, and the {\em rank of the tensor} representing the learning task. We start by recasting the learning task as that of realizing a $3$-tensor in $(\bbR^{|G|})^{\otimes 3}$, with the first two components being the one hot encodings of the operation arguments and the last being the one hot encoding of the output element. We call such tensor a word tensor. We then use irreducible representations, or more precisely their real-valued analogs, basic self-conjugate (bsc) representations, to find low rank (or sparse) representations of this tensor. Existence of such sparse representations should be the key reason for the ability of the network to represent and find high accuracy solutions which generalize well.

To this end, we project the word-tensor onto tensor-products of the sub-spaces corresponding to triplets of bscs and use fusion rules to rule out triplets where the projection is trivial. We find that often there are relatively few such triplets in the ``bsc-support'' of the word-tensor. We then use this decomposition to find classes of low-rank representations which implement the word tensor. By definition, each such class gives a bound on the rank of the word-tensor, and thus an optimal bound can be obtained by solving the combinatorial optimization problem of finding the minimum among them. Considering several examples, we observe that the rank of the word tensor is often much lower than the a-priori upper bound of $|G|^2$.

Next, we check whether the network is indeed able to find such low rank representations. To make the connection with the theory more straightforward, we replace the TLP model with a variant, which we call the Hadamard, or HD, model. The activation function applied to the neurons in the hidden layer in the case of TLP is replaced by taking products of pairs of matched neurons in the latter. We explain why this model is likely to capture the phenomenology of the original model, and show that it is comparable to one considered by  Gromov~\cite{gromov2023grokking}.
The advantage of working with this model is that at width $m$, it can implement any $3$-tensor of rank at most $m$ in a straightforward way, which can also be read directly from its weight configuration. We find that under standard first order optimization schemes, the HD network is able to find $100\%$ accuracy solutions given the full dataset of many groups and words. 

To study the terminal weight configuration, we project the rows of the weight matrices onto the subspaces associated with the bscs of the group. We observe that, in alignment with the theoretical study (and in generalization of the case of the simple group multiplication), the model indeed finds a low rank implementation of the word-tensor (or an approximation thereof, if the width of the model is too small) by representing it as sums of $3$-tensors whose bsc-support is relatively small. Remarkably, in many cases the terminal weight configuration of the model is one of the local minima of the combinatorial optimization problem mentioned above (again, sometimes only an approximation thereof).
Upon verifying that the outcome remains qualitatively the same under a partial dataset and the original TLP model, we conclude that word operations are learned through a representation and discovery of a low rank version of the required tensor and that this representation relies on a decomposition of the tensor along the bscs of the group. 

Next we turn to apply our theory to the case of the group multiplication studied in previous works. In this case, the bsc-support of the word tensor is composed of triplets of bscs where all the components are the same, and our general theory gives rise to a class of low rank representations, which coincides with the ones found by Nanda~\cite{chughtai2023toy} and Gromov~\cite{gromov2023grokking}. We then derive theoretical bounds on the rank of the word tensor, by bounding the tensor rank of the components in its bsc-support. The latter involve the (generally unknown) tensor rank of matrix multiplication as an implicit constant. Bounds on the latter are known since the work of Strassen~\cite{strassen1969gaussian} (see also \cite{ottaviani2020tensor} for a more modern survey). The rank of word tensor directly relates to the width of the model, required to fulfill the learning task.

Turning to experiments, we again switch first to the HD model where the correspondence with the theory is more direct. The class of representations mentioned above is implemented by this model via, what we call, mono-bsc-aligned weight configurations. In such weight configurations corresponding rows of the weight matrices belong to the subspace of the same unique bsc. We show that the space of such weight configurations is stable under a step of gradient based optimization algorithms. We use this to argue that, starting from an initial weight configuration, the dynamics effectively decouples, with each subset of rows of the weight-matrices, corresponding to the same bsc, evolving independently as a stand alone model, minimizing the corresponding projection of the total loss.

This observation has two consequences. First, it allows us to study each bsc-component of the low-rank representation of the word tensor independently. Remarkably, again, we find that our rank bounds are often met (at least in the dimension where this can be verified), showing that the network discovers the minimal-rank matrix multiplication tensor on its way to finding a low rank solution to the full problem. Second, the observation reveals aspects of the mechanism by which the model reaches its terminal weight configuration. Starting from a random initialization, as the model evolves, the rows of the weight matrices partition in tandem to subsets which correspond to different bscs. Then for each such subset the model effectively evolves independently, by minimizing the corresponding bsc-loss, using as many rows as assigned under this partition. As in the general case, we show that this also happens with a strict subset of the dataset and under the TLP model as well. 
	
\section{Learning task, model and preliminary empirical study}
\subsection{Setup}
\label{s:10.1}
\subsubsection{Notation}
Henceforth we shall index arrays which correspond to the elements of the group $G$ via the group elements themselves, thus we may write $x \in \bbR^G$, in place of $x \in \bbR^{|G|}$ and denote by $(x_g : g \in G)$ the elements of such vector. The only place where we must resort to integer indices is when the model is implemented. In this case, we shall fix an arbitrary ordering of the elements of $G$ and use the place in this ordering as the bijection between an element in $G$ and a number in $\{0,\dots, |G|-1\}$, which will be consistently used throughout the implementation of the model.

For $g \in G$, we shall write $1_g$ for the one-hot-encoding of $g$, namely the vector in $\bbR^G$ satisfying $(1_g)_h = \delta_{g=h}$ for all $h \in G$, where $\delta_{x,y}$ is the usual Kronecker delta function. Henceforth, all vectors are taken to be column vectors by default. Given two multi-dimensional arrays $A$ and $B$, we shall write $A|B$ for the concatenation of the two along an axis, which will be implicitly understood from the context, or otherwise specified.

\subsubsection{Learning task}
A \emph{word} $w$ in letters $a,b,a^{-1},b^{-1},$ is a finite string made of the letters $a$,$b$,$a^{-1}$ and $b^{-1}$. We shall identify such word with the operation it induces naturally on a group by interpreting the word as an expression in the arguments $a$, $b$ with $a^{-1}$, $b^{-1}$ being their group inverses and concatenation taken as applying the group multiplication. For example, $w=abab^{-1}aaa,$ is identified with the operation $w: G \times G \to G$, given by $w(a,b) = a \cdot b \cdot a \cdot b^{-1} \cdot a^3$. We shall occasionally refer to such group operation as a word operation. The learning task is that of learning group word operations.

\subsubsection{Encoding, decoding and dataset}
As a model's input and output are real valued vectors, we shall 
use one hot encoding to encode the group elements of the operations arguments and result. Under this encoding, the task becomes that of learning the full dataset:
\begin{equation}
\label{e:12.1}
\mathcal{D}_{G,w} = \Big\{(u,v) :\: u=1_a|1_b,\, v=1_c\,,\,\, c = w(a,b)\,,\,\, a,b \in G\Big\} \subseteq \bbR^{2|G|} \times \bbR^G
\,,
\end{equation}
where $u$ is the input and $v$ is the output (or label). In the other direction, the $\bbR^G$-output of a model is decoded via the argmax function. Thus, a model which computes the function $f: \mathbb{R}^{2|G|} \to \mathbb{R}^{|G|}$ is considered as implementing the operation,
\begin{equation}
\label{e:12.2}
	w_f(a,b) := {\rm argmax}_{c \in G}\, f(1_a|1_b)_c \,.
\end{equation}

\subsubsection{Loss and accuracy}
Loss will be computed using the MSE function, so that given samples 
$\cS = \big((u_i=1_{a_i}|1_{b_i},\, v_i=1_{c_i}): i=1,\dots,n \big)\subseteq \cD_{G,w}$, the total loss (empirical risk) is
\begin{equation}
\label{e:1.15}
	L(\cS; W) \equiv L_{f}(\cS; W) := \frac{1}{|G|n} \sum_{i=1}^n \big\|f(u_i;W) - v_i\big\|_2^2 \,,
\end{equation} 
and the accuracy is given by
\begin{equation}
A(\cS;W) \equiv A_{f}(\cS;W) := \frac{1}{n} \sum_{i=1}^n \delta_{w_f(a_i,\,b_i)=c_i} \,.
\end{equation}
Normalization by the size of the group in the total loss permits comparison between losses when the model is run on different groups.
\begin{figure}[t!]
\begin{minipage}{0.5\linewidth}
    \centering
    \fbox{\includegraphics[scale=0.5]{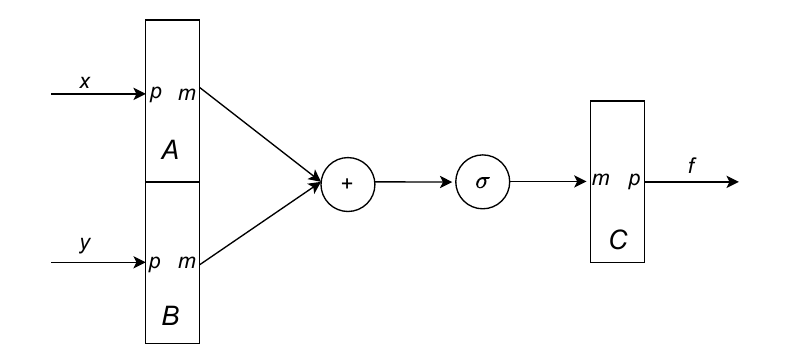}}
\end{minipage}\hfill
\begin{minipage}{0.5\linewidth}
    \centering
    \fbox{\includegraphics[scale=0.5]{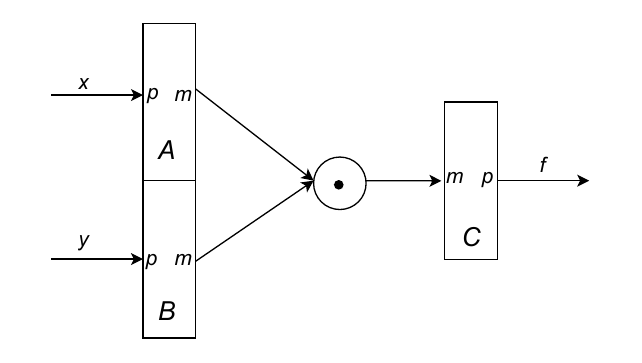}}
\end{minipage}
\captionsetup{width=0.98\linewidth}
\caption{\small A schematic diagram of the TLP ({\em Left)} and HD ({\em Right}) networks. Rectangles denote linear fully connected layers with input and output dimensions indicated inside. Circles denote element-wise operations. }
\label{f:0}
\end{figure}

\subsubsection{Model: Two-Layer Perceptron}
\label{s:10.2}
We shall consider first a standard two layer perceptron with one hidden layer of width $m \geq 1$, and activation function $\sigma: \mathbb{R} \to \mathbb{R}$. Given weights $W = (W^{(1)}, W^{(2)})$, where $W^{(1)} \in \bbR^{m \times 2|G|}$, $W^{(2)} \in \bbR^{m \times |G|}$, the {\em Two-Layer Perceptron} Model (TLP) implements
 $f_{\rm TLP, \sigma}(\,\cdot\,;W): \mathbb{R}^{2|G|} \to \mathbb{R}^{|G|}$, given by
\begin{equation}
\label{e:3.1}
	f_{\rm TLP, \sigma}(u; W) :=  W^{(2)} \sigma \big(W^{(1)} u\big) 
	\quad , \qquad u \in \bbR^{2|G|} \,,
\end{equation}
with $\sigma$ applied entry-wise. 
Interpreting the input as $u = x|y$ for $x,y \in \bbR^G$ and the weights as 
\begin{equation}
\label{e:1}
W^{(1)} = A|B \,,\,\, \qquad W^{(2)} = C^T \quad;\qquad A,B,C \in \mathbb{R}^{m\times G}\,,
\end{equation}
we shall also think of $f_{\rm TLP, \sigma}$ as a function from $\bbR^G \times \bbR^G$ to $\bbR^G$ with weights $A,B,C$, via the identification
\begin{equation}
\label{e:3.2}
f_{\rm TLP, \sigma}\big(x,y\,;\; A,B,C\big) \equiv f_{\rm TLP, \sigma}\big(x|y\;;\; (A|B, C)\big) =  C^T\sigma(Ax+By\big) 	\,.
\end{equation}
See Figure~\ref{f:0}
for a schematic diagram of the network.
The set of all weight assignments for the model is 
\begin{equation}
\label{e:20}
	\cW_G = \Big \{W = (W^{(1)}, W^{(2)}) = (A,B,C) \  : \ \ A,B,C \in \bbR^{m \times G},\, m \geq 1 \Big\} \,.
\end{equation}
Given $W = (A,B,C) \in \cW_G$, we shall write $|W|$ for the {\em width} of $W$, namely $m$ such that $A,B,C \in \bbR^{m \times G}$. We shall also write $\cW_{G,m}$ for the restriction of $\cW_G$ to all $W$ with $|W|=m$.

\subsubsection{Optimization and initialization}
\label{s:2.1.6}
To avoid unrelated effects, in most experiments we use pure Gradient Descent without acceleration or additional regularization. Formally, 
given learning rate $\eta > 0$, and sample set $\cS$,  the one-step gradient descent evolution is the function
${\rm GD} \equiv {\rm GD}_{f,\eta,\cS} : \cW \to \cW$ given by 
\begin{equation}
\label{e:1.17}
{\rm GD}_{f, \eta,\cS}(W) := W - \eta \nabla_w L_{f}(\cS;W) \,.
\end{equation}
For $t \in \bbN$, We shall write ${\rm GD}^t$ for the $t$-time composition of ${\rm GD}$ with itself, so that ${\rm GD}^t(W)$ is the the weights of the model after $t$ steps of gradient descent, starting from $W$.

Initial weights are chosen from the centered Gaussian distribution with variance inversely proportional to the width, as customary.

\subsection{Empirical study}
We first tested whether the TLP model with standard activation functions can learn a word operation when trained on a a subset of the full dataset. We tried several groups and words. The former includes cyclic groups, abelian and non-abelian (see Section~\ref{s:3} for the precise definitions). The latter includes words of different lengths. Initialization and optimization as indicated in Subsection~\ref{s:2.1.6}.
The results, a partial account of which is given in Figure~\ref{f:A1}, clearly showed that the model is able to robustly learn all groups and words, once the width of the network is sufficiently large (depending on the group, word and activation function). Grokking was also frequently observed.

\begin{figure}[!t]	
\begin{minipage}{0.75\linewidth}
\resizebox{\textwidth}{!}{
\begin{tabular}{|c|c|c|c|c|c|c|c|}
\hline
Group & Word & $N$ & Activation & $\alpha$ & \thead{Median final test \\accuracy} & \thead{Max. final test \\accuracy} & \thead{Max. final train\\ loss}  \tabularnewline 
\hline
\hline
$D_{8}$ & $a^2 b$ & 48 & ReLU & 0.8 & 0.923077 & 1 & 0.00022\tabularnewline
\hline
$D_{8}$ & $aba^{-1}ba^2b^3ab^{-1}$ & 32 & square & 0.6 & 1 & 1 & 7.6e-05\tabularnewline
\hline
$D_{8}$ & $aba^{-1}ba^2b^3ab^{-1}$ & 48 & sigmoid & 0.7 & 1 & 1 & 0.00027\tabularnewline
\hline
$M_{5}(2)$ & $aba^{-1}ba^2b^3ab^{-1}$ & 64 & ReLU & 0.5 & 0.996094 & 1 & 0.00089\tabularnewline
\hline
$M_{5}(2)$ & $aba$ & 64 & ReLU & 0.7 & 1 & 1 & 0.0038\tabularnewline
\hline
$M_{5}(2)$ & $a^2 b$ & 64 & square & 0.5 & 1 & 1 & 0.00034\tabularnewline
\hline
$S_{4}$ & $aba$ & 64 & square & 0.7 & 0.979769 & 1 & 0.00061\tabularnewline
\hline
$S_{4}$ & $a^2 b$ & 196 & sigmoid & 0.8 & 0.982759 & 1 & 3.7e-05\tabularnewline
\hline
\end{tabular}
}
\end{minipage}\hfill
\begin{minipage}{0.25\linewidth}
    \includegraphics[width=1\textwidth]{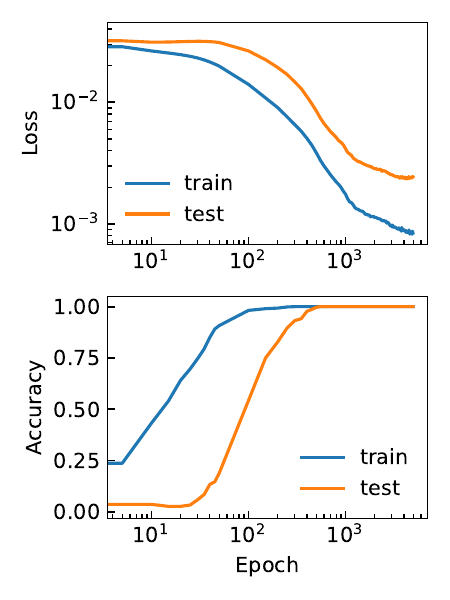}
\end{minipage}
\captionsetup{width=0.98\linewidth}
\caption{\small {\em Left:} Training results using the TLP model on various words and groups. $\alpha$ is fraction of samples given during test. Learning rate is 0.005, optimizer is AdamW. 20 runs per configuration.
{\em Right:} Loss and accuracy evolution during training for the  group $M_5(2)$, word $aba$ with the TLP model of width $N=64$, and the ReLU activation function, and with $\alpha=0.7$ fraction of the samples as the training set.
}
\label{f:A1} 
\vspace{0.3cm}
\end{figure}
 
\section{Background on groups, representations, tensors and fusion}
\label{s:3}
Next we briefly recall the necessary mathematical background on group, representation and fusion.
\subsection{Group theory}
A \emph{group} is a non-empty set $G$ with an binary operation $(x,y) \mapsto x \cdot y \equiv xy$, usually referred to as the \emph{group multiplication}, such that:
\begin{enumerate}
	\item The multiplication operation is {\em associative}, namely $(xy)z = x(yz)$, for all $x,y,z\in G$.
	\item There exists a \emph{unit element} $e\in G$, such that
$ex=xe=x$ for all $x \in G$.
	\item For all $x \in G$ there exists an {\em inverse} $x^{-1}$ such that $xx^{-1}=x^{-1}x = e$.
\end{enumerate}
A group is called {\em Abelian} if the operation is \emph{commutative}, namely $xy = yx$, for all $x,y \in G$. A subset $H \subseteq G$ {\em generates} $G$ if its closure under the group operation is $G$. An abelian group is called {\em cyclic} if it is generated by $\{g\}$ for some $g \in G$, which is then called a {\em generator}. A group is called {\em finite} if $|G| < \infty$. Figure~\ref{f:A2} includes various common finite groups and their properties. The set ${\rm GL}(\bbC^d)$, which includes all $d\times d$ invertible complex-valued matrices, forms an infinite group under the matrix multiplication as the group operation. This group will play an important role in what comes next.

\begin{figure}[t!]
\setlength{\tabcolsep}{1pt}
\scriptsize
\begin{tabular}{|p{0.05\linewidth}|p{0.2\linewidth}|p{0.3\linewidth}|p{0.03\linewidth}| p{0.1\linewidth}|p{0.28\linewidth}|}
\hline 
{\bf Sym} & {\bf Name} & {\bf Description} & {\bf Size} & {\bf Properties} & {\bf Notes} \tabularnewline
\hline 
\hline 
$\bbZ_p$ & Additive group mod $p$ & $\bbZ_p = \{0,\dots, p-1\}$ with  addition mod $p$. & p & Cyclic &  \tabularnewline
\hline 
$\bbZ_p^\star$ & Multiplcative group mod $p$ & 
    $\mathbb{Z}^\star_p = \{1,\dots, p-1\}$ for $p$ prime, with the multiplication mod $p$. & p-1 & Cyclic & Isomorphic to $\bbZ_{p-1}$ under the isomorphism $\bbZ_{p-1} \ni k \mapsto g^k \in \mathbb{Z}^\star_p$, for any generator $g \in \mathbb{Z}^\star_p$.\tabularnewline
\hline 
$S_n$ & Symmetric Group & 
The set of all bijections from $\{1,\dots,n\}$ to itself, with the operation being composition. & $n!$ &  Non-Abelian for $n \geq 3$ .& \tabularnewline	
\hline 
$D_n$ & The Dihedral Group & Includes all symmetries of a regular $n$-gon, with the operation being composition. & $2n$ & Non-Abelian for $n \geq 3$. &  The group is generated by two elements a $2\pi/n$ rotation and a reflection by a symmetry axis. \tabularnewline	
\hline
$Q_8$ & The Quaternionic Group & $Q_8=\{\pm 1,\pm i,\pm j,\pm k\}$ with $i^2=j^2=k^2=ijk=-1$. & 8 & Non Abelian & Has quaternionic representations.\tabularnewline	
\hline
$M_5(2)$ & Modular maximal cyclic \newline group of order 32 & Generated by $a,b$ whose only relations is $a^16=b^2=1,~bab=a^9$.  & 32 & Non Abelian & Has $2$ dimensional non self conjugate representations.\tabularnewline	
\hline
\end{tabular}
\captionsetup{width=0.98\linewidth}
\caption{\small Various finite groups and their properties.}
\label{f:A2} 
\vspace{0.3cm}
\end{figure}

\subsection{Representation theory}
Next, let us briefly recall the theory of group representation, We follow  \cite[Sections 1-3]{fulton2013representation}, and all lemmas in this subsection either appear there, or are straight forward to derive.
\subsubsection{Group representation}
Given a (finite) group $G$ and $d \geq 1$ a \emph{group representation} (over $\bbC$) $\phi$ is a \emph{homomorphism} between $G$ and the group ${\rm GL}(\bbC^d)$. That is, $\phi:G\to {\rm GL}(\bbC^d)$ satisfies 
\begin{equation}
	\phi(gh) = \phi(g) \phi(h)
	\quad ; \qquad g,h \in G
\end{equation}
The {\em dimension} of $\phi$ is $\dim(\phi) \equiv d_\phi = d$.
Two representations $\phi,\psi:G\to {\rm GL}(\bbC^d)$ are \emph{isomorphic}, or {\em versions} of each other, if 
there exists a change-of-basis matrix $P \in {\rm GL}(\bbC^d)$ such that 
$\phi(g) = P \psi(g) P^{-1}$ for all $g \in G$. 
The \emph{conjugate representation} $\bar\phi$ is defined as $g\to\overline{\phi(g)}$, where the latter means the conjugation of every entry of the matrix $\phi(g).$
A representation is \emph{self conjugate} (sc)  if it 
is isomorphic to its conjugate representation. We shall often omit the word representation and just write sc for a self-conjugate representation.

\subsubsection{Basic self conjugate representations}
The \emph{direct sum} of $\phi:G\to {\rm GL}(\bbC^d),\psi:G\to {\rm GL}(\bbC^{d'})$ is the representation $\phi\oplus\psi:G\to {\rm GL}(\bbC^{d+d'})$, given by
\begin{equation}
\label{e:2.1}
(\phi\oplus\psi)(g) = \begin{pmatrix}
    \phi(g)&0\\
    0&\psi(g)
\end{pmatrix} \,
\end{equation}
A representation is called \emph{basic self conjugate} or {\em bsc} if it is sc, but not isomorphic to the direct sum of two sc representations. We shall write ${\rm bscs}(G)$ for the set of all bscs of $G$ (up-to isomorphisms). Note that, the latter always includes the trivial representation $\phi \equiv 1$, which we henceforth denote by {\rm Triv}. Figure~\ref{f:A3} includes explicit examples of bscs of various groups and their properties. 
bscs are the sc-analogs of the (more familiar) {\em irreducible representations} or {\em irreps}, which are defined in the same way, albeit without the sc requirement. bscs are more suitable when working over vectors spaces over the reals, as is necessitated by our (real-valued) models. 

\begin{figure}[t!]
\setlength{\tabcolsep}{1pt}
\scriptsize
\begin{minipage}[t]{0.7\linewidth}
\vspace{0pt}
\begin{tabular}{|c|c|c|c|c|p{0.25\linewidth}|}
\hline 
$G$ & $\phi$ & $d$ & T & $D$ 
& Notes \tabularnewline
\hline 
\hline 
$\bbZ_p$ & $\phi_0 = {\rm Triv}$ & $1$ & I & $1$  & \tabularnewline
& $\phi_{p/2}(k) = (-1)^k$ & $1$ & I & $1$ 
& if $p$ even \tabularnewline
& $\phi_j(k) :=
\begin{pmatrix}
\cos \frac{2\pi j}{p}k &-\sin\frac{2\pi j}{p}k\\
\sin\frac{2\pi j}{p}k&\cos\frac{2\pi j}{p}k \,,
\end{pmatrix}
\ ;\ \  j \in \big[1, \lfloor (p-1)/2 \rfloor\big]
$ & $2$ & II & $2$ 
& \tabularnewline
\hline
$\bbZ_p^\star$ & $\phi^\star_j(g^k) = \phi_j(k)
\ \ ;\ \ \  j = 0, \dots, \lfloor p/2 \rfloor \,.$
& & &  
 & $\phi_j$ bsc of $\bbZ_p$\newline $g$ generates $\bbZ_p^\star$ \tabularnewline
\hline
$Q_8$ & 
$ 1\mapsto\begin{pmatrix}1&0\\0&1
\end{pmatrix},~i\mapsto\begin{pmatrix}\iota &0\\0&-\iota
\end{pmatrix},~j\mapsto\begin{pmatrix}0&1\\-1&0
\end{pmatrix},~k\mapsto\begin{pmatrix}0&\iota\\\iota&0
\end{pmatrix}$
& $2$ & III & $4$ & $\iota = \sqrt{-1}$ \tabularnewline
\hline
$S_3$& 
$\phi_1 = {\rm Triv}$, $\phi_2(\sigma) = sgn(\sigma)$
& $1$ & I & & $\simeq D_3$ 
\tabularnewline
& $\phi_3(a)=\begin{pmatrix}\cos(\frac{2\pi}{3}) &-\sin(\frac{2\pi}{3})\\
\sin(\frac{2\pi}{3}) &\cos(\frac{2\pi}{3})\end{pmatrix}
\,,~~\phi_3(b)=\begin{pmatrix}
    1&0\\
    0&-1
\end{pmatrix}\, $
& $2$ & II & $4$ &  $a=(123)$, $b=(12)$ generators
\tabularnewline
\hline
\end{tabular}
\end{minipage}\hfill
\begin{minipage}[t]{0.3\linewidth}
\vspace{0pt}
\flushright
\begin{tabular}{|c|c|c|c|}
\hline 
$G$ & $d$ & Type & $\#$ %& Notes? 
\tabularnewline
\hline 
\hline 
$\mathbb{Z}_n$ &1 &I &2~\text{if $2|n$, otherwise }1 
\tabularnewline
&2&II&$\lfloor\frac{n-1}{2}\rfloor$%& 
\tabularnewline
\hline
$S_4$ & 1 & I & 2 %& 
\tabularnewline
& 2 & I & 1 %& 
\tabularnewline
& 3 & I & 2 %& 
\tabularnewline
\hline
$M_5(2)$ & 1 & I & 4 %& 
\tabularnewline
 & 2& II& 5%&  
 \tabularnewline
 & 4& II& 2%&  
 \tabularnewline
\hline
$Q_8$ & 1 & I & 4 %& 
\tabularnewline
 &2 & III& 1%& 
 \tabularnewline
\hline
$D_{n}$ & 1 & I & 4~\text{if $2|n$, otherwise }2 %&  
\tabularnewline
 &2 &I &$\lfloor\frac{n-1}{2}\rfloor$ %&  
\tabularnewline
\hline
\end{tabular}
\end{minipage}
\captionsetup{width=0.98\linewidth}
\caption{\small {\em Left:} Explicit examples of bscs for various groups.
{\em Right:} Summary information on the bscs of various groups, namely the number of bscs of a given dimension $d$ and type.}
\label{f:A3} 
\vspace{0.3cm}
\end{figure}

\subsubsection{The space of matrix coefficients of a representation}
The {\em space $R_\phi$ of (real) matrix coefficients} associated with a $d$-dimensional sc representation $\phi$ is the subspace of $\bbR^{G}$ spanned by the real and imaginary parts of the $d^2$ matrix entries of (a version of) $\phi$, viewed as vectors in $\bbC^G$, namely
\begin{equation}
\label{e:6}
R_\phi := {\rm span}\, \Big\{
	\Re(\phi_{i,j}), \Im(\phi_{i,j}) :\: i,j \in [d_\phi] \Big\}
	\subseteq \bbR^{G} \,,
\end{equation} 
We will occasionally refer to $R_\phi$ just as the subspace {\em associated} with the representation $\phi$.
$R_\phi$ is invariant under isomorphisms of the representation and may have a smaller dimension than $2d^2$. 

If $\psi,\phi$ are two different irreps or bscs, then $R_{\phi \oplus \psi} = R_\phi \oplus R_\psi$, and
their corresponding subspaces are orthogonal with respect to the standard inner product
\[\langle u,v\rangle:=\sum_{g\in G}u_g \bar{v}_g.\]
In particular, every sc $\phi$ can be uniquely decomposed into a direct sum of bscs, and we shall denote the set of such bscs by $\mathrm{bscs}(\phi)$.
Similarly, the subspaces corresponding to all bscs of $G$ form an orthogonal decomposition of $\bbR^{|G|}$, namely
\begin{equation}
\label{e:7}
	\bbR^{|G|}=\bigoplus_{\phi \in {\rm bscs}(G)} R_{\phi} \,.
\end{equation}
In particular, every $v \in \bbR^G$ can be uniquely written as an orthogonal sum of its projections onto the subspaces corresponding to all bscs of $G$, and we shall write $\bscs(v)$ for the set of bscs for which this projection is non null.

\subsubsection{Types of bsc representations}
A representation
is {\em real} if it is isomorphic to a representation whose matrix entries are real. 
While every real representation is clearly sc, the converse is not true. A representation is called \emph{pseudoreal}, or \emph{quaternionic} if it is sc but not real. We will refer to a non sc representation as  \emph{complex}. 
\begin{lemma}\label{lem:triad}Every bsc is either real and irreducible (\emph{type I}), real but of the form $\phi\oplus\bar{\phi}$ for complex and irreducible $\phi$ (\emph{type II}) or quaternionic and irreducible (\emph{type III}). Conversely, evrery representation of one of these types is a bsc.    
\end{lemma}
The next lemma gives the dimension $D$ of the space associated with bsc $\phi$ of dimension $d$ and a given type.
\begin{lemma}\label{obs:dims_reps}Let $\phi$ be a bsc of dimension $d.$ Then 
\[D_\phi := \dim(R_{\phi}) =\begin{cases}d^2,&\phi~\text{is of types I,III}\\
    \frac{1}{2}d^2,&\phi~\text{is of type II}
\end{cases}.\]
\end{lemma}
\noindent
Figure~\ref{f:A3} includes a summary of the bscs of various groups and their types.

\subsubsection{Characters}
The {\em character} of a representation $\phi$ is defined as
\begin{equation}
	\chi_\phi(g) := \Tr \big(\phi(g)\big) = \sum_{i=1}^{d_\phi} \big(\phi(g)\big)_{i,i} \,.
	\quad ; \qquad g \in G \,,
\end{equation}
and is invariant under isomorphisms of $\phi$. If $\phi$ is sc then its character is real-valued.

\subsection{Tensors and Fusion}
\subsubsection{Tensors}
Recall that the \emph{tensor product} $V_1\otimes V_2\otimes \cdots\otimes V_m$ (over $\bbR$) of the (real) vector spaces $V_1,\ldots,V_m$ is the vector space spanned by all elements $v_1\otimes\cdots \otimes v_m,~v_i\in V_i,$ and subject to the relations generated by 
\[
v_1\otimes\cdots \otimes(\lambda v_i+\mu v'_i)\otimes\cdots \otimes v_m=\lambda v_1\otimes\cdots \otimes v_i\otimes\cdots \otimes v_m \, +\, \mu v_1\otimes v'_i\otimes\cdots \otimes v_m \,,
\]
for $\lambda, \mu \in \bbR$, $v_i, v_i' \in V_m$ and $i=1, \dots m$.
%\footnote{after merging - erase the definition of 3-tensors.
An element of the above space is a called a {\em tensor} (of order $m$). It is called {\em pure} or {\em elementary} if it can be written as $v_1\otimes\cdots \otimes v_m$ with $v_i$ as above. In this case, we also say that $T$ is the {\em the tensor product} of $v_i, \dots, v_m$. Tensor product and direct sum are associative:
\begin{equation}
\label{e:A23}
V_1 \otimes \dots (V_i \oplus V_i') \otimes \cdots \otimes V_m
= (V_1 \otimes \dots V_i \otimes \cdots \otimes V_m) \oplus
(V_1 \otimes \dots V_i' \otimes \cdots \otimes V_m) \,.
\end{equation}
Also, if $V_i$ are equipped with inner products $\langle \cdot, \rangle_i$, then a unique inner product may be defined on the tensor product space via
\begin{equation}
\label{e:A223}
	\langle v_1 \otimes \dots \otimes v_m\,,\,\, v'_1 \otimes \dots \otimes v'_m \rangle
	:= \langle v_1, v'_1 \rangle_1 \cdot \ldots \dot \langle v_m, v'_m \rangle_m \,,
\end{equation}
for all $v_i, v_i' \in V_i$ and $i \leq m$.

The space $V_1^* \otimes \cdots \otimes V_m^*$, where $V_i^*$ is the dual space of $V_i$, can be identified with the space of linear forms on $V_1 \times \cdots \times V_m$, via the isomorphism which maps
$v_1^*\otimes\cdots \otimes v^*_m$, for $v^*_i \in V_i^*$ to the linear form
\begin{equation}
 l_{v_1^*\otimes\cdots \otimes v^*_m}\big((v_1, \dots, v_m)\big)
= {v_1^*}(v_1) \cdot \ldots \cdot {v_m^*}(v_m)\,.
\end{equation}
We shall often also identify $v \in \bbR^d$ with the the linear functional $l_v := \langle v, \,\cdot \,\rangle \in (\bbR^d)^* \equiv \bbR^d$, given by the usual Euclidean inner product. Combining the two, if $v_i \in \bbR^{d_i}$ for $i = 1, \dots, m$, then $v_1 \otimes \dots \otimes v_m$ is identified with the linear form
\begin{equation}
\label{e:A1}
	l_{v_1 \otimes \dots \otimes v_m}(w_1, \dots, w_m) =
	\langle v_1, w_1 \,\rangle \cdot \ldots \cdot \langle v_m, w_m \,\rangle \,.
\end{equation}

The \emph{rank} of a tensor $T$ is the minimal number $N$ such that $T$ can be written as the sum of $N$ pure tensors. In general it is difficult to determine the rank of a tensor of order $m \geq 3$~\cite{haastad1990tensor,hillar2013most}. Nevertheless, the following is a straightforward bound on the rank of $3$-tensors:
\begin{lemma}\label{obs:tensor_rank_naiv1_bound}
    The rank of a $3$-tensor $T \in V_1 \otimes V_2 \otimes V_3$ is upper bounded by $\min(d_1d_2,d_1d_3,d_2d_3)$, where $d_i = \dim(V_i)$, $i \leq 3$.
\end{lemma}

\subsubsection{Tensor product of representations}
The {\em tensor product} of $\phi_1:G\to {\rm GL}(\bbC^d),\phi_2:G\to {\rm GL}(\bbC^{d'})$ is the representation $\phi_1\otimes\phi_2:G\to {\rm GL}(\bbC^d \otimes \bbC^{d'})$, given by
\begin{equation}\label{eq:def_tensor}
(\phi_1 \otimes \phi_2) (g)  = \phi_1(g) \otimes \phi_2(g) \,.
\end{equation}
If $\phi_1$ and $\phi_2$ are sc then so is $\phi_1 \otimes \phi_2$. \eqref{eq:def_tensor} implies  that $R_{\phi_1 \otimes \phi_2} = {\rm span}\{\nu_1\odot\nu_2|\nu_1\in R_{\phi_1},\nu_2\in R_{\phi_2}\}$, where $v_1 \odot v_2$ is the \emph{Hadamard} (element-wise) product 
of $v_1$ and $v_2$.
In particular, 
if $v_1 \in R_{\phi_1}$, $v_2 \in R_{\phi_2}$ then
\begin{equation}
\label{obs:product_gives_fusion}
v_1 \odot v_2 \in R_{\phi_1 \otimes\phi_2} \,.	
\end{equation}
In view of~\eqref{e:7},~\eqref{e:A23} and~\eqref{e:A223}, we have
\begin{equation}
\label{e:A226}
(\bbR^G)^{\otimes m} = \bigoplus_{(\phi_1, \dots, \phi_m) \in \bscs(G)^m} R_{\phi_1} \otimes \dots \otimes R_{\phi_m} \,,
\end{equation}
where subspaces in the above direct sum are orthogonal w.r.t. the natural (Euclidean) inner product on $(\bbR^G)^{\otimes m}$. As before, we define the {\em bsc$^m$-support} of a tensor $T \in  (\bbR^G)^{\otimes m}$ as the collection of triplets $(\phi_1, \dots, \phi_m)$ for which 
the projection of $T$ onto $R_{\phi_1} \otimes \dots \otimes R_{\phi_m}$, henceforth $T_{\phi_1 \otimes \dots \phi_m}$ is non-trivial. In particular, for elementary tensors we have
\begin{equation}
\bscs^m(v_1 \otimes \dots \otimes v_m) = \bscs(v_1) \times \dots \times \bscs(v_m) \,.
\end{equation}

\subsubsection{Fusion}
In general the tensor product of two (sc) representations is not bsc and as such it decomposes into a direct sum of bscs. The \emph{(sc) fusion structure} of (the representation category of) a group $G$ is the explicit isomorphisms between $\phi_1\otimes\phi_2$ and their decomposition into direct sum of bscs, for any pair of bscs $\phi_1,\phi_2$. The bscs which participate in the decomposition for each pair, form the combinatorial part of this structure, and are collectively referred to as 
the \emph{fusion table} of the group. The table in Figure~\ref{f:A5} shows the fusion tables of groups $S_4$,$D_8$ and $M_5(2)$, as examples.

\begin{figure}[t!]
\setlength{\tabcolsep}{3.5pt}
\scriptsize
\begin{minipage}[t]{0.6\linewidth}
\vspace{0pt}
\begin{tabular}{ |c|c| c| c| c| c| c|c|c|c|c|c|c|c|}
 \hline
 $\bf{M_5(2)}$ &$D_\phi$&0&1&2&3&4&5&6&7&8&9&10&11\\
 \hline
 0&1&0&1&2&3&4&5&6&7&8&9&10&11\\
 \hline
 1&1&&0&3&2&5&4&7&6&9&8&10&11\\
 \hline
 2&1&&&0&1&8&9&6&7&4&5&10&11\\
 \hline
 3&1&&&&0&9&8&7&6&5&4&10&11\\
 \hline
 4&2&&&&&0,6&1,7&4,8&5,9&2,6&3,7&10,11&10,11\\
 \hline    
 5&2&&&&&&0,6&5,9&4,8&3,7&2,6&10,11&10,11\\
 \hline    
 6&2&&&&&&&0,2&1,3&4,8&5,9&11&10\\
 \hline    
 7&2&&&&&&&&0,2&5,9&4,8&11&10\\
 \hline
 8&2&&&&&&&&&0,6&1,7&10,11&10,11\\
 \hline
 9&2&&&&&&&&&&0,6&10,11&10,11\\
 \hline
 10&8&&&&&&&&&&&0-5,8,9&4-9\\
 \hline
 11&8&&&&&&&&&&&&0-5,8,9\\
 \hline
\end{tabular}
\end{minipage}\hfill
\begin{minipage}[t]{0.4\linewidth}
\vspace{0pt}
\flushright
\begin{tabular}{ |c|c| c| c| c| c| c|}
 \hline
 $\bf{S_4}$&$D_\phi$& 0&1&2&3&4\\
 \hline
 0&1&0&1&2&3&4\\
 \hline
 1&1&&0&2&4&3\\
 \hline
 2&4&&&0,1,2&3,4&3,4\\
 \hline
 3&9&&&&0,2,3,4&1-4\\
 \hline
 4&9&&&&&0,2,3,4\\
 \hline    
\end{tabular}
\medskip

\begin{tabular}{|c|c| c| c| c| c| c|c|c|}
\hline
  $\bf{D_8}$& $D_\phi$ & 0&1&2&3&4&5&6\\
 \hline
 0&1&0&1&2&3&4&5&6\\
 \hline
 1&1&&0&3&2&4&5&6\\
 \hline
 2&1&&&0&1&6&5&4\\
 \hline
 3&1&&&&0&6&5&4\\
 \hline
 4&4&&&&&0,1,5&4,6&2,3,5\\
 \hline    
 5&4&&&&&&0,1,2,3&4,6\\
 \hline    
 6&4&&&&&&&0,1,5\\
 \hline    
\end{tabular}

\end{minipage}
\captionsetup{width=0.98\linewidth}
\caption{\small Fusion tables for groups $M_5(2)$, $S_4$ and $D_8$. bscs are indexed in non decreasing order of dimensions, starting from the trivial representation $0$. The $(i,j)$-th slot contains the indices of all bscs which are included in the tensor project of bsc $i$ and $j$. The second column contains the dimension of the subspace associated with the bsc in that row.
}
\label{f:A5} 
\vspace{0.3cm}
\end{figure}

\section{Analysis}
\label{s:4}
\subsection{Learning task as implementing a word tensor}
Fix $G$ and $w$. A sufficient condition for a model to achieve zero loss on the full set $\cD_{G,w}$, is for it to implement the $3$-tensor:
\begin{equation}
\label{eq:delta_word}
\delta_{c=w(a,b)}=\sum_{a,b\in G}1_a\otimes 1_b\otimes 1_{w(a,b)}\ \in\ (\mathbb{R}^G)^{\otimes 3} \,,
\end{equation} 
where the tensor product is interpreted as the product of the corresponding inner products, as in~\eqref{e:A1}. We shall call the latter the {\em word tensor} corresponding to $G$ and $w$ and abbreviate it as $\delta_{G,w}$
\subsection{Bsc$^3$-support of word tensors}

In view of~\eqref{eq:delta_word}, a word tensor acting on group $G$ has rank at most $|G|^2$. It therefore follows from the discussion above, that an HD model of width $m=|G|^2$ can achieve zero loss on the corresponding dataset. We wish to claim, however, that for many words $w$, the rank of this tensor is much lower, and thus considerably less width is required to implement it. To this end, we begin by showing that the bsc$^3$-support of word tensors is typically small.

A key point is that the fusion structure of $G$ restricts the above set considerably. 
Recall that $\bscs^m(T)$ and $\bscs(\phi)$ denote the (subspaces associated with the) bscs in the direct sum decomposition of (the subspace associated with) $m$-tensor $T$ and sc representation $\phi$. Define 
\begin{equation}\label{eq:bscs3_contained_fusion}
{\rm bscs}^3_\CF(\delta_{G,w}) := 
\Big\{(\phi,\psi,\zeta) \in \bscs(G)^3 :\: \phi\in{\rm bscs}\big(\zeta^{\otimes n_a(w)}\big)\,,\,\,\psi\in{\rm bscs}\big(\zeta^{\otimes n_b(w)}\big)\Big\} \,,
\end{equation}
where $n_a(w)$ and $n_b(w)$ are the number of appearances of $a^{\pm1},$ and $b^{\pm1}$ in $w$, respectively. We shall call the above set the {\em Combinatorial-Fusion-Cover} (or CFC) of the bsc$^3$-support of $\delta_{G,w}$. The name is explained by,

\begin{proposition}\label{prop:fusion_rep_of_tensor_short}
For any group $G$ and word $w$,
\begin{equation}
\label{e:24}
\bscs^3_\CF\big(\delta_{G,w}\big) \supseteq \bscs^3\big(\delta_{G,w}\big) \,.
\end{equation}
\end{proposition}
\noindent
The proposition thus provides a way to bound the bsc$^3$-support of $\delta_{G,w}$ using the fusion table, without explicit computation, which in general is quite tedious. More importantly, it shows that the fusion structure of the group, its combinatorial part in particular, is the core reason for the sparsity of the bsc$^3$-support of the word tensor, and thus for the ability of the network to learn the word. We remark that the the full fusion structure of the group, which determines the true bsc$^3$-support, may imply that more components in the orthogonal decomposition of the word tensor are zero and thus the inclusion in~\eqref{e:24} can be a strict one for certain groups and words.

The table in Figure~\ref{f:A6} lists the CFCs obtained via Proposition~\ref{prop:fusion_rep_of_tensor_short} and the fusion tables of various groups and words, along-side the true bsc-support of the word tensor along-side. The table clearly shows that that CFC of the bsc$^3$-support of $\delta_{G,w}$ and therefore the bsc$^3$-support itself can be much smaller set than $\bscs^3(G)$. It also shows that the CFC can be a proper superset of the true-support, as in the case of the group $M_5(2)$ and all words considered.

\renewcommand{\arraystretch}{1.3}
\begin{figure}[t!]
\setlength{\tabcolsep}{1pt}
\scriptsize
\begin{tabular}{|c|c|p{.39	\textwidth}|p{.39	\textwidth}|}
\hline 
w & G & $\bscs^3_\CF(\delta_{G,w})$ & $\bscs^3(\delta_{G,w})$
\tabularnewline
\hline
\hline
$a^2b$ or $aba$ & $S_4$ &  
\raggedright
$(0,0,0)$, $(0,1,1)$, $(0,2,2)$, $(1,2,2)$, $(2,2,2)$, $(0,3,3)$, $(2,3,3)$, $(3,3,3)$, $(4,3,3)$, $(0,4,4)$, $(2,4,4)$, $(3,4,4)$, $(4,4,4)$ & Same.
\tabularnewline
& $D_8$ & 
\raggedright
$(0,0,0)$, $(0,1,1)$, $(0,2,2)$, $(0,3,3)$, $(0,4,4)$, $(1,4,4)$, $(5,4,4)$, $(0,5,5)$, $(1,5,5)$, $(2,5,5)$, $(3,5,5)$, $(0,6,6)$, $(1,6,6)$, $(5,6,6)$ & Same.
\tabularnewline
& $M_5(2)$ & 
\raggedright
$(0,i,i),\, i=0-11$, $(2,i,i),\, i=6,7$,\newline $(6,i,i),\, i=4,5,8,9$,\newline $(j,i,i), j=1-5,8,9, i=10,11$ & 
$(0,i,i),i=0-3$, $(2,i,i),\, i=6,7$,\newline$(6,i,i),\, i=4,5,8,9$,\newline$(j,i,i), j=4,5,8,9, i=10,11$
\tabularnewline
\hline
$aba^{-1}ba^2b^3ab^{-1}$ & $S_4$ & 
\raggedright
$(0,0,0)$, $(1,0,1)$, $(i,j,2),~i,j=0,1,2$, $(i,j,k),~k=3,4,~i,j=0-4$ & Same.
\tabularnewline
& $D_8$ & 
\raggedright
$(i,0,i),~i=0-3$, $(5,i,5),~i=0-3$, $(i,j,k),~i,k=4,6,~j=0-3,5$ & Same.
\tabularnewline
& $M_5(2)$ & 
\raggedright
$(i,0,i),~i=0-3$, $(i,j,k),~i,k=4,8,~j=0,2,6$, $(i,j,k),~i,k=5,9,~j=0,2,6$, $(i,j,i),~i=6,7,~j=0,2$, $(i,j,k),~i,k=10,11,~j=0-9$
& $(i,0,i),~i=0,1,2,3,6,7$,\newline
$(4,2,8)$,$(8,2,4)$,$(5,2,9)$,$(9,2,5)$,$(10,6,11)$,\newline $(11,6,10)$
\tabularnewline
\hline
\end{tabular}
\captionsetup{width=0.98\linewidth}
\caption{\small bsc$^3$-support and its combinatorial fusion cover, for the word tensor in various groups and words. The numbers in the triplets are indices of bscs, under the same indexing scheme as that of Figure~\ref{f:A5}.
}
\label{f:A6} 
\vspace{0.3cm}
\end{figure} 
\renewcommand{\arraystretch}{1}

Proposition \ref{prop:fusion_rep_of_tensor_short} has the following two immediate consequences. 
\begin{corollary}\label{cor:g_or_h_power_1}
If $n_a(w)=1$ then only terms with $\phi=\zeta$ may appear in $\bscs^3(\delta_{G,w})$.
Similarly, if $n_b(w)=1$ then only terms with $\psi=\zeta$ may appear in $\bscs^3(\delta_{G,w})$.
In particular, if $n_a(w)=n_b(w)=1$ then $\bscs^3(\delta_{G,w}) =\{(\phi, \phi, \phi) :\: \phi \in \bscs(G)\}$.
\end{corollary}
$n_a(w)=n_b(w)=1$ includes the case of the usual group multiplication operation (up-to possible inversion), which was studied in earlier works. This will be treated more thoroughly in Section~\ref{s:5}. 
\begin{corollary}\label{cor:trivial_rep_in_tensor}
    The only element in $\bscs^3	(\delta_{G,w})$ which has $\zeta=\text{Triv}$ is $(\text{Triv},\text{Triv},\text{Triv})$.
\end{corollary}

\subsection{The rank of word tensors}
By definition, we can decompose $\delta_{G,w}$ as
\begin{equation}
	\delta_{G,w} = 	\sum_{(\phi,\psi,\zeta) \in \bscs^3(\delta_{G,w})}
		\big(\delta_{G,w}\big)_{\phi \otimes \psi \otimes \zeta} \,,
\end{equation}
where $(\delta_{G,w})_{\phi \otimes \psi \otimes \zeta}$ is the projection of the word tensor onto the subspace associated with $\phi \otimes \psi \otimes \zeta$. Then, 
Lemma~\ref{obs:tensor_rank_naiv1_bound} immediately give the following bound on the rank of $\delta_{G,w}$:
\begin{equation}
\label{e:A35}
	\rank\big(\delta_{G,w}) \leq 
		\sum_{(\phi,\psi,\zeta) \in \bscs^3(\delta_{G,w})}
		\mmin\{D_\phi,D_\psi,D_\zeta\}
		\,,
\end{equation}
where, henceforth, we write $\mmin\{a,b,c\}$ as a short for $\min \{ab, ac, bc\}$ and we recall that $D_\phi$ denotes the dimension of the subspace associated with $\phi$.
While this bound is already often better than the trivial bound of $|G|^2$ on the rank of the word-tensor, it can be improved upon by merging together bscs of $G$.

To this end, given $\emptyset \neq \Phi,\Psi,\Xi \subseteq \bscs(G)$, we shall call the set $B := \Phi \times \Psi \times \Xi$, a {\em box}. 
A collection of $k \geq 1$ boxes forms a {\em box-set}: $\cB := \{B_i :\: 1 \leq i \leq k\}$. 
A box-set $\cB$ is {\em dominated by}
a box-set $\cB' = \{B'_{i'} :\: 1 \leq i' \leq k'\}$ if there exists map $\varphi: \{1,\dots, k\} \to \{1, \dots, k'\}$ such that $B_i \subseteq B'_{\varphi(i)}$ for all $i \leq k$. The box set $\cB$ is {\em smaller than} $\cB'$ if $\cB$ is dominated by $\cB'$ and, in addition, the above map $\varphi$ is injective. Both relations define a partial order on box-sets.
A box-set $\cB$ {\em covers} $A \subseteq \bscs^3(G)$ if $A \subseteq \cup_{i=1}^k B_i$, in which case we shall often call $\cB$ a {\em box-cover} of $A$ and, abusively, write $A \subseteq \cB$. The box-set $\cB$ is a {\em minimal box cover} of $A$ if it covers $A$ and 
	there is no other box cover of $A$ which is smaller than $\cB$. Lastly, a box $B$ is called {\em thin} if at most one of $\{\Phi, \Psi, \Xi\}$ is the full $\bscs(G)$. A box-set $\cB$ is {\em thin} if all of its boxes $B_i$ are thin.

The {\em box-rank} of the box $B = \Phi \times \Psi \times \Xi$ is 
\begin{equation}
\brank(B) := \mmin \big\{D_\Phi,\, D_{\Psi},\, D_{\Xi}\big\} \,,
\end{equation}
where henceforth for $\Psi \subseteq \bscs(G)$,
\begin{equation}
D_{\Psi} \equiv \dim(\Psi) := \sum_{\psi \in \Psi} D_\psi\,.
\end{equation}
The {\em box-rank} of a box-set $\cB = \{B_i\}_{i \leq k}$ is 
\begin{equation}
\brank(\cB) := \sum_{i \leq k} \brank(B_i) \,.
\end{equation}
 Trivially, a box-rank does not increase under the ``smaller than'' relation for box-sets.
 Finally, the {\em box-rank} of a tensor $T \in (\bbR^G)^{\otimes 3}$
\begin{equation}
\label{e:A33}
\brank \big(T) := 
		\min \Big\{ \brank(\cB) :\: \cB \supseteq \bscs^3(T) \Big\} \,.
\end{equation}
Note that the box rank of a tensor depends only on its $\bscs^3$-support.
We shall call a minimizer of the right hand side above a {\em box-rank minimizing (box) cover} of $\bscs^3(T)$ and denote it by 
$\argbrank \big(T)$. While not every minimal box cover of $\bscs^3(T)$ is box-rank 
minimizing, the opposite must clearly hold.

A stronger version of~\eqref{e:A35} is therefore. 
\begin{proposition}
\label{p:5.4}
For a group $G$ and word $w$, 
\begin{equation}
	\rank\big(\delta_{G,w}) \leq \rank_\square \big(\delta_{G,w}) \,.
\end{equation}
\end{proposition}
\noindent
We remark that the {\em finest} box-set, $\bscs^3(\delta_{G,w})$ (with its elements thought of as singletons), and the {\em coarsest} box-set, $\bscs(G)^{\times 3}$ (thought of as a singleton), are always box-covers of $\bscs^3(\delta_{G,w})$. In fact $\bscs^3(\delta_{G,w})$ is a minimal box-cover and often so is $\bscs(G)^{\times 3}$. Nevertheless, the latter is not a minimizing box-cover, as shown by Corollary~\ref{cor:decomp_into_triv_non_triv}, and often this is also the case for $\bscs^3(\delta_{G,w})$.

We thus obtain an analytic method for bounding the tensor rank of $\delta_{G,w}$. This is done by solving the combinatorial optimization problem,
\begin{equation}
\label{e:A33'}
	\min \bigg\{
		\sum_{i=1}^k
		\mmin \Big\{D_\Phi,\, D_{\Psi},\, D_{\Xi}\Big\}
:\:   
\bigcup_{i=1}^k (\Phi_i \times \Psi_i \times \Xi_i)
\supseteq \bscs^3 \big(\delta_{G,w}\big) 
\ , \ \ 
\Phi_i, \Psi_i, \Xi_i \subseteq \bscs(G)
\,,\,\, k \geq 1
\bigg\} \,,
\end{equation}
Thanks to Proposition~\ref{prop:fusion_rep_of_tensor_short}, one may further replace in~\eqref{e:A33'} the quantity $\bscs^3 \big(\delta_{G,w}\big)$ by $\bscs^3_\CF \big(\delta_{G,w})$, which is much easier to compute via~\eqref{eq:bscs3_contained_fusion} and the fusion table of $G$. This gives a coarser, yet more accessible bound on the tensor rank of the word tensor.

Proposition \ref{prop:fusion_rep_of_tensor_short} implies that ranks of word tensors are always smaller than the na\"ive bound $|G|^2:$\begin{corollary}
\label{cor:decomp_into_triv_non_triv}
\begin{equation*}
	\rank(\delta_{G,w}) \leq |G|(|G|-1)+1 \,.
\end{equation*}
\end{corollary}
The table in Figure~\ref{f:A7} lists bounds on the rank of word tensors for various words and groups, which were obtained using~\eqref{e:A33'} and the bsc$^3$-supports that were calculated in Table \ref{f:A6}.
We see that for many groups and words, this method yields values which are considerably smaller than the bound in Corollary \ref{cor:decomp_into_triv_non_triv}. We thus state
\begin{gp}
	Ranks of word tensors are likely to be small ($\rank(\delta_{G,w}) \ll |G|^2$).
\end{gp}
 \renewcommand{\arraystretch}{1.3}
\begin{figure}[t!]
\setlength{\tabcolsep}{1pt}
\scriptsize
\begin{tabular}{|c|c|p{0.55\textwidth}|p{0.18\textwidth}|c|}
\hline 
w & G & Minimal box-covers of $\bscs^3(\delta_{G,w})$ & $\brank(\delta_{G,w})$ & $|G|^2$
\tabularnewline
\hline
\hline
$a^2b$ or $aba$ & $S_4$ &  
\raggedright
$B_1=\{0,2-4\}\times\{3,4\}\times\{3,4\}$, $B_2 = \{0-2\}\times\{2\}\times\{2\}$, $B_3 = \{0\}\times\{0,1\}\times\{0,1\}$
& $18^2+4^2+2=342$ & 576
\tabularnewline
& $D_8$ & 
\raggedright
$B_1 = \{0\}\times\{0-3\}\times\{0-3\}$, $B_2 = \{0-3\}\times\{5\}\times\{5\}$, $B_3 = \{0,1,5\}\times\{4\}\times\{4\}$, $B_4= \{0,1,5\}\times\{6\}\times\{6\}$, or $B_1,B_2$ and $B'_3=\{0,1,5\}\times\{4,6\}\times\{4,6\}$
& $4+16+16+16=52$ & 256
\tabularnewline
& $M_5(2)$ & 
\raggedright
$B_1=\{0\}\times\{0-3\}\times \{0-3\}$, $B_2=\{%0,
6\}\times\{4,5,8,9\}\times\{4,5,8,9\}$, $B_3=\{%0,
2\}\times\{6,7\}\times\{6,7\}$, $B_4=\{%0-,-
4,5,8,9\}\times\{10,11\}\times\{10,11\}$ or $B_1,B_2,B_3,B'_4=\{%0-
4,5,8,9\}\times\{10\}\times\{10\}$, $B'_5=\{%0-
4,5,8,9\}\times\{11\}\times\{11\}$
& $4+16+4+128=152$ \newline  & 1024
\tabularnewline
\hline
$aba^{-1}ba^2b^3ab^{-1}$ & $S_4$ & 
\raggedright
$B_1=\{0,1\}\times\{0\}\times\{0,1\}$, $B_2=\{0-2\}\times\{0-2\}\times\{2\}$, $B_3=\{0-4\}\times\{0-4\}\times\{3,4\}$
& $2+12+432=446$ & 576
\tabularnewline
& $D_8$ & 
\raggedright
$B_1=\{0-3,5\}\times\{0\}\times\{0-3,5\}$, $B_2=\{4,6\}\times\{0-3\}\times\{4,6\}$
& $8+32=40$ & 256
\tabularnewline
& $M_5(2)$ & 
\raggedright
$B_1=\{0-3,6,7\}\times\{0\}\times\{0-3,6,7\}$, $B_2=\{4,5,8,9\}\times\{2\}\times\{4,5,8,9\}$, $B_3=\{10,11\}\times\{6\}\times\{10,11\}$
& $8+8+32=48$ & 1024
\tabularnewline
\hline
\end{tabular}
\captionsetup{width=0.98\linewidth}
\caption{\small Minimal box-covers of the true bsc$^3$-support of $\delta_{G,w}$ and the box-rank of various words and groups.}
\label{f:A7} 
\vspace{0.3cm}
\end{figure}
 \renewcommand{\arraystretch}{1} 

\ToVerTwo{We end this subsection with two simple corollaries of Proposition~\ref{p:5.4}.
\begin{corollary}\label{cor:som1_sort_of_bound}
\begin{equation}
	\rank(\delta_{G,w}) \leq \dim(\zeta) \sum_{\zeta \in \bscs(G)} 
	\min \Big\{ \dim\big(\bscs(\zeta^{n_a(w)}) \big)\,,\, \,
					\dim \big({\bscs(\zeta^{n_b(w)})}\big) \Big\} \,.
\end{equation}      
\FromRan{double check: could there be repetitions?}
\end{corollary}\RT{I omit it from here - we don't use, and RHS requires defining exactly what I mean there}}

\subsection{The Hadamard Model}
In order to see what the theoretical findings of the previous two sections imply on the learning task at hand, we first switch to consider a variant of the TLP Model, which we call the {\em Hadamard Model} (HD), and in which $3$-tensors are more straightforwardly implemented. This model is similar to the TLP model, except that instead of applying an activation function on a linear combination of the $2|G|$ inputs, we perform a product of a linear combination of the first $|G|$ inputs, with a linear combination of the last $|G|$ inputs. Formally, for $m \geq 1$, given weights, 
\begin{equation}
\label{e:1.1}
W = (A,B,C) \quad ; \qquad A,B,C \in \bbR^{m\times G}\,,
\end{equation}
the model computes $f_{\rm HD}\big(\cdot; W\big) : \bbR^G \times \bbR^G \to \bbR^G$, given by
\begin{equation}
\label{e:3.21}
	f_{\rm HD}\big(u; W\big) 
	\equiv	f_{\rm HD}\big(x,y \,;\, A,B,C\big)
	:=  C^T(Ax \odot By\big) 
	\quad ; \qquad u = x|y \,,\,\, x,y \in \bbR^G \,,
\end{equation}
where, we recall that $\odot$ represents the Hadamard product of two vectors. Thus, for $x,y,z \in \bbR^G$, 
\begin{equation}
\label{e:1.25}
	f_{\rm HD}(x,y; A,B,C)^T z = \sum_{i=1}^m (Ax)_i(By)_i(Cz)_i \,.
\end{equation}
Notice that the weight space $\cW_G$ is as for the TLP model. See Figure~\ref{f:0} for a schematic diagram of the network. Note that the LHS of \eqref{e:1.25} is invariant under a simultaneous permutation of the rows of $A,B,C$. For this reason, we shall regard weights in $\cW_G$ which differ by such permutation as equivalent.

In the language of tensors, the HD model implements the $3$-tensor (over $\bbR$)
\begin{equation}
\label{e:A29}
T^\HD_W := \sum_{i=1}^m A_{i,:} \otimes B_{i,:} \otimes C_{i,:}
\ \in\ (\mathbb{R}^G)^{\otimes 3} \,,
\end{equation}
where $X_{i,:}$ denotes the $i$-th row of matrix $X$. Thus, the set of tensors which can be implemented by HD models with width $m$ is precisely the set of all $3$-tensors in $(\mathbb{R}^G)^{\otimes 3}$ of rank at most $m$. Formally,
\begin{equation}
\label{e:A42}
	\Big\{T^\HD_W :\: W \in \cW_{G,m} \Big\} = \Big\{T \in (\bbR^{G})^{\otimes 3} :\: \rank(T) \leq m \Big\} \,.
\end{equation}
We also define the 
{\em bsc$^3$ box-set} of an HD model with weights $W$ as 
\begin{equation}
\label{e:A123}
\bscs^3_\square(W) = \Big\{\bscs(A_{i,:}) \times \bscs(B_{i,:}) \times \bscs(C_{i,:}\big)\Big\}_{i=1}^m \,.
\end{equation}
It follows straightforwardly from~\eqref{e:A29} that the latter is a box cover of $\bscs^3(T^\HD_W)$, namely
\begin{equation}
\label{e:A124}
\bscs^3(T^\HD_W) \subseteq \bscs^3_\square(W) \,.
\end{equation}

Lastly, we have the following lemma which shows that expressive power of the TLP model with the square activation function $\sigma(s) = {\rm sqr}(s) = s^2$ is at least as strong as that of the Hadamard model.
\begin{lemma}
\label{l:1.4.1}
Fix a finite group $G$. Then, for any $W \in \cW_G$ there exists $W' \in \cW_G$ with $|W'| = 2|W|$ such that 
\begin{equation}
	f_{\rm TLP,\, sqr}(\cdot;W') = f_{\rm HD}(\cdot;W) \,.
\end{equation}
\end{lemma}
\begin{figure}[!t]	
\begin{minipage}[t]{0.6\linewidth}
\resizebox{\textwidth}{!}{
\begin{tabular}{|c|c|c|c|c|c|}
	\hline
	$w$ & $G$ & $N$ & Loss & Acc & Dominating box cover for $\bscs^3_\square(W_{\rm term})$
	\tabularnewline 
	\hline
	\hline
	$aab$ & $S_4$ & $64$ & $5.7\cdot 10^{-7}$ & $1$ & $B_1+B_2+B_3$ \tabularnewline
    & $D_8$ & $32$ & $1.4\cdot10^{-8}$ & $1$ & $B_1+B_2+B'_3$ \tabularnewline
          & $M_5(2)$ & $64$ & $7.1\cdot 10^{-8}$ & $1$ & $B_1+B_2+B_3+B_4+B'_5$\tabularnewline
          
	\hline
	$aba$ & $S_4$ & $64$ & $1.4\cdot 10^{-7}$ & $1$ & $B_1+B_2+B_3$ \tabularnewline
          & $D_8$ & $32$ & $2.3\cdot 10^{-8}$ & $1$ & $B_1+B_2+B'_3$ or $B_1+B_2+B_3+B_4$\tabularnewline
          & $M_5(2)$ & $64$ & $7.5\cdot 10^{-8}$ & $1$ & $B_1+B_2+B_3+B_4$\tabularnewline
          
	\hline
    \hline
	$aba^{-1}ba^2b^3ab^{-1}$ & 
    $S_4$ & $64$ & $5.9 \cdot 10^{-7}$ & $1$ & $\{0-2\}\times\{0-2\}\times\{0-4\}$\tabularnewline
    & & & & & $\qquad +\{0-4\}\times\{0-4\}\times\{3,4\}$\tabularnewline
    & $D_8$ & $32$ & $4.3\cdot10^{-9}$ & $1$ & $B_1+B_2$ \tabularnewline
    & $M_5(2)$ & $64$ & $8.7\cdot 10^{-8}$ & $1$ & $B_1+B_2+B_3$ \tabularnewline
	 \hline
\end{tabular}}
\end{minipage}\hfill
\begin{minipage}{0.4\textwidth}
\includegraphics[width=\textwidth]{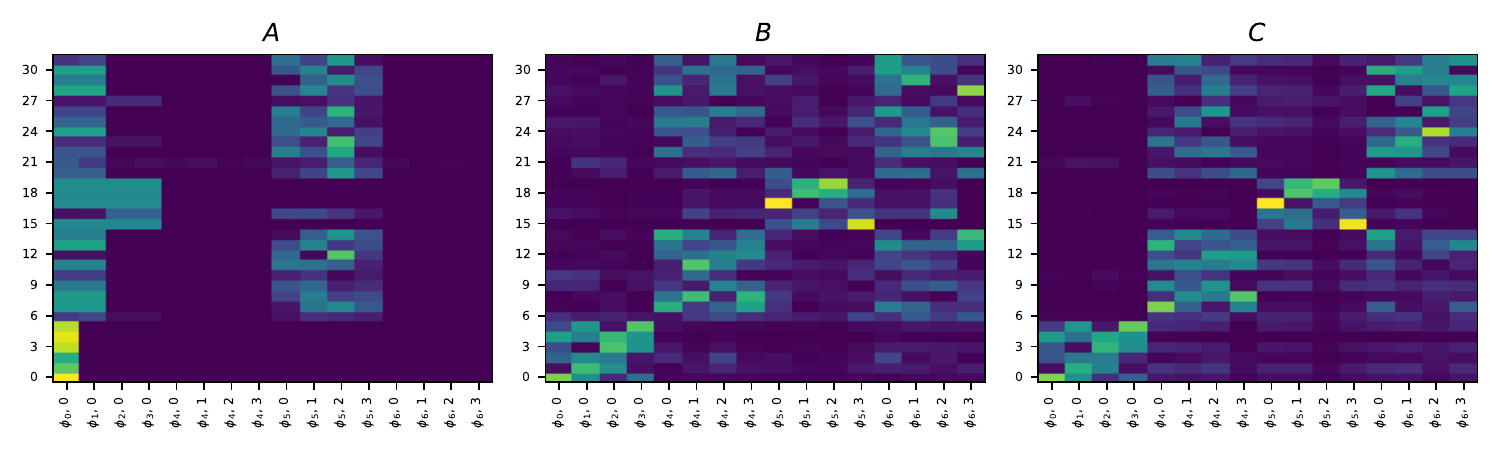}
\includegraphics[width=\textwidth]{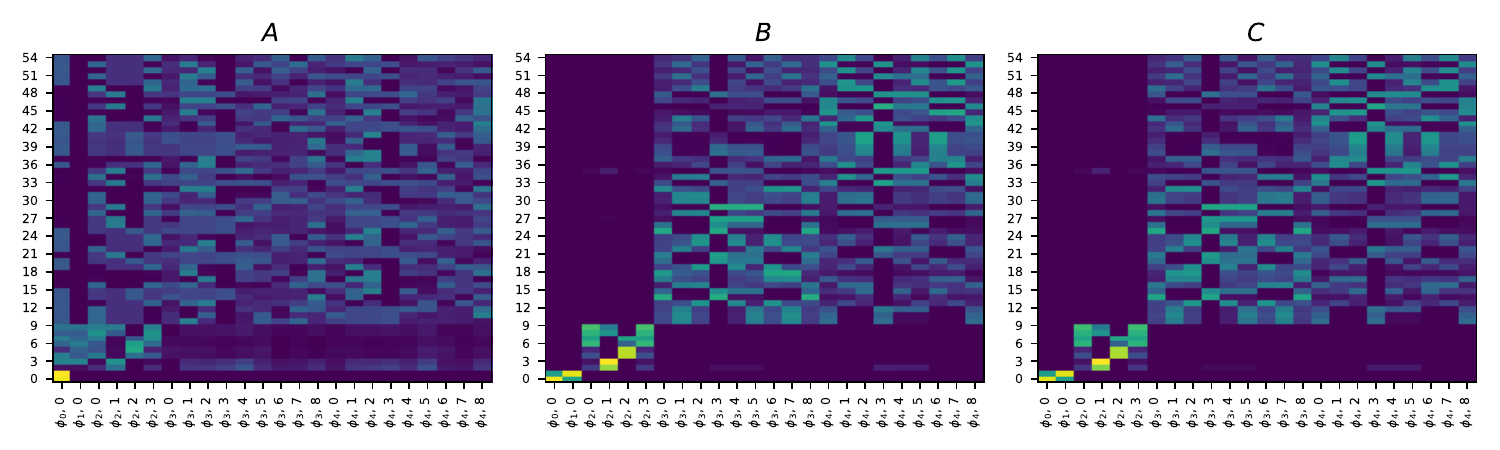}
\includegraphics[width=\textwidth]{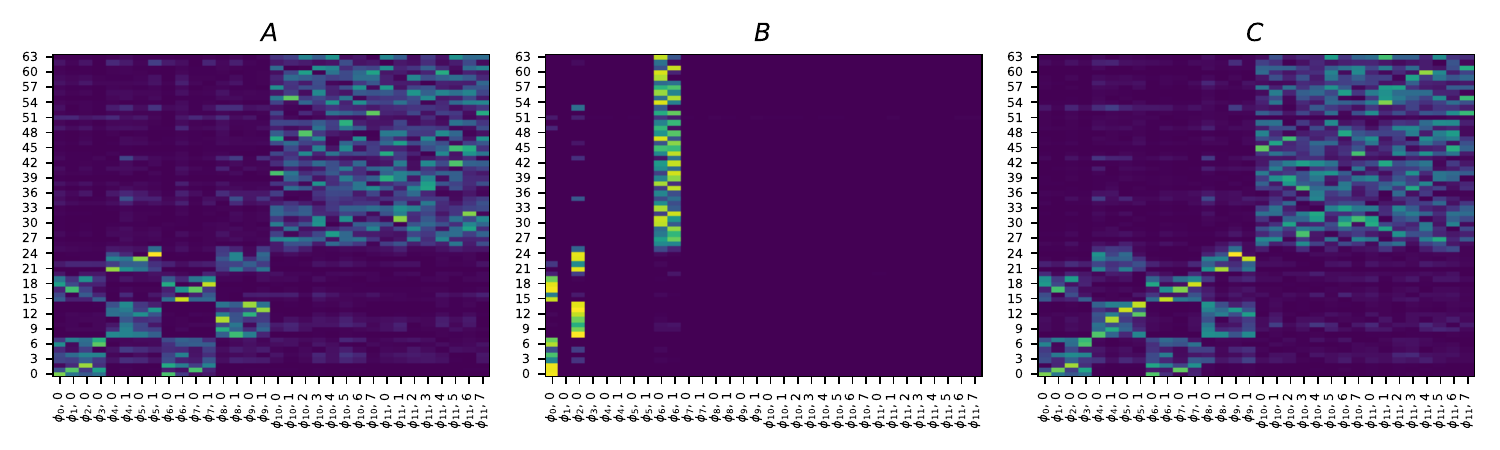}
\end{minipage}
\captionsetup{width=0.98\linewidth}
\caption{\small {\em Left:} Results of training the HD model on various words and groups. Observed bsc$^3$-box-sets of terminal weight indicated as last column and using the box notation from Figure~\ref{f:A7}.
{\em Right:} Projections of the terminal weights of the rows of $A$, $B$ and $C$ (Y-axis) on the matrix entries of all bscs of $G$ as $\bbR^G$-vectors (X-axis; entries of the same bsc are adjacent). $(w,G,N)$ are $(aab, D_8, 32)$ {\em (top)}, $(aba, S_4, 55)$ ({\em middle}) and $(aba^{-1}ba^2b^3ab^{-1}, M_5(2), 64)$ ({\em bottom}). The block structure of each matrix and alignment between rows of different matrices are apparent.
}
 \label{f:A8} 
\vspace{0.3cm}
\end{figure}

\subsection{Empirical study}

We trained the HD model with various widths $m$ and with the full data set $\cD_{G,w}$ for various groups $G$ and words $w$. Initialization and optimization was as indicated in Subsection~\ref{s:2.1.6}. In each case, we recorded the terminal loss and accuracy. To study the terminal weights, we projected each of the rows of matrices $A$, $B$ and $C$ in the terminal configuration $W_{\rm term}$ onto the subspaces of each of the bscs of $G$ and captured the results as heatmaps. The (empirical) bsc-support of each row was then deduced, whenever there was a clear separation between exhibited and non-exhibited bsc-components, and the (empirical) bsc$^3$-box-set of the weight configuration $W_{\rm term}$ was computed, as in~\eqref{e:A123}.

The results are summarized in the table of Figure~\ref{f:A8}. The figure also include heatmaps of the projections onto the matrix elements of all bscs of the group (as vectors in $\bbR^G$) of matrices $A$, $B$ and $C$ in the final weight configuration of 3 sample runs. More heatmaps and details on the results can be found in Appendix~\ref{s:B.1}. The data suggests the following general principle. 

\begin{gp}\label{gp:1}
$ $
\begin{enumerate}
\item If $\rank_\square(\delta_{G,w}) < m < |G|^2,$ then $\bscs^3_\square(W_{\rm term})$ is thin and dominated by a box cover of $\delta_{G,w}$ of rank smaller than $|G|^2$.
\item If, in addition, $\rank_\square(\delta_{G,w}) \ll |G|^2,$ then the above dominating box-cover is also a minimal.
\end{enumerate}
\end{gp}
\subsection{Generalization, grokking and the TLP model}
\label{s:4.6}
Lastly, we checked empirically whether the HD model reaches a generalizing solution given only a subset of the dataset, and moreover, whether the terminal weight configuration reached is the same as that in the case of the full sample set (albeit perhaps less pronounced). We also verified that the usual Grokking phenomenon is still exhibited when the learning task is of general group words. The answer to all of these questions turns out to be positive, as can be seen, for example. from Figure~\ref{f:A9}. We remark that both the maximal fraction of held train samples which still allow for full generalization and the level of pronunciation of the grokking features, depend on the group, word and width, and can be quite low in some cases. See Appendix~\ref{s:B.3} for additional plots.

Finally, initial experiments involving the TLP model with various activation functions indicate that a similar principle to the one above also holds in this case, albeit with more ``noise''  appearing in the terminal weight configuration. The box-cover which dominates the terminal weight configuration is often slightly different than the case of the HD model. We expect that the reason is that activation functions have low degree polynomial approximations, e.g. via Taylor expansion, and that replacing the activations by the approximations yield relatively low rank tensors which can be analyzed using the fusion tools developed in this work. See Subsection~\ref{s:7.3} for further discussion and Appendix~\ref{s:B.2} for heatmaps of the terminal weights under the TLP model.

\begin{figure}[t!]
\begin{minipage}{0.53\linewidth}
    \includegraphics[width=1\textwidth]{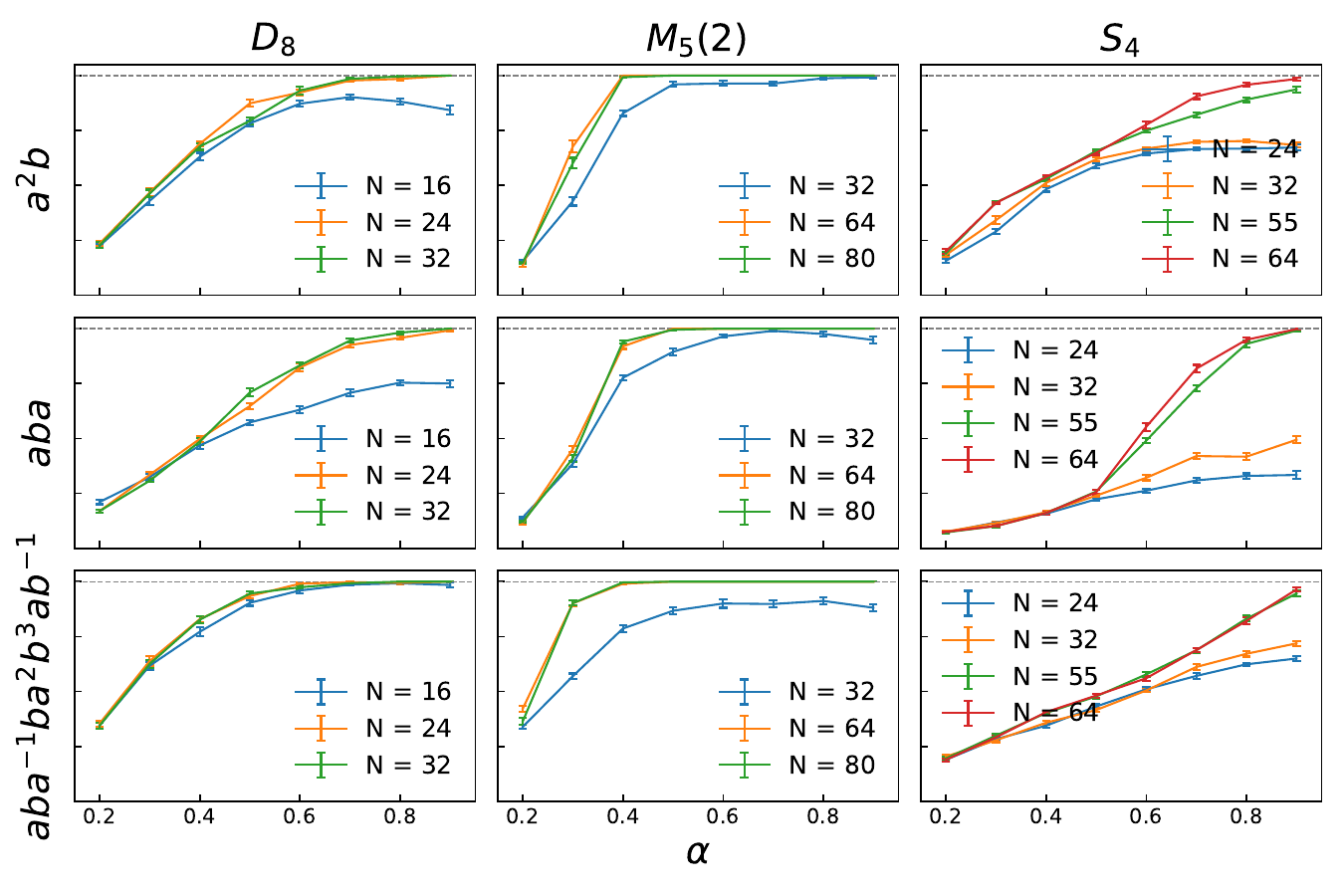}
\end{minipage}\hfill
\begin{minipage}{0.22\linewidth}
    % \centering
    \includegraphics[width=1\textwidth]{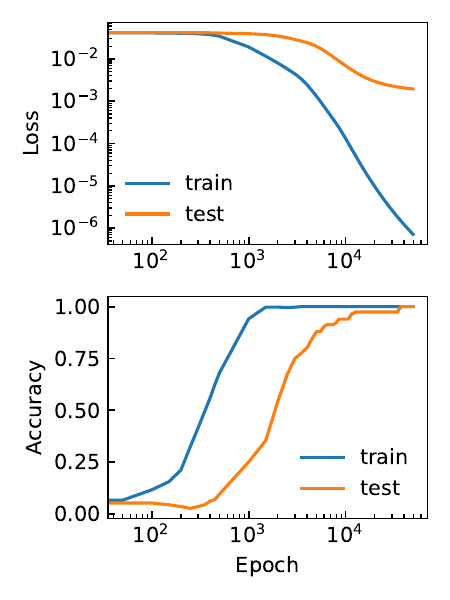}
\end{minipage}\hfill
\begin{minipage}{0.22\linewidth}
    \includegraphics[width=1\textwidth]{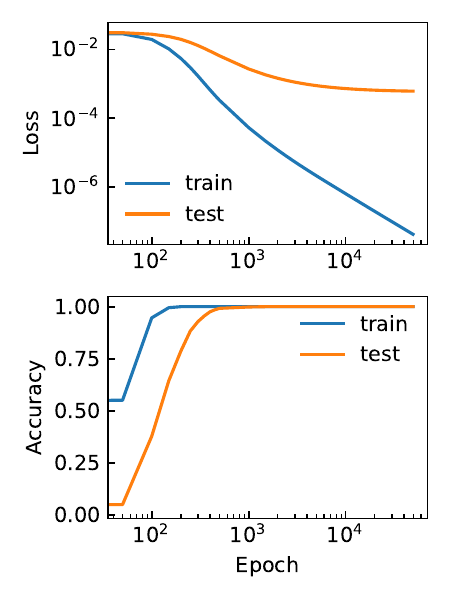}
\end{minipage}\hfill
\captionsetup{width=0.98\linewidth}
\caption{\small {\em Left:} Final test accuracy, for different groups, widths $N$ and train fractions, as average over $20$ runs of GD for the HD-model starting from a random initialization and using a random train-test split.
    Error bars mark one standard deviation. {\em Right:} The evolution of train/test loss/accuracy during training in one run for $G=S_4, w=aba,\alpha=0.8, N=64$ ({\em left column}) and $G=M_5(2), w=aba^{-1}ba^2b^3ab^{-1}, N=80, \alpha=0.4$ ({\em right column}). }    
\label{f:A9}
\end{figure}

\section{The case of group multiplication}
\label{s:5}
Next we restrict attention to the simplest word $w=ab$, in which case we denote the word tensor $\delta_{G,w}$ simply by $\delta_G$ and the full dataset $\cD_{G,w}$ by $\cD_{G}$.
 In this case the group operation is its usual ``multiplication''. The case of $G=\bbZ_p$ with addition modulo $p$ was studied by Gromov~\cite{gromov2023grokking}. The generalization to general groups was treated by Nanda et al~\cite{nanda2023progress}.
 In this part of the manuscript we study this problem using the tensor formalism developed in the previous section, and use the general theory to extend and refine the results of these earlier works.

\subsection{bsc$^3$-support of the word tensor}
In view of Corollary~\ref{cor:g_or_h_power_1} we see that $\bscs^3(\delta_G) = \{(\phi, \phi, \phi) :\: \phi \in \bscs(G)\}$, so that $\delta_G$ decomposes as the direct sum
\begin{equation}
\label{e:A147}
\delta_G = \sum_{\phi \in \bscs(G)} \delta_{G,\phi^{\otimes 3}} \,, 	
\end{equation}
where $\delta_{G,\phi^{\otimes 3}} \equiv (\delta_{G})_{\phi^{\otimes 3}}$ is the projection of the word tensor onto $R_{\phi^{\otimes 3}}$, henceforth a {\em single-bsc projection}. An explicit expression for the latter is given by the following proposition. Recall that $\chi_\phi$ is the character of representation $\phi$.
\begin{proposition}
\label{p:5.1}
For all $G$ and $\phi \in \bscs(G)$,
\begin{equation}
\label{e:4.2}
\big(\delta_{G, \phi^{\otimes 3}}\big)_{a,b,c}= 
\frac{\dim(R_\phi)}{d_\phi|G|}\,\chi_{\phi} \big(abc^{-1}\big)
\quad ; \qquad a,b,c \in G \,.
\end{equation}
\end{proposition}

\subsection{The rank of a single-bsc projection}
To bound the tensor rank of $\delta_{G,\phi^{\otimes 3}}$ we observe that $\chi_{\phi} \big(abc^{-1}\big)$ can be written as the trace of the matrix multiplication $\phi(a) \phi(b) \phi(c)^{-1}$. When $\phi$ is of types I, II, we may take a version of it with real valued matrix entries. In this case a na\"ive implementation is obtained by following the standard Gaussian method, namely
\begin{equation}\label{eq:tr_I_II}
\delta_{G, \phi^{\otimes 3}} = \frac{\dim(R_\phi)}{d_\phi|G|} \sum_{1 \leq i,j,k \leq d_\phi} \phi_{i,j} \otimes 
\phi_{j,k} \otimes (\phi^{-1})_{k,i} \,,
\end{equation}
where we recall that $\phi_{i,j}$ stands for the $(i,j)$ component of $\phi$ as a vector in $\bbR^G$. This decomposition yields the bound $d_\phi^3$ on the rank. A similar na\"ive implementation for the case when $\phi$ is of type III (and thus not real) gives the bound $2d_\phi^3$ on the rank (See Remark \ref{rmk:naiv1_decomp_quat} in Appendix~\ref{s:AA}).

The na\"ive decompositions are however not optimal, for two reasons. First, it is not difficult to see (e.g., from~\eqref{eq:tr_I_II}) that $\delta_{G,\phi^{\otimes 3}}$ is equivalent to the matrix multiplication tensor on the subspace of matrices spanned by $(\phi(g) :\: g \in G)$.
It is well known that the tensor rank $m_d$ of matrix multiplication for $d\times d$ matrices is less than $d^3$. This was first shown (albeit not in this terminology) by Strassen~\cite{strassen1969gaussian}. See also \cite{ottaviani2020tensor} for a more modern survey. Second, as this subspace may be a proper subspace of $\bbC^{d_\phi \times d_\phi}$, the restriction of the matrix multiplication tensor to this subspace may allow a further reduction of the tensor rank. The next proposition makes this precise.

\begin{proposition}
\label{p:1.5.1}
Denote by $m_d$ the tensor rank of matrix multiplication for real $d\times d$ matrices. Then with $d=d_\phi$,
\begin{equation}
\label{e:A1044}
	\rank(\delta_{G, \phi^{\otimes 3}}\big) \leq 
\begin{cases}m_{d},&\phi~\text{is of type I },\\
3m_{\frac{d}{2}},&\phi~\text{is of type II},\\
8m_{\frac{d}{2}},&\phi~\text{is of type III }.
\end{cases}
\end{equation}
\end{proposition}
Unfortunately, while $m_1 = 1$ and $m_2 = 7$, the precise value of $m_d$ for large $d$ is not known, nor the precise exponent of its asymptotic growth.

\subsection{Mono-bsc-aligned weight configurations}
In view of~\eqref{e:A42},~\eqref{e:A123} and~\eqref{e:A124}, the tensor $\delta_{G, \phi^{\otimes 3}}$ can be implemented by an HD model with width which is at least the rank of $\delta_{G, \phi^{\otimes 3}}$ and with all rows of $A$,$B$ and $C$ chosen from $R_\phi$, or equivalently with weights $W$ such that $\bscs^3_\square(W) = \{(\phi, \phi ,\phi)\}$.  It follows from~\eqref{e:A147} that the full word tensor $\delta_G$ can be realized by an HD model with $W$ such that $\bscs^3_\square(W) = \{(\phi, \phi, \phi) :\: \phi \in \bscs(G)\}$ and $|W_\phi| \geq \rank(\delta_{G, \phi^{\otimes 3}})$ for all bsc $\phi$, where $W_\phi$ denotes the weight vector obtained from $W$ by keeping only those rows in $A$,$B$ and $C$ which lie in $R_\phi$. The required width of the network can then be bounded using Proposition~\ref{p:1.5.1}.

Henceforth we shall call a weight vector $W$ satisfying $\bscs^3_\square(W) = \{(\psi,\psi,\psi) :\: \psi \in \Psi\}$ for some $\Psi \subseteq \bscs(G)$, a {\em mono-bsc-aligned} weight configuration with bsc-support $\Psi$ and, a bit abusively, write $\bscs(W) = \Psi$.  For such $W$ each row of $A$, $B$ and $C$ lies in a subspace of a unique bscs from $\Psi$ with the same bsc for corresponding rows of these matrices. If $\Psi=\{\psi\}$ then we say that $W$ is a {\em single-bsc} weight configuration (with bsc $\psi$).

\subsection{Loss decomposition and decoupling of dynamics}
The total loss~\eqref{e:1.15} on the full dataset $\cD_{G}$ for the HD model $f$ with weights $W$ can be written in tensor notation as 
\begin{equation}
	L_{f_\HD}(\cD_{G}; W) := \frac{1}{|G|^3} \sum_{a,b,c \in G} 
	\big(T^{\HD}_W - \delta_{G}\big)^2_{a,b,c} = \frac{1}{|G|^3} \big\|T^\HD_W - \delta_{G} \big\|^2_2 \,.
\end{equation}
This loss can be decomposed along tensor products of elements of $\bscs(G)^3$ by summing 
\begin{equation}
L_{f_\HD}(\cD_{G}; W) = \sum_{\bscs(G)^3} L_{f_\HD, (\phi, \psi, \zeta)}(\cD_{G}; W)
\quad    ; \quad 
	L_{f_\HD, (\phi, \psi, \zeta)}(\cD_{G}; W) = \frac{1}{|G|^3} \Big\|T^\HD_{W,\phi \otimes \psi \otimes \zeta} - \delta_{G, \phi \otimes \psi \otimes \zeta} \Big\|_2^2 \,,
\end{equation}
where $T^\HD_{W,\phi \otimes \psi \otimes \zeta}$ and  $\delta_{G, \phi \otimes \psi \otimes \zeta}$ are the respective projections of $T^\HD_W$ and $\delta_G$ onto $R_\phi \otimes R_\psi \otimes R_\zeta$. We shall refer to $L_{f_\HD, \phi, \psi,\zeta}(\cD_{G}; W)$ as the {\em bsc$^3$-loss} corresponding to $(\phi,\psi,\zeta)$.

Another useful decomposition of the loss is just along the bscs of the output:
\begin{equation}
\label{e:A253}
	L_{f_\HD}(\cD_G; W) = \sum_{\bscs(G)} L_{f_\HD,\phi}(\cD_G; W)
	\quad; \qquad 
	L_{f_\HD,\phi}(\cS; W) := \frac{1}{|G|^3} \sum_{a,b \in G} \Big\| f_{\HD, \phi}(a,b; W) - (1_{ab})_\phi \Big\|_2^2 \,,
\end{equation}
where $f_{\HD, \phi}(a,b; W) \equiv (f_\HD(a,b; W))_\phi$ and $(1_{ab})_\phi$ are the respective projections of the output and the label of each sample point onto $R_\phi$.
We shall refer to the $L_{f_\HD,\phi}(\cD; W)$ as the {\em bsc-loss} corresponding to $\phi$, or $\phi$-bsc-loss, for short.

The following follows from Corollary~\ref{cor:g_or_h_power_1} and Proposition~\ref{p:5.1}.
\begin{proposition}
\label{p:A15.3}
In order for the HD model with weights $W$ to have zero $\phi$-bsc-loss on the full dataset $\cD_G$, it is necessary and sufficient that
\begin{equation}
\label{e:A1049}
	f_{\HD, \phi}(a,b;W) = \bigg( \frac{\dim(R_\phi)}{d_\phi|G|}\,\chi_{\phi} \big(abc^{-1}\big) :\: c \in G \bigg) \quad ; \qquad a,b \in G \,.
\end{equation}
Moreover, zero total loss is obtained if and only if~\eqref{e:A1049} holds for all $\phi \in \bscs(G)$.
\end{proposition}
\noindent

If $W$ is mono-bsc-aligned, then
$L_{f_\HD, (\phi, \psi, \zeta)}(\cD_{G}; W)$ is zero unless $\phi=\psi=\zeta$, in which case it coincides with $L_{f_\HD, \phi}(\cD_{G}; W)$. In this case, we also have,
\begin{equation}
\label{e:A254}
	T^\HD_{W,\phi^{\otimes 3}} \equiv T^\HD_{W_\phi} 
\quad, \qquad 
	f_{\HD, \phi}(\,\cdot\,; W) \equiv f_{\HD}(\,\cdot\,; W_\phi) \,,
\end{equation}
and Proposition~\ref{p:A15.3} becomes,
\begin{corollary}
\label{c:A5.5}
Let $W$ be a mono-bsc-aligned weight configuration. Then $L_{f_\HD,\phi}(\cD_G; W) = 0$ if and only if
\begin{equation}
\label{e:A1024}
f_\HD(a,b;W_\phi) = \bigg( \frac{\dim(R_\phi)}{d_\phi|G|}\,\chi_{\phi} \big(abc^{-1}\big) :\: c \in G \bigg) 
	\quad ; \qquad a,b \in G \,.
\end{equation}
Moreover, $L_{f_\HD}(\cD_G; W) = 0$ if and only if $W$ has full bsc-support and~\eqref{e:A1024} holds for all $\phi \in \bscs(G)$.
\end{corollary}

The following proposition shows the stability of bsc-alignment under GD. Together with Decomposition~\eqref{e:A253} and~\eqref{e:A254} this gives the decoupling of the dynamics along different bscs of $G$. Recall that $W | W'$ denotes concatenation of (weight) matrices.
\begin{proposition}
\label{t:5.1}
Under the HD model, if $W$ is mono-bsc-aligned with bsc-support $\{\phi_1, \dots, \phi_k\} \subseteq \bscs(G)$ then so is ${\rm GD}_{\cD_{G}}^t(W)$ for all $t \geq 0$. Moreover, 
\begin{equation}
\label{e:A53}
	{\rm GD}^{t}_{\cD_G}(W) =
		{\rm GD}^{t}_{\cD_G}(W_{\phi_1})\,\big|\,
		{\rm GD}^{t}_{\cD_G}(W_{\phi_2})\big|\dots\,\big|\,
		{\rm GD}^{t}_{\cD_G}(W_{\phi_k}) \,.
\end{equation}
\end{proposition}

Remark~\ref{rmk:attractive} in the appendix shows that not only the decompositions into representations is stable under GD, it is also locally attractive, in the sense that under a GD step, a set of weights close enough to being decomposed into bscs, tends to become closer to such a decomposition.

\subsection{Empirical study}
\label{s:4.5.3}
\subsubsection{Single-bsc dynamics}
Decomposition~\eqref{e:A253} together with~\eqref{e:A254}, and the decoupling of the dynamics~\eqref{e:A53}, suggest that in order to understand the empirical evolution of the full network, one should study the latter when the weight space is restricted to single-bsc configurations. To this end, given a bsc $\phi \in \bscs(G)$, we ran the GD dynamics on the Hadamard model on the full dataset $\cD_G$, with the rows of matrices $A$,$B$ and $C$ randomly chosen (as discussed in~\ref{s:2.1.6}) from the subspace $R_\phi$, and with the target function being bsc-loss corresponding to $\phi$ in place of the full loss. Proposition~\ref{t:5.1} guarantees 
that under this initalization the rows of $A$,$B$ and $C$, will forever remain bsc-supported only on $R_\phi$ and thus the network effectively minimizes only the bsc-loss corresponding to that bsc.

\begin{figure}[!t]
\tiny
\renewcommand\theadfont{\tiny}
\begin{subfigure}{0.32\linewidth}
    \includegraphics[width=\textwidth]{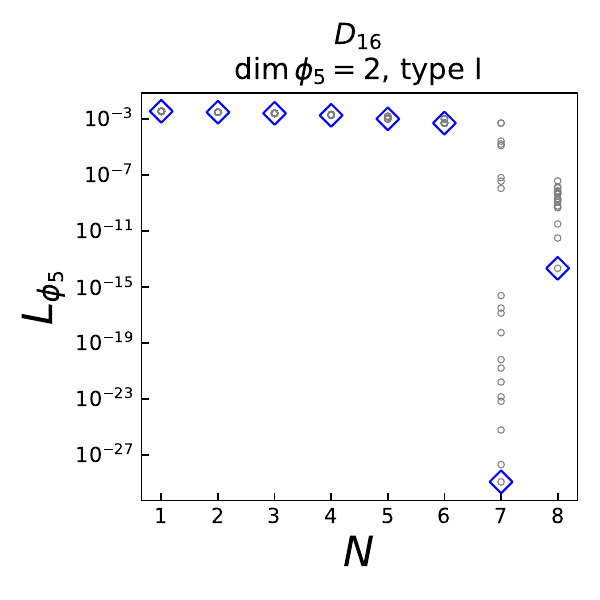}
\end{subfigure}
\begin{subfigure}{0.32\linewidth}
	\includegraphics[width=\textwidth]{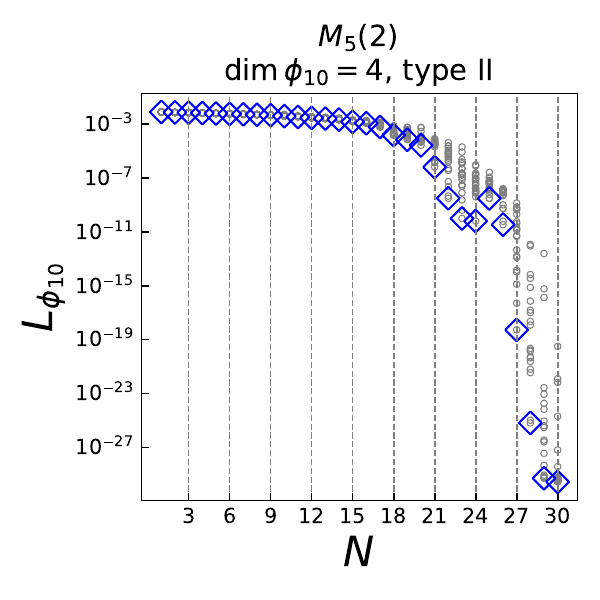} 
\end{subfigure}
\begin{subfigure}{0.32\linewidth}
    \includegraphics[width=\textwidth]{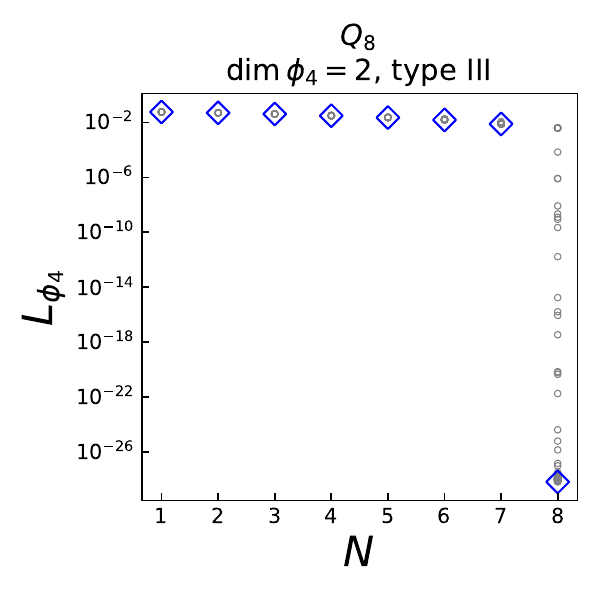} 
\end{subfigure}
\resizebox{\textwidth}{!}{
\begin{tabular}{|c|c|c|c|c|c|c|c|%c|
}
\hline 
Group & \thead{bsc} & $\dim$ & Type %& Faithful 
& \thead{Theoretical\\ min rows.} & \thead{Min rows \\for loss $<10^{-6}$ }& Accuracy & bsc-loss\tabularnewline
\hline 
\hline 
$D_{16}$ & $\phi_{4}$ & 2 & I %& Yes 
& 7 & 7 & 1 & $8.3\times10^{-26}$\tabularnewline
\hline 
$D_{16}$ & $\phi_{5}$ & 2 & I %& No 
& 7 & 7 & 0.5 & $1.2\times10^{-29}$\tabularnewline
\hline 
$S_{4}$ & $\phi_{4}$ & 3 & I %& Yes 
& $m_3$ & 23 & 1 & $3.2\times10^{-23}$\tabularnewline
\hline 
$\mathbb{Z}_{32}$ & $\phi_{2}$ & 2 & II %& Yes 
& $\leqslant3$ & 3 & 1 & $3.4\times10^{-31}$\tabularnewline
\hline 
$(\mathbb{Z}_{4}\times\mathbb{Z}_{2})\rtimes\mathbb{Z}_{2}$ & $\phi_{7}$ & 2 & II %& No 
& $\leqslant3$ & 3 & 0.25 & $4.3\times10^{-33}$\tabularnewline
\hline 
$M_{5}(2)$ & $\phi_{10}$ & 4 & II %& Yes 
& $\leqslant21$ & 21 & 1 & $6.2\times10^{-7}$\tabularnewline
\hline 
$Q_{8}$ & $\phi_{4}$ & 2 & III %& Yes 
& 8 & 8 & 1 & $6.4\times10^{-29}$\tabularnewline
\hline
\end{tabular}}
%\end{minipage}
\captionsetup{width=\textwidth}
\caption{{\em Top:} Terminal bsc-loss in repeated (20-100) runs of the model for various groups and bscs as a function of the width of the network, with initial weights chosen randomly from $R_\phi$. The minimal loss is marked with a blue diamond. {\em Bottom:} Minimal number of rows needed in order to have at least one run (among 20-100 tried) with terminal bsc-loss $<10^{-6}$. Accuracy and bsc-loss are those at the end of one such run. The theoretical minimal number of rows for achieving noticeably low bsc-loss is computed using Proposition~\ref{p:1.5.1}}.
\label{f:2} 
\vspace{0.3cm}
\end{figure}
\ToVerTwo{I (Ran) removed the column of "faithful" since we no longer talk about it, since we saw that Nanda did it. If someone wants to brings back we need: define what it means + bring back the claim (to appendix)}
Figure~\ref{f:2} shows the terminal loss as a function of the width of the model, in repeated runs (20-100, depending on the group and bsc) of the above experiment for different choices of groups and bscs. As can be seen from the plots, as soon as the number of rows reaches the theoretical value given by the r.h.s. of~\eqref{e:A1044} in Proposition~\ref{p:1.5.1}, a run whose terminal loss is noticeably low was observed. A table summarizing the empirical minimal number of rows required to achieve a low bsc-loss per group and bsc is included in the figure as well. The results thug suggest:
\begin{gp}\label{gp:3}
$ $
\begin{enumerate}
	\item 
Starting from a single-bsc supported weight configuration, the HD model is able to implement the corresponding bsc$^3$ projected word tensor, thereby achieving zero bsc-loss for this bsc.
\item Moreover, the model is able to do so, as soon as the width is at least the the theoretical upper bound, given in Proposition~\ref{p:1.5.1}. 
\item In particular, the HD model finds Strassen-type low-rank representations for the bsc-projected word tensors of~\eqref{e:4.2}.
\end{enumerate}
\end{gp}

\begin{figure}[t!]
\setlength{\tabcolsep}{0pt}
\renewcommand{\arraystretch}{1}
\renewcommand\theadfont{\tiny}
	 
\begin{minipage}{.49\linewidth}
\tiny
\begin{tabular}{|c|c|c|c|c|c|c|c|}
\hline
 $N$ & \thead{A} & \thead{Total \\ loss} & \thead{ $\phi_{0}$\\I, $d = 1$}&\thead{ $\phi_{1}$\\I, $d = 1$}&\thead{ $\phi_{2}$\\I, $d = 1$}&\thead{ $\phi_{3}$\\I, $d = 1$}&\thead{ $\phi_{4}$\\III, $d = 2$}\tabularnewline 
\hline 
$8$ &\tiny{$1$} & \tiny{$0.04$} & \tiny{$1 \times 10^{-30}$} & \tiny{$1.1 \times 10^{-30}$} & \tiny{$9.3 \times 10^{-31}$} & \tiny{$1.1 \times 10^{-30}$} & \tiny{$0.031$}\tabularnewline
 & &  & \tiny{$1$} & \tiny{$1$} & \tiny{$1$} & \tiny{$1$} & \tiny{$4$}\tabularnewline 
\hline 
$12$ &\tiny{$1$} & \tiny{$6.5 \times 10^{-29}$} & \tiny{$9.5 \times 10^{-31}$} & \tiny{$9.6 \times 10^{-31}$} & \tiny{$9.8 \times 10^{-31}$} & \tiny{$9 \times 10^{-31}$} & \tiny{$6.2 \times 10^{-29}$}\tabularnewline
 & &  & \tiny{$1$} & \tiny{$1$} & \tiny{$1$} & \tiny{$1$} & \tiny{$8$}\tabularnewline 
\hline 
$16$ &\tiny{$1$} & \tiny{$1.5 \times 10^{-10}$} & \tiny{$1.9 \times 10^{-23}$} & \tiny{$4.2 \times 10^{-33}$} & \tiny{$1.9 \times 10^{-27}$} & \tiny{$1.8 \times 10^{-32}$} & \tiny{$7.6 \times 10^{-12}$}\tabularnewline
 & &  & \tiny{$1$} & \tiny{$1$} & \tiny{$1$} & \tiny{$1$} & \tiny{$11$}\tabularnewline 
\hline 
\end{tabular}
\end{minipage}\hfill
\begin{minipage}{0.51\linewidth}
    \includegraphics[width=1\textwidth]{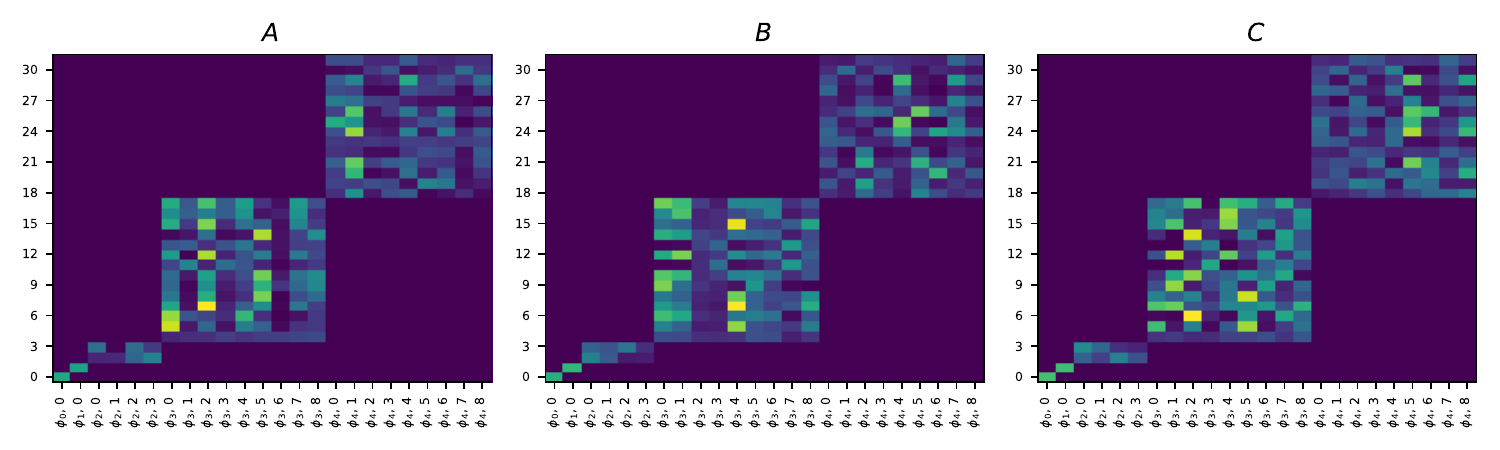}
\end{minipage}
\begin{minipage}{\linewidth}

\tiny
\begin{tabular}{|c|c|c|c|c|c|c|c|c|c|c|c|c|c|}
\hline
$N$ & \thead{A} & \thead{Total \\ loss} & \thead{ $\phi_{0}$\\I, $d = 1$}&\thead{ $\phi_{1}$\\I, $d = 1$}&\thead{ $\phi_{2}$\\I, $d = 1$}&\thead{ $\phi_{3}$\\I, $d = 1$}&\thead{ $\phi_{4}$\\I, $d = 2$}&\thead{ $\phi_{5}$\\I, $d = 2$}&\thead{ $\phi_{6}$\\I, $d = 2$}&\thead{ $\phi_{7}$\\I, $d = 2$}&\thead{ $\phi_{8}$\\I, $d = 2$}&\thead{ $\phi_{9}$\\I, $d = 2$}&\thead{ $\phi_{10}$\\I, $d = 2$}\tabularnewline 
\hline 
$32$ &\tiny{$1$} & \tiny{$0.013$} & \tiny{$3.5 \times 10^{-31}$} & \tiny{$3.6 \times 10^{-31}$} & \tiny{$3.8 \times 10^{-31}$} & \tiny{$3.5 \times 10^{-31}$} & \tiny{$0.0024$} & \tiny{$0.0018$} & \tiny{$0.002$} & \tiny{$0.0019$} & \tiny{$0.0018$} & \tiny{$0.0014$} & \tiny{$0.0012$}\tabularnewline
 & &  & \tiny{$1$} & \tiny{$1$} & \tiny{$1$} & \tiny{$1$} & \tiny{$3$} & \tiny{$4$} & \tiny{$4$} & \tiny{$4$} & \tiny{$4$} & \tiny{$4$} & \tiny{$5$}\tabularnewline 
\hline 
$53$ &\tiny{$1$} & \tiny{$0.002$} & \tiny{$3.7 \times 10^{-31}$} & \tiny{$3.7 \times 10^{-31}$} & \tiny{$3.8 \times 10^{-31}$} & \tiny{$3.8 \times 10^{-31}$} & \tiny{$8.3 \times 10^{-10}$} & \tiny{$1.6 \times 10^{-9}$} & \tiny{$2.1 \times 10^{-9}$} & \tiny{$0.00024$} & \tiny{$0.00049$} & \tiny{$1.3 \times 10^{-9}$} & \tiny{$2.3 \times 10^{-12}$}\tabularnewline
 & &  & \tiny{$1$} & \tiny{$1$} & \tiny{$1$} & \tiny{$1$} & \tiny{$7$} & \tiny{$7$} & \tiny{$7$} & \tiny{$7$} & \tiny{$6$} & \tiny{$7$} & \tiny{$7$}\tabularnewline 
\hline 
$72$ &\tiny{$1$} & \tiny{$7.2 \times 10^{-8}$} & \tiny{$1.1 \times 10^{-10}$} & \tiny{$3.5 \times 10^{-10}$} & \tiny{$2.4 \times 10^{-10}$} & \tiny{$2.4 \times 10^{-10}$} & \tiny{$6.3 \times 10^{-9}$} & \tiny{$7 \times 10^{-9}$} & \tiny{$1 \times 10^{-8}$} & \tiny{$7.6 \times 10^{-9}$} & \tiny{$7.9 \times 10^{-9}$} & \tiny{$9.8 \times 10^{-9}$} & \tiny{$9.8 \times 10^{-9}$}\tabularnewline
 & &  & \tiny{$1$} & \tiny{$1$} & \tiny{$1$} & \tiny{$1$} & \tiny{$10$} & \tiny{$9$} & \tiny{$9$} & \tiny{$9$} & \tiny{$9$} & \tiny{$9$} & \tiny{$10$}\tabularnewline 
\hline 
\end{tabular}
\end{minipage}	

\begin{minipage}{\linewidth}
\tiny
\begin{tabular}{|c|c|c|c|c|c|c|c|c|c|c|c|c|c|c|}
\hline
 $N$ & \thead{A} & \thead{Total \\ loss} & \thead{ $\phi_{0}$\\I, $d = 1$}&\thead{ $\phi_{1}$\\I, $d = 1$}&\thead{ $\phi_{2}$\\I, $d = 1$}&\thead{ $\phi_{3}$\\I, $d = 1$}&\thead{ $\phi_{4}$\\II, $d = 2$}&\thead{ $\phi_{5}$\\II, $d = 2$}&\thead{ $\phi_{6}$\\II, $d = 2$}&\thead{ $\phi_{7}$\\II, $d = 2$}&\thead{ $\phi_{8}$\\II, $d = 2$}&\thead{ $\phi_{9}$\\II, $d = 2$}&\thead{ $\phi_{10}$\\II, $d = 4$}&\thead{ $\phi_{11}$\\II, $d = 4$}\tabularnewline 
\hline 
$32$ &\tiny{$1$} & \tiny{$0.013$} & \tiny{$2.5 \times 10^{-33}$} & \tiny{$2.4 \times 10^{-33}$} & \tiny{$2.3 \times 10^{-33}$} & \tiny{$2.4 \times 10^{-33}$} & \tiny{$1.1 \times 10^{-32}$} & \tiny{$0.0005$} & \tiny{$0.0005$} & \tiny{$0.0005$} & \tiny{$0.00025$} & \tiny{$0.0005$} & \tiny{$0.0048$} & \tiny{$0.0055$}\tabularnewline
 & &  & \tiny{$1$} & \tiny{$1$} & \tiny{$1$} & \tiny{$1$} & \tiny{$3$} & \tiny{$2$} & \tiny{$2$} & \tiny{$2$} & \tiny{$2$} & \tiny{$2$} & \tiny{$8$} & \tiny{$6$}\tabularnewline 
\hline 
$64$ &\tiny{$1$} & \tiny{$0.00036$} & \tiny{$2.1 \times 10^{-33}$} & \tiny{$2.1 \times 10^{-33}$} & \tiny{$2.1 \times 10^{-33}$} & \tiny{$2.1 \times 10^{-33}$} & \tiny{$6.9 \times 10^{-33}$} & \tiny{$7.9 \times 10^{-33}$} & \tiny{$8.4 \times 10^{-33}$} & \tiny{$6.1 \times 10^{-33}$} & \tiny{$1.1 \times 10^{-32}$} & \tiny{$6.5 \times 10^{-33}$} & \tiny{$5.4 \times 10^{-5}$} & \tiny{$0.00018$}\tabularnewline
 & &  & \tiny{$1$} & \tiny{$1$} & \tiny{$1$} & \tiny{$1$} & \tiny{$3$} & \tiny{$3$} & \tiny{$3$} & \tiny{$3$} & \tiny{$3$} & \tiny{$3$} & \tiny{$21$} & \tiny{$19$}\tabularnewline 
\hline 
$80$ &\tiny{$1$} & \tiny{$4.4 \times 10^{-8}$} & \tiny{$1.5 \times 10^{-14}$} & \tiny{$8 \times 10^{-16}$} & \tiny{$6 \times 10^{-15}$} & \tiny{$5.8 \times 10^{-14}$} & \tiny{$9 \times 10^{-13}$} & \tiny{$3.8 \times 10^{-14}$} & \tiny{$2.4 \times 10^{-13}$} & \tiny{$1.2 \times 10^{-12}$} & \tiny{$1.8 \times 10^{-13}$} & \tiny{$4.1 \times 10^{-14}$} & \tiny{$6.2 \times 10^{-9}$} & \tiny{$1.6 \times 10^{-8}$}\tabularnewline
 & &  & \tiny{$1$} & \tiny{$1$} & \tiny{$1$} & \tiny{$1$} & \tiny{$4$} & \tiny{$4$} & \tiny{$4$} & \tiny{$3$} & \tiny{$4$} & \tiny{$4$} & \tiny{$26$} & \tiny{$25$}\tabularnewline 
\hline 
\end{tabular}
\end{minipage}
\captionsetup{width=0.98\linewidth}
\caption{\small {\em Tables:} Median terminal accuracy, total loss, bsc-loss and number of rows per bsc across 20 runs for various model widths and groups: 
$Q_8$ ({\em top}), $D_{16}$ ({\em middle}), $M_{5}(2)$ {\em (bottom)}.
{\em Plots:} Projection (in absolute value) of the terminal weights of the rows of matrices $A$, $B$ and $C$ (Y-axis) on the matrix entries of all bscs of the group $S_4$ as $\bbR^G$-vectors (X-axis). 
}% 
\label{f:3}
\end{figure}

\subsubsection{The full dynamics}
Next we ran the HD model on the full dataset with the full total loss and with a standard random initialization (as described in Subsection~\ref{s:2.1.6}) which does not restrict them to a single bsc subspace. The results for different groups are summarized in Figure~\ref{f:3} and lead to:
\begin{gp}\label{gp:4}
$ $ 
\begin{enumerate}
	\item 
Under GD for the HD model with standard initialization, the weights eventually converges to a mono-bsc-aligned terminal weight configuration $W$ with bsc-support $\bscs(W)$, which is determined, somehow, according to the initial assignment of weights. 
\item Under this configuration, for each bsc $\phi$ in the bsc-support, the corresponding bsc-loss is the minimal possible using $|W_\phi|$-many rows, i.e. it is similar to the loss achieved by a model with as many rows, which is initialized from values chosen only from $R_\phi$.
\item In particular, if $|W_\phi|$ is larger or equal to the rank of the corresponding bsc-projected word tensor $\delta_{G,\phi^{\otimes 3}}$, then the bsc-loss corresponding to $\phi$ will be noticeably low and $T^\HD_{W_\phi}$ will be essentially equal to $\delta_{G,\phi^{\otimes 3}}$.
\item This becomes more likely the larger the width of the model is. In particular, for a very large model width, typically all bscs are sufficiently represented in the terminal weight, resulting in a zero terminal total loss.
\end{enumerate}
\end{gp}

\section{Related work}
\label{s:2}
\subsection{Learning discrete operations}
Bivariate polynomials over $\bbZ_p$ were already studied by Power et al~\cite{power2022grokking} who showed that some polynomials could be learned using the same transformer based architecture, while others could not. Using a more layers was shown to allow for learning general biivariate polynomials over the same field by Gromov et al~\cite{doshi2024grokking}, who used an MLP network, provided depth and width are tuned correctly.

\subsection{Efficient matrix multiplication}
There is a vast literature and on going study on the topic of efficient matrix multiplication and, more generally, bilinear function computation.~\cite{ottaviani2020tensor} is a good survey on the subject and lecture notes can be found, e.g., in~\cite{blaser2009complexity}. Using machine learning models to discover efficient matrix multiplication was pioneered in~\cite{fawzi2022discovering}, were the authors used the reinforcement learning model AlphaZero to discover efficient matrix multiplication for various matrix sizes and underlying fields. 

\subsection{Geometric deep learning}
Using models whose output is invariant or equivarient under the symmetries of the underlying dataset, as means to obtaining more efficient and accurate predictors, was shown to be a successful paradigm across many learning tasks. The latter include, but not limited to, image, sound and video processing, pattern recognition and graph algorithms. The use of mathematical group theory to design and study such models dates back, at least, to the 70s, with notable earlier works including~\cite{amari78, kanatani2012group,lenz1990group} for general ML models,~\cite{fukushima1980neocognitron, wood1996representation} for the case of neural networks, and~\cite{goller1996learning, sperduti1993encoding} for graph learning tasks. A recent survey on this subject can be found in the Book~\cite{bronstein2021geometric}.

\subsection{Grokking}
Following the discovery of Grokking in~\cite{power2022grokking}, there has been a surge of works by the community on this subject and shall not be able to survey all of them. One line of work on grokking involves reconstructing this phenomenon in new setups, i.e. in different models and for different learning tasks. Examples here include~\cite{fan2024deep}, where it is shown that grokking occurs for an MLP network of large depth used for classifying the MINST dataset, and~\cite{liu2022omnigrok}, where grokking is occurs for various ``real'' tasks involving images, text and molecules and under various ``real'' architectures such as LSTM and graph convolutional networks. Other synthetic datasets were treated by~\cite{barak2022hidden}, where the task is learning the parity of a sparse binary vector as a label, or~\cite{xu2023benign} where the data is XOR clustered.

Another direction of research focuses on explaining why grokking occurs. While a mathematically rigorous proof is only limited to the case treated in~\cite{xu2023benign}, the acceptable coarse picture (c.f.~\cite{kumar2023grokking,liu2022towards, nanda2023progress}), is that of a sharp transition between a lazy-learning phase to a feature learning phase. During the former, the model quickly finds a solution to fit to the training data, but this solution is not generalizable to the population dataset. Later in the dynamics, the model is able to learn a much lower rank (sparse feature) solution, which is able to fit the full dataset, and ultimately becomes the dominant component of the output of the model. The precise mechanism by which these two solutions are discovered, including the sharp transition between the ``memorization'' to the ``representation'' solutions, and the necessity for explicit regularization for this mechanism to work, are a subject of debate (see also~\cite{doshi2023grok, humayun2024deep, lyu2023dichotomy, mohamadi2024you, rubin2023grokking, thilak2022slingshot,  varma2309explaining}).

\section{Summary and Discussion}
\subsection{Summary}
In this work we showed that a simple two-layer network with standard activation functions can learn an arbitrary word operation in an arbitrary finite group $G$, provided enough width is available. Recasting the problem as that of learning a particular $\bbR^{|G| \otimes 3}$ tensor, we showed that this word-tensor is typically of low rank. A way to obtain low-rank implementations of the tensor, is by decomposing it along (the tensor product of the sub-spaces of) triplets of basic self-conjugate (real values analogs of the irreducible) representations of the group and then use the fusion structure of the group to rule out many of the components. 
Focusing mainly on a surrogate model (the Hadamard Model), which is easier to study, yet phenomenologically similar, we showed that the network finds (approximations of) such low-rank implementations, thereby able to use limited width to approximate the word-tensor in a generalizable way. In the case of the simple multiplication word, we further elucidate the form of these low-rank implementations, showing that the network effectively implements efficient matrix multiplication in the sense of Strassen~\cite{strassen1969gaussian} and also shed light on the mechanism by which the network reaches such solution under gradient descent.

\subsection{Global attractiveness of low-rank tensor implementations}
This work exposed a class of low-rank implementations for the word-tensor which the HD model can represent. The existence of such implementation is likely a necessary condition a model to able to represent a generalizing solution with limited width. This work did not address, however, the mechanism-by-which and reason-for-which such a solution is reached via gradient descent, for a general word. Mathematically, it is not clear why (approximations of) those low-rank sparse-bsc-support implementations of the word tensor appearing in Suggested General Principle~\ref{gp:1} should be globally (not just locally) attractive. A proof for this is missing even in the simple case of group multiplication on $\bbZ_p$. In the case of the multiplication word, the decoupling of the dynamics starting from a mono-bsc-aligned configuration reduces the analysis to the, seemingly tractable case of a single-bsc.

\subsection{Further study of The TLP model}
\label{s:7.3}
For the case of the TLP model, even the question of existence of a generalizing solution is only partially answered in this work. As remarked in Subsection~\ref{s:4.6}, by using low-degree polynomial approximation to the activation function, it seems that the TLP model is able to (approximately) represent tensors which belong to the subspace spanned by (the subspaces of) certain low-degree tensor products of the bscs of the group (the analog of the tensor product of a triplet of bscs, in the case of the HD model). It is thus plausible that the network will converge to (an approximation of) a low-rank implementation of (an approximation of) the word-tensor, which lies in that subspace. This leads to a reformulation of the notions of a box-cover and minimal box-cover from Section~\ref{s:4} using such low-degree tensor products instead of the original $3$-products, so that Suggested General Principle~\ref{gp:1} still holds. Preliminary results show that this is indeed the case (See Appendix~\ref{s:B.2}). Nevertheless, making this precise and statistically significant requires further work.

\subsection{Finding bsc (and irreducible) representations}
Lastly, we remark that a possible application of Suggested General Principles \ref{gp:3},~\ref{gp:4} and Proposition \ref{t:5.1}, it to numerically find the bsc representations of a given finite group (and thus the corresponding non-abelian Fourier Transform). Indeed, given a group $G$, in order to learn its representations one can run the Hadamard model until it converges to a terminal configuration $W_{\rm term}$, which would (approximately) be mono-bsc-aligned. Assuming sufficient width, the rows of $(W_{\rm term})_\phi$ (i.e. the rows of matrices $A$, $B$ and $C$) will span the subspace corresponding to $\phi$ for all of the bscs of $G$. The partition of the rows $W_{\rm term}$ according to different bscs can be obtained using the orthogonality of these latter subspaces. Moreover, thanks to Proposition \ref{t:5.1}, if we recover a few representations in full,
we can then run the algorithm with rows orthogonal to those representations, and recover the rest. This permits using less width and thus less computing power.
Using a complex-valued version of the problem (which empirically and analytically behave similarly), one can recover the irreducible representations of the group in the same way.
\subsection{Additional structure in terminal solutions}
Another interesting phenomenon which arises from examples, and can be seen in the heatmaps, is that there seem to be an additional structure in the terminal weight configuration, not explained by the combinatorial fusion data of the group and the covering of the support by boxes. For example, we can see that sometimes one or more components in the decomposition of a row of $A$, $B$ or $C$ along matrix entry vectors of a bsc in the bsc-support of the row, non-generically, vanish. We believe that this phenomenon is related to the multiplicities of bscs which appear in the matrix coefficient spaces, and that perhaps in many tensors the projections onto some of the (non-canonically defined) copies of the bscs vanish.
\bibliographystyle{abbrv}
\bibliography{Groups.bib}

\appendix

\section{Proofs of statements}
\label{s:AA}
\begin{proof}[Proof of Lemma \ref{obs:tensor_rank_naiv1_bound}]
Assume without loss of generality $a\leq b\leq c.$ Then $T$ can be written as $\sum_{i\in[a]}1_i\otimes T_i,$ where $1_i$ is the $i$th standard basis element of $\mathbb{R}^a$ and $T_i$ is a $2-$tensor in $\mathbb{R}^b\otimes \mathbb{R}^c,$ that is a $b\times c$ matrix. For matrices the notion of tensor rank agrees with the usual notion of rank. For a $b\times c$ matrix the rank is upper bounded by $\min(b,c)=b,$ and we can write $T_i = \sum_{j=1}^b v^i_j\otimes u^i_j.$ Thus,
\[T=\sum_{i=1}^a \sum_{j=1}^b 1_i\otimes v^i_j\otimes u^i_j.\]
\end{proof}

\subsection{Proofs for Section~\ref{s:4}}

\begin{proof}[Proof of Proposition~\ref{prop:fusion_rep_of_tensor_short}]
This proposition is an immediate consequence of the following proposition, whose proof is given below.
\begin{proposition}\label{prop:fusion_rep_of_tensor}
For every bsc $\phi$ of dimension $D=D_\phi=\dim(R_\phi)$ fix a real orthonormal basis $B$ for $R_\phi$\[B=B(\phi)=\{v^1=v^1(\phi),\ldots,v^D=v^D(\phi)\}.\]
Then there exists a unique explicit representation
    \[\delta_{w(g,h)} = \sum_{\phi,\psi,\zeta\in{\rm bscs}(G)}\sum_{(i,j,k)\in[D_\phi]\times[D_\psi]\times[D_\zeta]}U_{ijk}(\phi,\psi,\zeta)v^i(\phi)\otimes v^j(\psi) \otimes v^k(\zeta),\]
    where the coefficients $U_{ijk}(\phi,\psi,\zeta)$ are zero unless
    $\phi\in{\rm bscs}(\zeta^{\otimes n_a(w)})$ and $\psi\in{\rm bscs}(\zeta^{\otimes n_b(w)}).$
\end{proposition}
\end{proof}
\begin{proof}[Proof of Proposition \ref{prop:fusion_rep_of_tensor}]
As usual, denote by $1_g$ the standard unit vector whose non zero entry is at the $g$th position.
Isomorphism \eqref{e:7} allows us to write
\begin{equation}\label{eq:decomp_of_unit}
\forall g\in G,\qquad 1_g = \sum_{\phi\in {\rm bscs}(G)}\sum_{i=1}^{D_\phi}\langle v^i(\phi),1_g\rangle v^i(\phi)=\sum_{\phi\in {\rm bscs}(G)}\sum_{i=1}^{D_\phi}v^i_g(\phi) v^i(\phi).
\end{equation}
Using \eqref{eq:decomp_of_unit} we can write
\begin{align}\label{eq:trans_four_delta}
\notag\delta_{G,w}&=\sum_{g,h\in G}\sum_{\phi,\psi,\zeta\in{\rm bscs}(G)}\left(\sum_{i=1}^{D_\phi}v^i_g(\phi)v^i(\phi)\right)\left(\sum_{j=1}^{D_\psi}v^j_h(\psi)v^j(\psi)\right)\left(\sum_{k=1}^{D_\zeta}v^k_{w(g,h)}(\zeta)v^k(\zeta)\right)\\
&=\sum_{\phi,\psi,\zeta\in{\rm bscs}(G)}\sum_{(i,j,k)\in[D_\phi]\times[D_\psi]\times[D_\zeta]}v^i(\phi)\otimes v^j(\psi) \otimes v^k(\zeta)\sum_{g,h\in G}
v^i_g(\phi)v^j_h(\psi)v^k_{w(g,h)}(\zeta)
\end{align}
We now simplify the coefficient of $v^i(\phi)\otimes v^j(\psi) \otimes v^k(\zeta).$
\begin{obs}\label{obs:word_tensor_in_terms_of_g_h}
    For every bsc $\zeta$ and $k\in [D_\zeta]$ there exist $N=N_{w,\zeta,k}$ and explicit homogeneous polynomials $P^w_\ell(x^1,\ldots,x^{D_\zeta})=P^{w,k,\zeta}_\ell(x^1,\ldots,x^{D_\zeta}),$ $Q^w_\ell(y^1,\ldots,y^{D_\zeta})=Q^{w,k,\zeta}_\ell(y^1,\ldots,y^{D_\zeta}),$ for $\ell=1,\ldots,N,$ of degrees $n_a(w),n_b(w)$ respectively, such that
    \[v^k_{w(g,h)}(\zeta) = \sum_{\ell=1}^NP^w_\ell(v^1_g(\zeta),\ldots,v^{D_\zeta}_g(\zeta))Q^w_\ell(v^1_h(\zeta),\ldots,v^{D_\zeta}_h(\zeta)).\]
\end{obs}
We will prove the observation below. Using the observation we can write
\begin{align*}U_{i,j,k}&=\sum_{g,h\in G}
v^i_g(\phi)v^j_h(\psi)v^k_{w(g,h)}(\zeta)=
\sum_{g,h\in G}
v^i_g(\phi)v^j_h(\psi)\sum_{\ell=1}^NP^w_\ell(v^1_g(\zeta),\ldots,v^{D_\zeta}_g(\zeta))Q^w_\ell(v^1_h(\zeta),\ldots,v^{D_\zeta}_h(\zeta))=\\
&=\sum_{\ell=1}^N\left(\sum_g P^w_\ell(v^1_g(\zeta),\ldots,v^{D_\zeta}_g(\zeta))v^i_g(\phi)\right)\left(\sum_h Q^w_\ell(v^1_h(\zeta),\ldots,v^{D_\zeta}_h(\zeta))v^j_h(\psi)\right)\\
&=\sum_{\ell=1}^N\langle P^w_\ell,v^i_g(\phi)\rangle\langle Q^w_\ell,v^j_h(\psi)\rangle,
\end{align*}
where $P^w_\ell$ is the vector whose $g$th entry is $P^w_\ell(v^1_g(\zeta),\ldots,v^{D_\zeta}_g(\zeta))$ and $Q^w_\ell$ is the vector whose $h$th entry is $Q^w_\ell(v^1_h(\zeta),\ldots,v^{D_\zeta}_h(\zeta))$. By  \eqref{obs:product_gives_fusion} $P^w_\ell$ belongs to $R_{\zeta^{\otimes n_a(w)}}$ and $Q^w_\ell$ belongs to $R_{\zeta^{\otimes n_b(w)}}.$ Since different bscs are orthogonal, it follows that the coefficients $U_{ijk}$ vanish unless the condition from the statement is satisfied. 
\end{proof}
In order to prove Observation \ref{obs:word_tensor_in_terms_of_g_h} we need the following observations.
\begin{obs}\label{obs:triv_in_fusion}If $\phi,\psi$ are bscs such that \text{Triv} is contained in $\phi\otimes\psi$ then $\phi=\psi.$
\end{obs}
\begin{proof}
In this case there exist real vectors $v\in R_\phi,~u\in R_\psi$ with 
$\langle u\odot v,\mathbf{1}>\neq0,$
where $\mathbf{1}$ is the all-$1$ vector spanning \text{Triv}.
This implies
\[\sum_{g\in G}u_gv_g\neq0\Leftrightarrow\langle u,v\rangle\neq0,\]which implies, since different bscs are orthogonal, that $\phi=\psi.$
\end{proof}
\begin{obs}\label{obs:inverse in the same rep}
Let $G$ be a group. Then the map $\text{inv}:\mathbb{R}^G\to\mathbb{R}^G$ which takes the vector $v$ to the vector $u$ whose $g$th entry is $v_{g^{-1}}$
takes every $R_\phi$, $\phi\in bscs(G)$ to itself. 
\end{obs}
\begin{proof}
Let $\phi$ be a bsc representation. Define 
\[\text{inv}(\phi)_g:=(\phi_{g^{-1}})^T.\]It is straightforward to see that $\text{inv}(\phi)$ is a representation, that $R_{\text{inv}(\phi)}=\text{inv}(R_\phi),$ and that also $\text{inv}(\phi)$ is a bsc. We must show that they are the same bsc. Since $\phi_g(\text{inv}(\phi)_{g})^T$ is the identity it follows from \eqref{obs:product_gives_fusion} that the trivial representation appears in the fusion product $\phi\otimes\text{inv}(\phi).$ From Observation \ref{obs:word_tensor_in_terms_of_g_h} this implies $\phi=\text{inv}(\phi).$
\end{proof}
\begin{proof}[Proof of Observation \ref{obs:word_tensor_in_terms_of_g_h}]
Let $\phi$ be a representation and $B=\{v^i=v^i(\phi),i=1\ldots,D_\phi\}$ an orthonormal basis. $\phi$ being a representation implies the existence of structure constants $r_{ij}^k=r_{ij}^k(B),$ which depend on $B$ and satisfy \begin{equation}\label{eq:prod_in_rep}v^k_{gh}=r^k_{ij}v^i_gv^j_h.\end{equation}Similarly, from Observation \ref{obs:inverse in the same rep}, there exist constants $s^k_i=s^k_i(B),$ depending on $B$ again, such that 
\begin{equation}\label{eq:inv_in_rep}v^k_{g^{-1}}=s^k_{i}v^i_g.\end{equation}
    We will prove by induction on the length $l(w)=n_a(w)+n_b(w)$ of the word.
    If $l(w)=1$ then $w$ is either $a,a^{-1},b,b^{-1}.$
    In this case $N=1$ and, using \eqref{eq:inv_in_rep}, it holds that
    \begin{equation}\label{eq:l_w_1}(P_1^{w,k,\zeta}(x^1,\ldots,x^{D_\zeta}),Q_1^{w,k,\zeta}(y^1,\ldots,y^{D_\zeta}))=\begin{cases}(x^k,1),&w=a\\
    (s^k_ix^i,1),&w=a^{-1}\\
    (1,y^k),&w=b\\
    (1,s^k_iy^i),&w=b^{-1}
    \end{cases}.\end{equation}
    Note that the degree constraints of the statement are met.
    If $l(w)>1$ then let $w'$ be the subword made of the first $l(w)-1$ letters, and $w''$ be the subword made of the last letter. Then by induction, for every $l\in [D_\zeta]$ there exist $N^l=N^{w',l,\zeta}$, and homogeneous polynomials $P^{w',l,\zeta}_\ell(x^1,\ldots,x^{D_\zeta}),$ $Q^{w',l,\zeta}_\ell(y^1,\ldots,y^{D_\zeta}),~\ell\in[N^l]$ of degrees $n_a(w'),n_b(w')$ respectively, such that
    \[v^l_{{w'}(g,h)} = \sum_{\ell=1}^{N^l}P^{w',l,\zeta}_\ell(v^1_g(\zeta),\ldots,v^{D_\zeta}_g(\zeta))Q^{w',l,\zeta}_\ell(v^1_h(\zeta),\ldots,v^{D_\zeta}_h(\zeta)).\] Also, by the induction base above, we have
    \[v^{m}_{{w''}(g,h)}=P^{w'',m,\zeta}(v^1_g(\zeta),\ldots,v^{D_\zeta}_g(\zeta))Q^{w'',m,\zeta}(v^1_h(\zeta),\ldots,v^{D_\zeta}_h(\zeta)),\] where the polynomials in the right hand side are given in \eqref{eq:l_w_1}. Using \eqref{eq:prod_in_rep} it holds that
    \begin{align*}v^k_{{w}(g,h)}&=\sum_{l,m=1}^{D_\zeta}r^k_{lm}v^l_{{w'}(g,h)}v^{m}_{{w''}(g,h)}\\
    &=\sum_{l,m=1}^{D_\zeta}r^k_{lm}\sum_{\ell=1}^{N^l}P^{w',l,\zeta}_\ell(v^1_g(\zeta),\ldots,v^{D_\zeta}_g(\zeta))Q^{w',l,\zeta}_\ell(v^1_h(\zeta),\ldots,v^{D_\zeta}_h(\zeta))\cdot \\&\qquad \quad\qquad \quad\qquad \cdot P^{w'',m,\zeta}(v^1_g(\zeta),\ldots,v^{D_\zeta}_g(\zeta))Q^{w'',m,\zeta}(v^1_h(\zeta),\ldots,v^{D_\zeta}_h(\zeta))\\
    &=\sum_{l,m=1}^{D_\zeta}\sum_{\ell=1}^{N^l}r^k_{lm}\left((P^{w',l,\zeta}_\ell P^{w'',m,\zeta})(v^1_g(\zeta),\ldots,v^{D_\zeta}_g(\zeta))\right)\left((Q^{w',l,\zeta}_\ell Q^{w'',m,\zeta})(v^1_h(\zeta),\ldots,v^{D_\zeta}_h(\zeta))\right).
    \end{align*}%where in the first two lines we have used the Einstein summation convention. 
    The sum in the last line consists of at most $|D_\zeta|^2\max_{l\in[D_\zeta]}N^l$ polynomials $P_\ell^{w',l,\zeta}P^{w'',m,\zeta}$ and $Q_\ell^{w',l,\zeta}Q^{w'',m,\zeta}$ satisfying the requirements. The induction follows.
\end{proof}

\ToVerTwo{\begin{proof}[Proof of Corollary \ref{cor:som1_sort_of_bound}]
We can cover the sets of triples which appear in the representation of the tensor given in Proposition \ref{prop:fusion_rep_of_tensor} by the boxes \[\{\zeta\}\times\mathrm{bscs}(\zeta^{\otimes n_a(w)})\times\mathrm{bscs}(\zeta^{\otimes n_b(w)}).\]The rank of such a box is upper bounded by $D_\zeta\min(D^\zeta_{n_a(w)},D^\zeta_{n_b(w)}).$
\end{proof}\RT{bring back if we bring back this corollary}}
\begin{proof}[Proof of Proposition~\ref{p:5.4}]
Let $\cB=\{B_1,\ldots,B_m\}$ be a box cover for $\mathrm{bscs}^3(\delta_{G,w}),$ and write $B_i=\Phi_i\times\Psi_i\times\Xi_i,~i=1,\ldots,m.$ Then we can write, e.g. using Proposition~\ref{prop:fusion_rep_of_tensor}, $\delta_{G,w}=\sum_{i=1}^mT_i,$ where each $T_i,~i=1,\ldots,m$ is a trilinear tensor in $\bbR^{\Phi_i}\otimes\bbR^{\Psi_i}\otimes\bbR^{\Xi_i}.$ By Lemma~\ref{obs:tensor_rank_naiv1_bound} the rank of $T_i$ is bounded by $rk(B_i),$ hence the rank of the whole tensor is bounded by the rank of the box set $\cB.$
\end{proof}
\begin{proof}[Proof of Corollary \ref{cor:decomp_into_triv_non_triv}]
Using Proposition \ref{prop:fusion_rep_of_tensor_short} and Corollary \ref{cor:trivial_rep_in_tensor} we can cover the representations appearing in each word tensor $\delta_{f=f_w(g,h)}$ using the two boxes
\[B_1=(\text{Triv},\text{Triv},\text{Triv}),~~B_2=(\mathrm{bscs}(G)\times \mathrm{bscs}(G)\times(\mathrm{bscs}(G)\setminus\{\text{Triv}\}).\] This also holds for linear combination of word tensors.
The rank of $B_1$ is $1$ and of $B_2$ is $|G|(|G|-1),$ which implies the result.
\end{proof}

\begin{proof}[Proof of Lemma~\ref{l:1.4.1}]
For $W=(A,B,C)\in\cW_G$ define $W'=(A',B',C')\in\cW_G$ by
\[A' = \frac{1}{2}\begin{pmatrix}
    A\\
    A
\end{pmatrix},~B' = \frac{1}{2}\begin{pmatrix}
    B\\
    -B
\end{pmatrix},~C' = \begin{pmatrix}
 C \\-C   
\end{pmatrix}.\]
Then for this choice of $W'$ it holds that
\[	f_{\rm TLP, sqr}(\cdot;W') = f_{\rm HD}(\cdot;W) \,.
\]Indeed, this is an immediate consequence of the identity consequence
\[\left(\frac{x+y}{2}\right)^2-\left(\frac{x-y}{2}\right)^2=xy.\]
\end{proof}

\subsection{Proofs for Section~\ref{s:5}}

\begin{proof}[Proof of Proposition \ref{p:5.1}]
The decomposition \eqref{e:A147} is well defined and uniquely determines the right hand side. \[\chi_\phi(abc^{-1})=Tr(\phi(abc^{-1})=\Tr\left(\phi(a)\phi(b)\phi^{-1}(c)\right),\]by the definition of characters and the defining property of representations.
\[\Tr\left(\phi(a)\phi(b)\phi^{-1}(c)\right) = \sum_{i,j,k \in [d_\phi]} \phi(a)_{i,j} 
\phi(b)_{j,k} (\phi(c)^{-1})_{k,i}. \]
Thus, the tensor $T\in (\bbR^{G})^{\otimes3}$ whose $(a,b,c)$ component is $\chi_{\phi}(abc^{-1})$ is
$\sum_{i,j,k \in [d_\phi]} \phi_{i,j} \otimes 
\phi_{j,k} \otimes (\phi^{-1})_{k,i}$ which clearly belongs to $R_\phi^{\otimes 3}.$ Thus, also the right hand side of \eqref{e:4.2} is in $R_\phi^{\otimes 3}.$

By uniqueness, in order to prove Proposition \ref{p:5.1} we just need to show that
\[\forall a,b,c\in G,~\sum_{\phi\in\bscs(G)}\frac{\dim(R_\phi)}{d_\phi|G|}\,\chi_{\phi}(abc^{-1}) =(\delta_{G,w})_{a,b,c},\]
or equivalently
\begin{equation}\label{eq:5.1_bsc}\forall g\in G,~\sum_{\phi\in\bscs(G)}\frac{\dim(R_\phi)}{d_\phi|G|}\,\chi_{\phi}(g) =\delta_{g=e},\end{equation}where $\delta_{g=e}$ is Kronecker's delta function.
Note that 
\[\frac{\dim(R_\phi)}{d_\phi|G|} =\begin{cases}
\frac{d_\phi}{|G|} & \phi\text{ of type I or III}\\
\frac{d_\phi}{2|G|} & \phi\text{ of type II}
\end{cases}.\]
If $\phi$ is a bsc of types I or III then it is irreducible, while if it is of type II then it is the sum of two conjugate irreducible representations $\psi,\bar\psi,$ and it holds that
\[\chi_\phi=\chi_\psi+\chi_{\bar\psi}.\]
Thus, we can equivalently write \eqref{eq:5.1_bsc} in terms of irreducible representations as
\begin{equation}\label{eq:5.1_irr}\forall g\in G,~\sum_{\phi~\text{is irreducible}}{d_\phi}\,\chi_{\phi}(g) =|G|\delta_{g=e}.\end{equation}
\eqref{eq:5.1_irr} is a standard fact in finite group representation theory, whose proof we now recall. 
The \emph{regular representation of $G$} is the representation $\phi_{\text{reg}}:G\to GL_{|G|}$ defined by
\[\phi_{\text{reg}}(g)1_h = 1_{gh}.\]
It is well known that
\[\phi_{\text{reg}} = \bigoplus_{{\phi~\text{is irreducible}}}d_\phi \cdot \phi,\]
that is, $\phi_{\text{reg}}$ is the sum of all irreducible representations $\phi$, each one appear with multiplicity $d_\phi.$
Thus
\[\chi_{\text{reg}}:=\chi_{\phi_{\text{reg}}}=\sum\sum_{\phi~\text{is irreducible}}d_\phi\chi_\phi.\]
On the other hand, \[\chi_{\text{reg}}(g)=|G|\delta_{g=e}.\]Indeed, if $g=e$ then $\chi_{\text{reg}}(e)$ is the $|G|\times|G|$ identity matrix, while if $g\neq e$ the matrix $\chi_{\text{reg}}(e)$ is a permutation matrix with zeroes on the diagonal, since for no $h\in G,~gh=h.$
Thus, \eqref{eq:5.1_irr}, hence also Proposition \ref{p:5.1}, follow. 
\end{proof}

\begin{rmk}\label{rmk:naiv1_decomp_quat}
     We can also write a representation of rank $2d_\phi^3$ for the tensor $\frac{1}{2}Tr\left(\phi(a)\phi(b)\phi(c^{-1})\right)$, when $\phi$ is of type III.  

Our starting point is the simple observation that every quaternionic representation is of even dimension $2|d$, and has a version of the form
\begin{equation}\label{eq:quaterninic_canonical_form}g\mapsto\phi(g)=\begin{pmatrix}
\alpha(g) &{\beta}(g)\\
-\overline{\beta(g)} &\overline{\alpha(g)}\end{pmatrix} \,,\end{equation} 
for some functions $\alpha, \beta: G \to \bbC^{d/2 \times d/2}$. 
Then
    \[\frac{Tr\left(\phi(a)\phi(b)\phi(c^{-1})\right)}{2}=\Re Tr\left(\alpha(a)\alpha(b)\alpha(c^{-1})-\beta(a)\overline{\beta(b)}\alpha(c^{-1})-\alpha(a)\beta(b)\overline{\beta(c^{-1})}-\beta(a)\overline{\alpha(b)}\overline{\beta(c^{-1})}\right)\]
    Write $\alpha=\alpha_1+\iota\alpha_2,~\beta=\beta_1+\iota\beta_2$ for their real and imaginary parts.
    We can expand the above expression in terms of $\alpha_1,\alpha_2,\beta_1,\beta_2$ to obtain
    \begin{align}\label{eq:long_exp}
&\alpha_1(a)\alpha_1(b)\alpha_1(c^{-1})-\alpha_1(a)\alpha_2(b)\alpha_2(c^{-1})-\alpha_2(a)\alpha_1(b)\alpha_2(c^{-1})-\alpha_2(a)\alpha_2(b)\alpha_1(c^{-1})+\\
\notag&-\beta_1(a)\beta_1(b)\alpha_1(c^{-1})-\beta_1(a)\beta_2(b)\alpha_2(c^{-1})+\beta_2(a)\beta_1(b)\alpha_2(c^{-1})-\beta_2(a)\beta_2(b)\alpha_1(c^{-1})+
\\\notag&
  -\alpha_1(a)\beta_1(b)\beta_1(c^{-1})-\alpha_1(a)\beta_2(b)\beta_2(c^{-1})-
  \alpha_2(a)\beta_1(b)\beta_2(c^{-1})
  +\alpha_2(a)\beta_2(b)\beta_1(c^{-1})+\\
  & \notag
  -  \beta_1(a)\alpha_1(b)\beta_1(c^{-1})+\beta_1(a)\alpha_2(b)\beta_2(c^{-1})-
  \beta_2(a)\alpha_1(b)\beta_2(c^{-1})
  -\beta_2(a)\alpha_2(b)\beta_1(c^{-1})
    \end{align}
    Realizing the trace of the above expression in the most na\"ive way as in \eqref{eq:tr_I_II} involves $16(d/2)^3=2d^3$ terms. 
\end{rmk}
\begin{rmk}\label{rmk:mamu}
    The trace tensor given in \eqref{eq:tr_I_II} can be identified with the matrix multiplication tensor.
We recall the reader that the matrix multiplication tensor for $d\times d$ matrices is the tensor $MaMu\in M_d^*\otimes M_d^*\otimes M_d,$ where $M_d$ is the $d^2$ dimensional vector space of $d\times d$ matrices, and $M_d^*$ is its dual space, given by
\begin{equation}\label{eq:Maumu}MaMu=\sum_{i,j,k\in[d]}E_{ij}^1\otimes E_{jk}^2\otimes E_{ik}^3,\end{equation}
where $E_{ij}^1$ is the functional on the first copy of $M_d$ defined by $E^1_{ij}(M)=M_{ij}.$ $E^2_{jk}$ is defined similarly. $E_{ik}^3$ is the $(i,k)$ elementary matrix in the third copy of $M_d.$ A rank $r$ representation of this tensor implies a rank $r$ representation of any tensor which has the form of the right hand side of \eqref{eq:Maumu}. The right hand side of \eqref{eq:tr_I_II} has the same form, if we take $d=d_\phi$, identify $E^1_{ij}$ with the functional which picks the $(i,j)$th
     entry of $\phi(a),~E^2_{jk}$ is the functional which picks the $(j,k)$th entry of $\phi(b),$ and $E^3_{ik}$ is replaced by the $(k,i)$th entry of $\phi^{-1}(c).$
     From this identification it follows that calculating the tensor of \eqref{eq:tr_I_II} is equivalent to calculating the matrix multiplication tensor on the space of matrices appearing in the representation $\phi.$
\end{rmk}
\begin{proof}[Proof of Proposition~\ref{p:1.5.1}]
We show that for the different types I,II,III the matrix multiplication tensors satisfy the prescribed bounds on ranks. By Remark \ref{rmk:mamu} this implies the same for the tensors of interest, which proves the proposition.

If $\phi$ is of type I, that is, real irreducible, then the claim is immediate, since the projected tensor realizes the matrix multiplication tensor of matrices with real entries.

Assume $\phi$ is of type II, that is, the bsc $\phi$ is isomorphic to $\psi\oplus\bar\psi,$ where $\psi$ is a complex irreducible representation and $\psi\neq\bar\psi.$ Thus, there is an invertible matrix $P$ with
\[\forall g\in G,~ \phi(g)=P^{-1}\begin{pmatrix}\psi(g)&0\\0&\bar{\psi}(g)\end{pmatrix}P\qquad\qquad\text{and }\phi(g)~\text{is a real matrix}.\]
Thus, there exists a \emph{real} invertible matrix $Q$ such that
\[\forall g\in G,~ \phi(g)=Q^{-1}\begin{pmatrix}1&1\\\iota&-\iota\end{pmatrix}^{-1}\begin{pmatrix}\psi(g)&0\\0&\bar{\psi}(g)\end{pmatrix}\begin{pmatrix}1&1\\\iota&-\iota\end{pmatrix}Q=Q^{-1}\frac{1}{2}\begin{pmatrix}\psi(g)+\bar{\psi}(g)&\iota(\psi(g)-\bar{\psi}(g))\\\iota(\bar\psi(g)-\psi(g))&\psi(g)+\bar{\psi}(g)\end{pmatrix}Q\]where $\iota=\sqrt{-1}.$
Note that the middle matrix in the right hand side is real. We may assume 
\[\forall g\in G,~ \phi(g)=\frac{1}{2}\begin{pmatrix}\psi(g)+\bar{\psi}(g)&\iota(\psi(g)-\bar{\psi}(g))\\\iota(\bar\psi(g)-\psi(g))&\psi(g)+\bar{\psi}(g)\end{pmatrix}\]since conjugating with $Q$ does not change the tensor rank.
We will show that in order to calculate $\phi(g)\phi(h)$ one needs only to apply three matrix multiplications of $\frac{d}{2}\times\frac{d}{2}$ matrices. 
To this end, note that
\begin{align*}
\phi(g)\phi(h)=
\frac{1}{4}\begin{pmatrix}\alpha_g\alpha_h-\beta_g\beta_h&-\alpha_g\beta_h-\beta_g\alpha_h\\\alpha_g\beta_h+\beta_g\alpha_h&\alpha_g\alpha_h-\beta_g\beta_h\end{pmatrix},~\text{where }\alpha_f=\psi(f)+\bar\psi(f),~\beta_f=\iota(\bar\psi(f)-\psi(f)).
\end{align*}We will show that we can calculate the two repeating (up to sign) entries $\alpha_g\alpha_h-\beta_g\beta_h$ and $\alpha_g\beta_h+\beta_g\alpha_h$ using matrix multiplications.
Indeed, 
\[\alpha_g\alpha_h-\beta_g\beta_h = \frac{\gamma_{g,h}+\delta_{g,h}}{2}-2\varepsilon_{g,h},\qquad\qquad\alpha_g\beta_h+\beta_g\alpha_h=\frac{\gamma_{g,h}-\delta_{g,h}}{2},\]
where 
\[\gamma_{g,h}=(\alpha_g+\beta_g)(\alpha_h+\beta_h),\qquad\qquad\delta_{g,h}=(\alpha_g-\beta_g)(\alpha_h-\beta_h),\qquad\qquad\varepsilon_{g,h}=\beta_g\beta_h.\]
Thus, we can write the product of $\phi(g)$ and $\phi(h)$ as linear combination of the three $\frac{d}{2}\times\frac{d}{2}$ matrix multiplications $\gamma_{g,h},\delta_{g,h},\varepsilon_{g,h},$ which easily yields a $3m_{\frac{d}{2}}$ representation for the matrix multiplication tensor restricted to $R_\phi\times R_\phi.$

The last case is that $\phi$ is of type III. In this case $d$ is even. We use the notations of Remark \ref{rmk:naiv1_decomp_quat}, and sketch the proof.
One can encode a quaternionic representation in terms of quaternions as follows
\[\phi(g)=\begin{pmatrix}\alpha_1(g)+\iota\alpha_2(g)&\beta_1(g)+\iota\beta_2(g)\\
-\beta_1(g)+\iota\beta_2(g)&\alpha_1(g)-\iota\alpha_2(g)
\end{pmatrix}\mapsto q(g):= \alpha_1(g)+i\alpha_2(g)+j\beta_1(g)+k\beta_2(g),\]
where the coefficients of $\alpha_1,\alpha_2,-\beta_1,-\beta_2$ in the above formal expression are $1,i,j,k\in Q_8,$ the quaterionic group mentioned above. Moreover, if we think of $q(g)$ as a quaternionic $\frac{d}{2}\times\frac{d}{2}$ then it is easy to see that $q(g)q(h)=q(gh),$ where for the product to make sense we make use of the quaternionic relations $i^2=j^2=k^2=ijk=-1$ (and $1$ commutes with $i,j,k$).
Naively multiplying $q(g),q(h)$ should use $16$ matrix multiplications. We will show how to perform it only using $8,$ using a well known analogous trick from multiplication of (standard) quaternions, see, e.g. \cite{4629959}. Then, if we realize each matrix multplication using a representation of tensor rank $m_{\frac{d}{2}}$, we obtain the claim.

Write
\begin{align*}
&m_1(g,h)= 2 \alpha_1(g)\alpha_1(h),\qquad\qquad m_2(g,h)=-2\beta_2(g)\beta_1(h)\\
&m_3(g,h)=-2\alpha_2(g)\beta_2(h),\qquad\qquad m_4(g,h)=-2\beta_1(g)\alpha_2(h)\\
&m_5(g,h)=\frac{1}{4}(\alpha_1(g)+\alpha_2(g)+\beta_1(g)+\beta_2(g))(\alpha_1(h)+\alpha_2(h)+\beta_1(h)+\beta_2(h))\\
&m_6(g,h)=\frac{1}{4}(\alpha_1(g)-\alpha_2(g)+\beta_1(g)-\beta_2(g))(\alpha_1(h)-\alpha_2(h)+\beta_1(h)-\beta_2(h))\\
&m_7(g,h)=\frac{1}{4}(\alpha_1(g)+\alpha_2(g)-\beta_1(g)-\beta_2(g))(\alpha_1(h)+\alpha_2(h)-\beta_1(h)-\beta_2(h))\\
&m_8(g,h)=\frac{1}{4}(\alpha_1(g)-\alpha_2(g)-\beta_1(g)+\beta_2(g))(\alpha_1(h)-\alpha_2(h)-\beta_1(h)+\beta_2(h))
\end{align*}
Then direct calculation shows that
\begin{align*}
&\alpha_1(gh)=m_1(g,h)-m_5(g,h)-m_6(g,h)-m_7(g,h)-m_8(g,h)\\
&\alpha_2(gh)=m_2(g,h)+m_5(g,h)-m_6(g,h)+m_7(g,h)-m_8(g,h)\\
&\beta_1(gh)=m_3(g,h)+m_5(g,h)+m_6(g,h)-m_7(g,h)-m_8(g,h)\\
&\beta_2(gh)=m_4(g,h)+m_5(g,h)-m_6(g,h)-m_7(g,h)+m_8(g,h).
\end{align*}
\end{proof}

\begin{proof}[Proof of Proposition~\ref{t:5.1}]	
The proposition is an immediate consequence of the stronger lemma:
\begin{lemma}\label{prop:decomp_technical}
    Let $R$ be the space of matrix coefficients of a sc representation of $G.$ Assume that there is a decomposition of the rows of the Hadamard model $\{1,\ldots, m\}=S_1\sqcup S_2$ such that for $i\in S_1$ the $i$th  rows of $A,B,C$, belong to $R$, and that the remaining rows are orthogonal to $R$. Then this property is preserved under the dynamics.
\end{lemma}
\begin{proof}
The loss function is given by
\[\sum_{g,h,f\in G}||\sum_{i=1}^Na_{i,g}b_{i,h}{c}_{i,f}-\delta_{f=gh}||^2_2.\]
Consider the derivative w.r.t to $a_{i,g}$, for $i\in S_1$. The derivatives with respect to other matrix entries have a similar form.
\begin{align}\label{eq:partial_a_ig}\frac{\partial}{\partial a_{i,g}}:~&\sum_{h,f\in G}b_{i,h}{c}_{i,f}\sum_{j}{a}_{j,g}{b}_{j,h}{c}_{j,f}-\sum_{h\in G}b_{i,h}{c}_{i,gh}\\\notag&=\sum_{j}{a}_{j,g}\left(\sum_{h\in G}b_{i,h}{b}_{j,h}\right)\left(\sum_{f\in G}c_{i,f}{c}_{j,f}\right)-\sum_{h\in G}b_{i,h}{c}_{i,gh},\end{align}
by the assumption that the rows indexed by $S_1$ are perpendicular to those of indexed by $S_2$ we see that the coefficient of ${a}_{j,g}$ in the first term vanishes unless $j\in S_1.$
Since the vector 
$({a}_{i,g})_{g\in G}\in R$, the first term belongs to $R.$ 

We need to show that also the vector $(\sum_{h\in G}b_{i,h}{c}_{i,gh})_{g\in G}\in R.$
Choose an orthonormal basis of vectors $v^1,\ldots,v^D$ for $R$, where $D=\dim(R)$.
We can write \[(b_{i,g})_{g\in G}=\sum_{l=1}^D\beta_l {v}^l, ~({c}_{i,g})_{g\in G}=\sum_{l=1}^D\gamma_l {v^l}.\]
Being a basis for a matrix coefficients space of a representation implies the existence of constants $r^i_{jk}$ such that
\begin{equation}\label{eq:quad_description}v^i_{gh}=r^i_{jk}v^j_gv^k_h,\end{equation}where we use the Einstein's summation convention.
Thus,
\[{c}_{i,gh}=\gamma_l r^l_{jk}v^j_gv^k_h,\]
hence
\begin{align*}
\sum_{h\in G}b_{i,h}{c}_{i,gh}&
% \sum_{h\in G}b_{i,h}\bar{c}_{i,gh}
=\sum_{h\in G}\beta_{l'} {v}^{l'}_h \gamma_l r^l_{jk}v^j_gv^k_h=(\sum_{l,l'}\beta_{l'}r^l_{jl'}\gamma_l)v_g^j,
\end{align*}where the last passage used orthonormality of the vectors $v^1,\ldots, v^D.$
Thus,\[(\sum_{h\in G}b_{i,h}{c}_{i,gh})_{g\in G}=(\sum_{l,l'}\beta_{l'}r^l_{jl'}\gamma_l)v^j\in R,\] as claimed.

We have shown that the at a point $W=(A,B,C)$ satisfying our requirements, the gradient for the rows indexed by $S_1$ is in $R,$ hence the dynamics will leave these rows at $R.$

Note that we did not require $R$ to be bsc. Thus $R$ can be an arbitrary sc representation. Now, since the rows of $S_2$ are orthogonal to those of $S_1$ they are also contained in a sum of the self conjugate representations not contained in $R$, by Equation \eqref{e:7}.  Thus, applying the previous part of the proof to the rows indexed by $S_2,$ shows that also they are left in the sum of the latter representations under the dynamics.
\end{proof}
\end{proof}
It turns out that if a generic $W$ nearly decomposes into bscs, then the dynamics step tends to reduce the error. We sketch this idea in the following remark.
\begin{rmk}\label{rmk:attractive}
With the notations of Proposition \ref{prop:decomp_technical}, for matrices $\tilde A,\tilde B,\tilde C$, we refer to the matrices obtained from them by ortho-projecting the rows labelled by $S_1$ to $R^\perp,$ and the remaining rows to $R$, as the \emph{normal error} (with respect to $R$ and the decomposition $[m]=S_1\sqcup S_2$).
Let $A',B',C'$ be weight matrices such that the $i$th rows of $A',B',C'$ for $i\in S_1$ belong to $R^\perp,$ and the remaining rows belong to $R.$ Assume that $A,B,C$ are \emph{generic} in the sense that at least one of the set of vectors $\{A_i\otimes B_i\}_{i\in [m]},\{A_i\otimes C_i\}_{i\in [m]}\{B_i\otimes C_i\}_{i\in [m]} $ is linearly independent. %, where $a_i,b_i,c_i$ are the rows of $A,B,C.$
This is indeed the generic case when the number of rows is much smaller than $|G|^2$, which is common to all cases studied in this work.\footnote{When the number of rows is of order $|G|^2$ linear algebraic reasoning allows other, not necessarily group theoretic, representations of the tensor.}
Then for small enough $\epsilon>0$ a gradient descent step applied to $(A+\epsilon A',B+\epsilon B',C+\epsilon C')$ reduces the normal error.
To see this, let $a'_{i,g},b'_{i,g},c'_{i,g}$
be the $g$th component of $A'_i,B'_i,C'_i$ respectively.
Then, similarly to \eqref{eq:partial_a_ig}, the partial derivative with respect to the $(i,g)$ component of $A+\epsilon A'$ is  \begin{align}\label{eq:partial_a_ig_1st_order}\sum_{j}({a}_{j,g}+\epsilon a'_{j,g})\left(\sum_{h\in G}(b_{i,h}+\epsilon b'_{i,h})({b}_{j,h}+\epsilon b'_{j,h})\right)&\left(\sum_{f\in G}(c_{i,f}+\epsilon c'_{i,f})({c}_{j,f}+\epsilon c'_{j,f})\right)\\\notag&-\sum_{h\in G}(b_{i,h}+\epsilon b'_{i,h})({c}_{i,gh}+\epsilon c'_{i,gh}).\end{align}
Suppose now that $i\in S_1.$ By Proposition \ref{prop:decomp_technical} the normal error of the zeroth order in $\epsilon$ vanishes.
For the first order, the same orthogonality arguments used in the proof of Proposition \ref{prop:decomp_technical} show that the first line of the above equation equals
\begin{align*}&\epsilon\sum_{j\in S_1} a'_{j,g}\left(\sum_{h\in G}b_{i,h}{b}_{j,h}\right)\left(\sum_{f\in G}c_{i,f}{c}_{j,f}\right)+
\\&+\epsilon\sum_{j\in S_2}{a}_{j,g}\left\{\left(\sum_{h\in G}(b_{i,h}b'_{j,h}+b'_{i,h}{b}_{j,h})\right)\left(\sum_{f\in G}c_{i,f}{c}_{j,f}\right)
+
\left(\sum_{h\in G}b_{i,h}{b}_{j,h}\right)\left(\sum_{f\in G}(c_{i,f}c'_{j,f}+ c'_{i,f}{c}_{j,f}) \right)
\right\},\end{align*}the ortho-projection to $R^\perp$ is thus \begin{equation}\label{eq:normal_grad}\epsilon\sum_{j\in S_1} a'_{j,g}\left(\sum_{h\in G}b_{i,h}{b}_{j,h}\right)\left(\sum_{f\in G}c_{i,f}{c}_{j,f}\right).\end{equation} The same analysis performed in the proof of Proposition \ref{prop:decomp_technical} shows that the second term of \eqref{eq:partial_a_ig_1st_order} has zero linear coefficient in $\epsilon.$
Thus, the $(i,g)$ component of gradient of the loss function in the normal direction is precisely the expression \eqref{eq:normal_grad}.

In order to show that a gradient descent step reduces the normal error, it is enough to show that the inner product of the gradient and the initial normal error is positive. Since the zeroth order of the gradient has zero normal error, it is enough to show this for the first order.
After dividing by $\epsilon^2m$ this inner product equals
\begin{align*}
\sum_{i,j}\langle A'_i ,A'_j\rangle\langle B_i ,B_j\rangle\langle C_i ,C_j\rangle&+\sum_{i,j}\langle A_i, A_j\rangle\langle B'_i ,B'_j\rangle\langle C_i ,C_j\rangle+\sum_{i,j}\langle A_i ,A_j\rangle\langle B_i, B_j\rangle\langle C'_i ,C'_j\rangle =\\
&=||\sum_{i}A'_i\otimes B_i\otimes C_i||^2+
||\sum_{i}A_i\otimes B'_i\otimes C_i||^2
+||\sum_{i}A_i\otimes B_i\otimes C'_i||^2.
\end{align*}
The genericity assumption guarantees that at least one of the expressions inside the square is non zero. %Thus, for small enough $\epsilon,$ the claim follows.
\end{rmk}
\newpage

\section{Additional material for Section~\ref{s:4}}\label{sec:appB}
\subsection{Additional heatmaps for the terminal weights of the HD model and their analysis}
\label{s:B.1}
In this section we examine instances of Suggested General Principle \ref{gp:1}. We study heatmaps of different words and different groups, and show that, up to negligible noise that could be explained in various ways, the box covers of the resulting networks satisfy the suggested general principle. In eight out of nine the second, stronger, property of a minimal thin box cover holds, while in the cover is only thin, hence satisfies the first, weaker, property of the principle.
\begin{ex}[The words $w=a^2b,~w=aba$.] 
$S_4:$ Figures \ref{fig:s4ggh} and \ref{fig:s4ghg}
show the heatmaps of the words $a^2b$ and $aba$ respectively for the group $S_4.$
In both cases $\bscs^3_\square(W)$ are exactly the thin minimal box cover given in Table \ref{f:A7}, %of Example \ref{ex:aba_rank}, 
that is\[B_1=\{0,2,3,4\}\times\{3,4\}\times\{3,4\},~~B_2=\{0,1,2\}\times\{2\}\times\{2\},~~B_3=\{0\}\times\{0,1\}\times\{0,1\}.\] 
There were a few experiments in which $\bscs^3_\square(W)$ were other minimal box covers. In some of these experiments $B_1$ has been replaced by $\{0,2,3,4\}\times\{3\}\times\{3\}$ and $\{0,2,3,4\}\times\{4\}\times\{4\}$ and in some $B_3$ has been replaced by $\{0\}\times\{0\}\times\{0\}$ and $\{0\}\times\{1\}\times\{1\}.$
\begin{figure}[H]
\tiny
\renewcommand\theadfont{\tiny}
    \includegraphics[width=\textwidth]{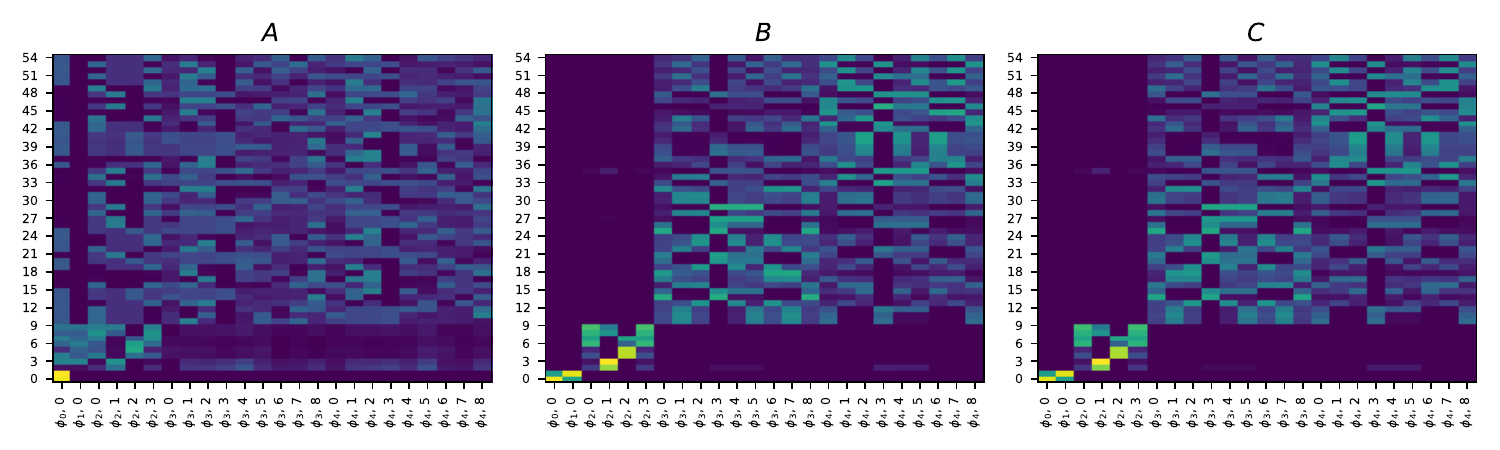}
\caption{$S_4$ and the word $aba$.}
\label{fig:s4ghg}
\end{figure}
\begin{figure}[H]
\tiny
\renewcommand\theadfont{\tiny}
    \includegraphics[width=\textwidth]{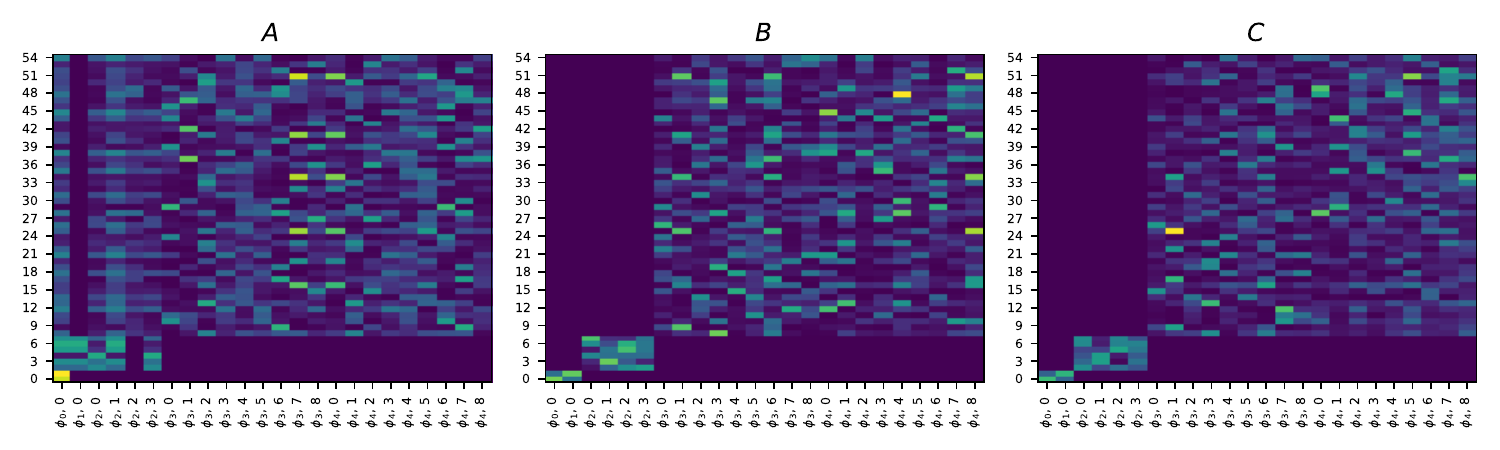}
\caption{$S_4$ and the word $a^2b$.}
\label{fig:s4ggh}
\end{figure}
\textbf{$D_8:$ }
Figures \ref{fig:d8ggh} and \ref{fig:d8ghg}
show the heatmaps of the words $aba$ and $a^2b$ for the group $D_8.$
Ignoring negligible noise 
$\bscs^3_\square(W)$ are readily seen to be dominated by the minimal box covers of %Example \ref{ex:aba_rank}
Table \ref{f:A7}. For $w=aba$
it is
\[
\begin{split}
& B_1=\{0\}\times\{0,1,2,3\}\times\{0,1,2,3\},~~B_2=\{0,1,2,3\}\times\{5\}\times\{5\},~~B_3=\{0,1,5\}\times\{4\}\times\{4\},\\
& B_4=\{0,1,5\}\times\{6\}\times\{6\},
\end{split}
\]
while for $w=a^2b$ it is $B_1,B_2,$ and $B'_3=\{0,1,5\}\times\{4,6\}\times\{4,6\}.$ The cover $B_1,B_2,B_3$ has appeared in all our experiments for the word $aba,$ while for the word $a^2b$ we saw the two different covers in different experiments. For most of the rows the corresponding boxes actually agree with the boxes in the box cover, and are not just being dominated.
\begin{figure}[H]
\tiny
\renewcommand\theadfont{\tiny}
    \includegraphics[width=\textwidth]{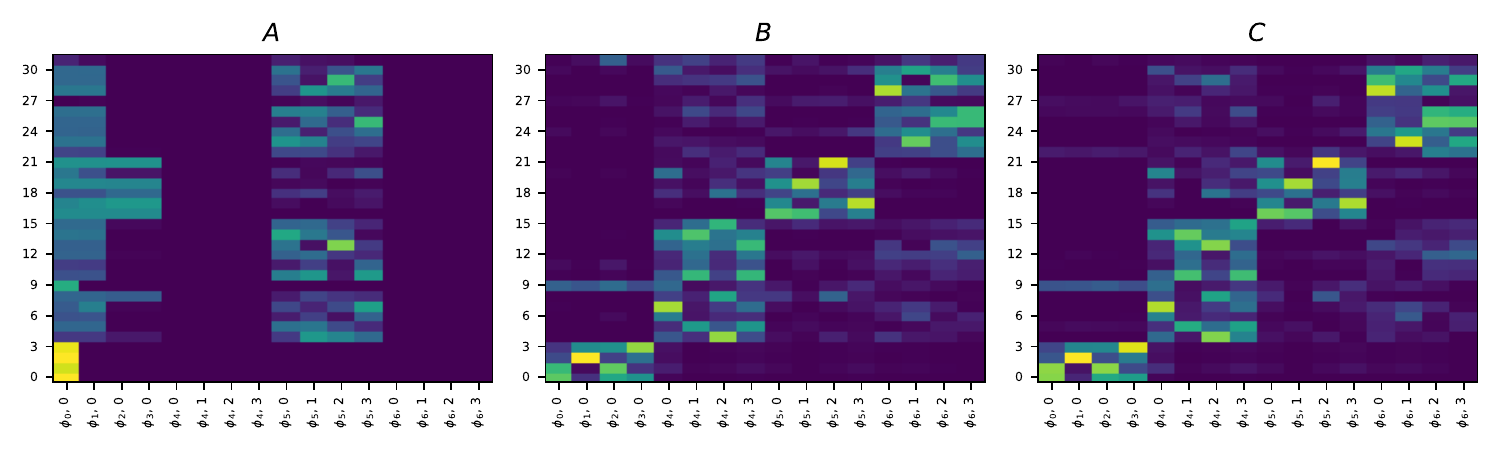}
\caption{$D_8$ and the word $aba$.}
\label{fig:d8ghg}
\end{figure}
\begin{figure}[H]
\tiny
\renewcommand\theadfont{\tiny}
    \includegraphics[width=\textwidth]{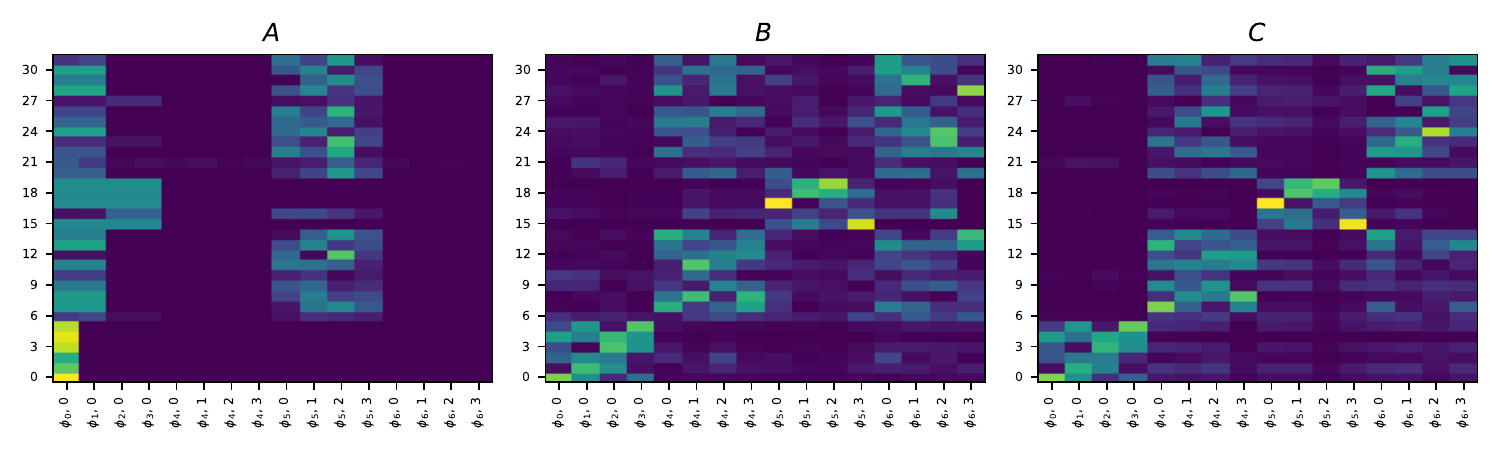}
\caption{$D_8$ and the word $a^2b$.}
\label{fig:d8ggh} 
\end{figure}

\textbf{$M_5(2):$ }
show the heatmaps of the words $aba$ and $a^2b$ for the group $M_5(2).$
$\bscs^3_\square(W)$ are dominated by the two minimal box cover of Table \ref{f:A7}. For $w=aba$
it is, in the notations of Table \ref{f:A7}, $B_1,B_2,B_3,B_4,$ and for $w=a^2b$ it is $B_1,B_2,B_3,B'_4,B'_5.$ Interestingly, unlike the $D_8$ case, here in all experiments we performed we saw the the first cover for the word $aba$ and the second for $a^2b.$ Again for most rows, in most experiments, the corresponding boxes agree with boxes in the box covers, and are not just being dominated by them.
Figures \ref{fig:m52ggh} and \ref{fig:m52ghg}
\begin{figure}[H]
\tiny
\renewcommand\theadfont{\tiny}
    \includegraphics[width=\textwidth]{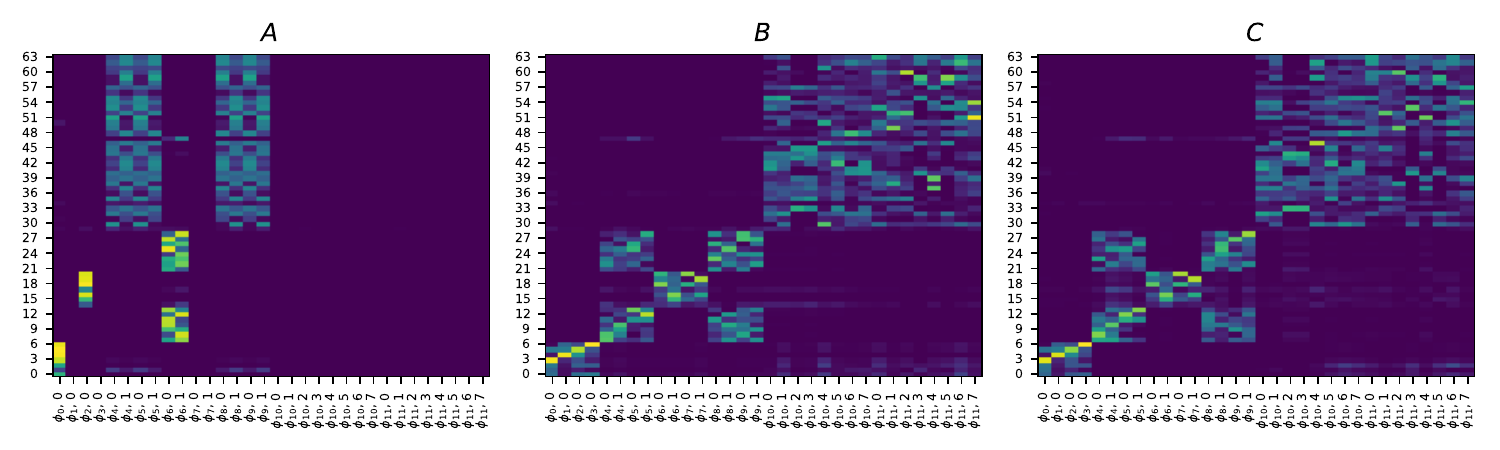}
\caption{$M_5(2)$ and the word $aba$.}
\label{fig:m52ghg}
\end{figure}
\begin{figure}[H]
\tiny
\renewcommand\theadfont{\tiny}
    \includegraphics[width=\textwidth]{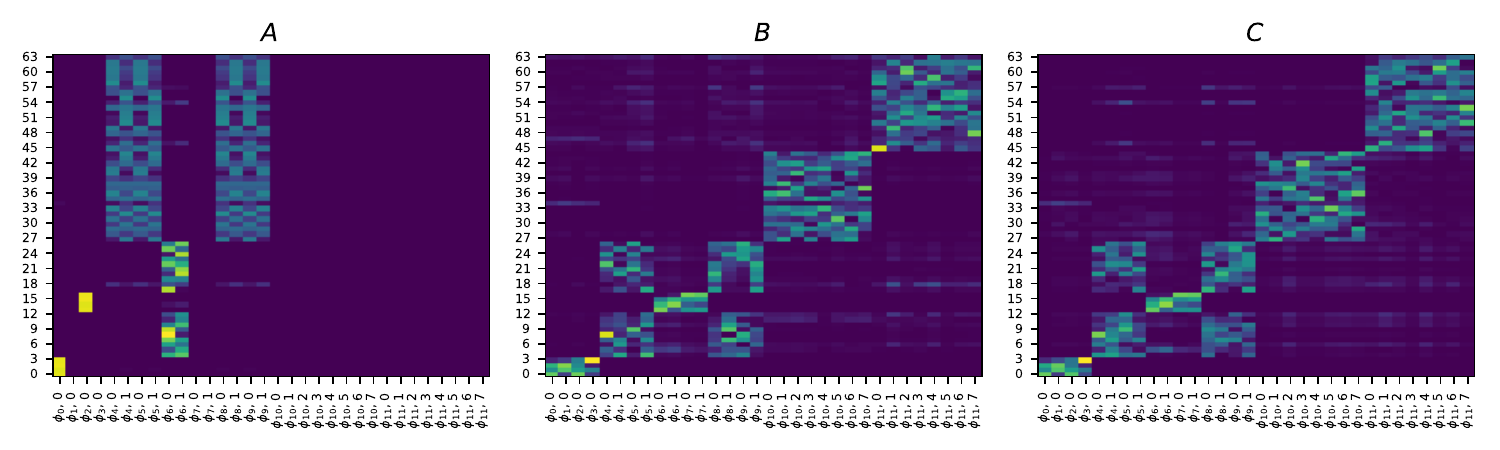}
\caption{$M_5(2)$ and the word $a^2b$.}
\label{fig:m52ggh}
\end{figure}
\end{ex}

\begin{ex}[The word $w=aba^{-1}ba^2b^3ab^{-1}$]
We now concentrate on $w=aba^{-1}ba^2b^3ab^{-1}$, and again test $S_4,D_8,M_5(2).$
\\$S_4:$ 
Figure \ref{fig:s4long_adamW} shows a heatmap for the Hadamard network which studies the word $w$ for the group $S_4.$
We can see in the heatmap rows which yield boxes $\{0,1,2\}\times\{0,1,2\}\times\{2,3,4\}$, $\{0,1,2,3,4\}\times\{0,1,2\}\times\{3,4\},$ $\{3,4\}\times\{0,1,2,3,4\}\times\{3,4\}$, $\{0,1\}\times\{0\}\times\{0,1,3,4\},$ $\{0,1,2,3,4\}\times\{3,4\}\times\{3,4\}$ and more.
$\bscs^3_\square(W)$ satisfies the weaker version of the general principle: it is
a thin box cover, and this cover is dominated by the box cover $\{0-2\}\times\{0-2\}\times\{0-4\}+\{0-4\}\times\{0-4\}\times\{3,4\}$, which is neither thin nor minimal. There is no single minimal box cover containing them. Moreover, the heatmap shows additional structure of vanishing that we cannot deduce immediately from the combinatorics of triplets of $\bscs$ which appear in Table \ref{f:A6}. Perhaps it is related to the multiplicities of $\bscs$ inside corresponding matrix coefficient spaces. Numerical experiments lead us to believe that either the true rank is much smaller than the box cover bounds in this case, or that there are extremely good low rank approximations. We leave the study of this hidden structure to future works.
\begin{figure}[H]
\tiny
\renewcommand\theadfont{\tiny}
    \includegraphics[width=\textwidth]%{plots/words/long_word/S4_long_AdamW.pdf}
{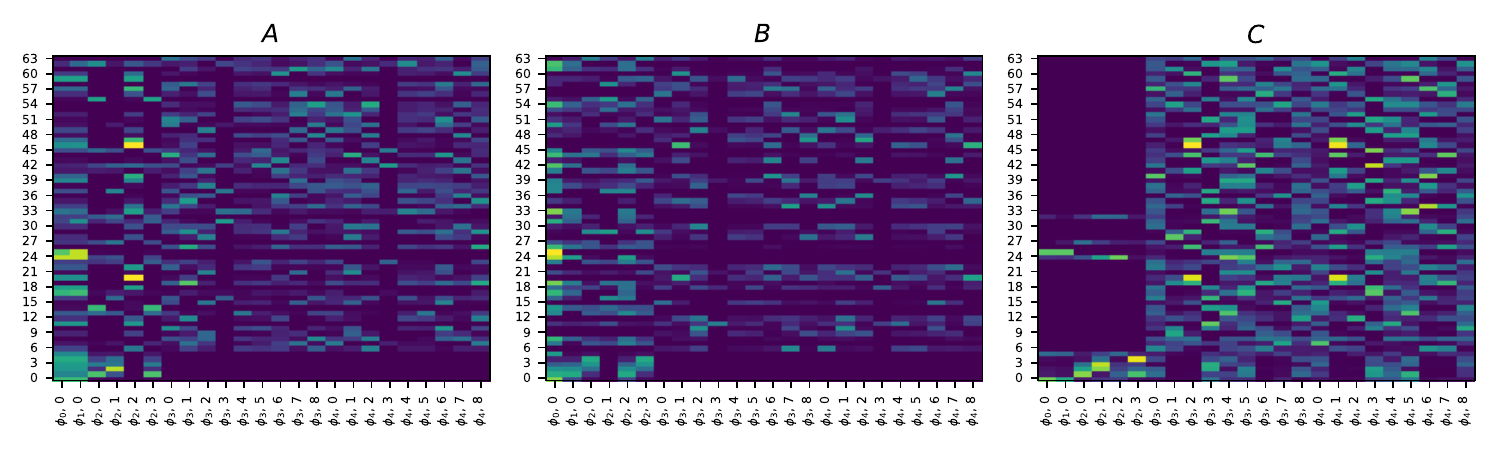}
\caption{$S_4$ and the word $aba^{-1}ba^2b^3ab^{-1}$.}
\label{fig:s4long_adamW}
\end{figure}

\textbf{$D_8:$ }
Figure \ref{fig:d8long}
shows the heatmap for $w$ and the group $D_8.$
Ignoring noise, $\bscs^3_\square(W)$ is dominated by the minimal thin box cover of Table \ref{f:A7}, which is
\[B_1=\{0,1,2,3,5\}\times\{0\}\times\{0,1,2,3,5\},~~B_2=\{4,6\}\times\{0,1,2,3\}\times\{4,6\}.\]As above, the boxes corresponding to most rows actually equal boxes from that cover.
\begin{figure}[H]
\tiny
\renewcommand\theadfont{\tiny}
    \includegraphics[width=\textwidth]{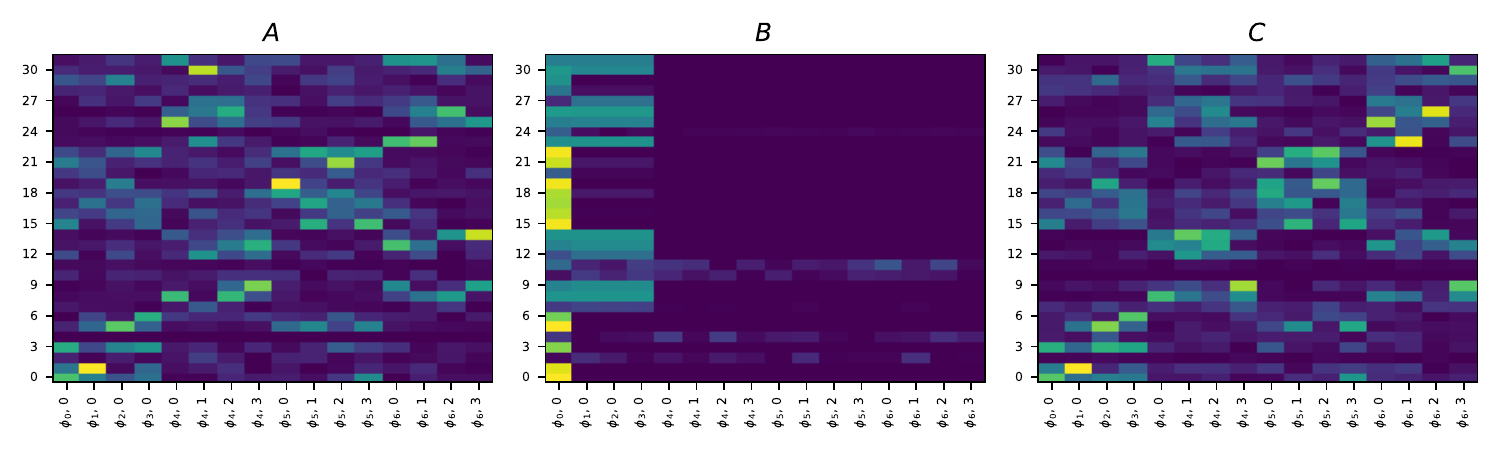}
\caption{$D_8$ and the word $aba^{-1}ba^2b^3ab^{-1}$.}
\label{fig:d8long} 
\end{figure}

\textbf{$M_5(2):$ }Figure \ref{fig:m52long}
shows the heatmap for $w$ and the group $M_5(2).$ Again $\bscs^3_\square(W)$ is dominated by the minimal box cover of of Table \ref{f:A7}, which is
\[
\begin{split}
& B_1=\{0,1,2,3,6,7\}\times\{0\}\times\{0,1,2,3,6,7\},~~B_2=\{4,5,8,9\}\times\{2\}\times\{4,5,8,9\},\\
& B_3=\{10,11\}\times\{6\}\times\{10,11\}.
\end{split}
\]
Again for most rows we have full agreement and not just domination.
\begin{figure}[H]
\tiny
\renewcommand\theadfont{\tiny}
    \includegraphics[width=\textwidth]{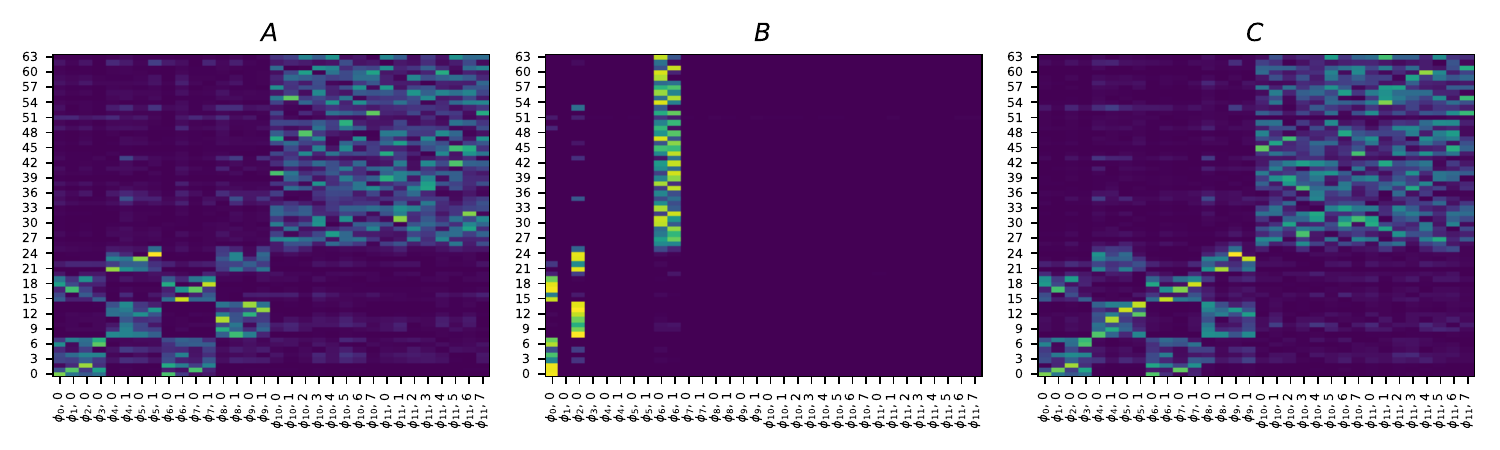}
\caption{$M_5(2)$ and the word $aba^{-1}ba^2b^3ab^{-1}$.}
\label{fig:m52long} 
\end{figure}
\end{ex}

\subsection{Additional train/test accuracy/loss evolutions for various words and groups}
\label{s:B.3}
The evolution of train/test loss/accuracy during training in one run under the HD model for various groups, words and fraction of training samples.
\begin{figure}[H]
\begin{minipage}{0.33\linewidth}
    \centering
    \includegraphics[width=1\textwidth]{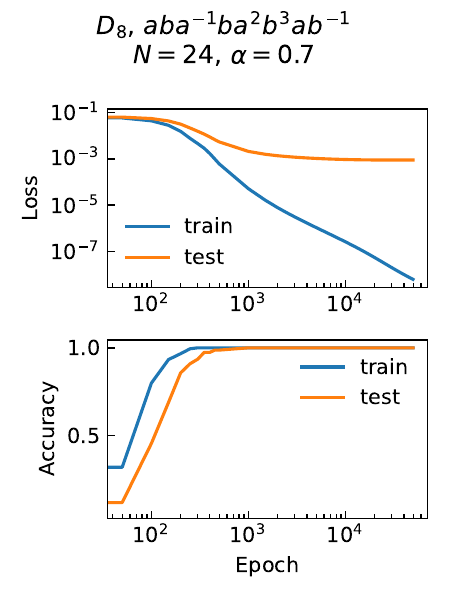} 
\end{minipage}
\begin{minipage}{0.32\linewidth}
    \centering
    \includegraphics[width=1.05\textwidth]{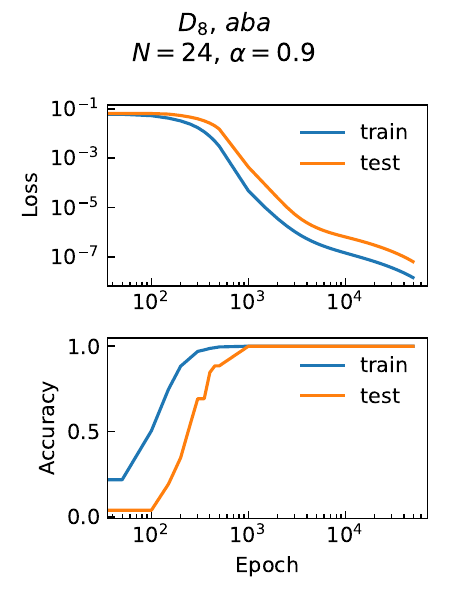} 
\end{minipage}
\begin{minipage}{0.33\linewidth}
    \centering
    \includegraphics[width=1\textwidth]{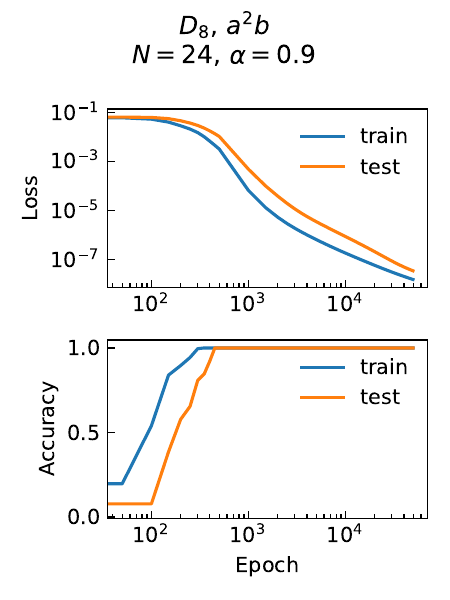} 
\end{minipage}

\begin{minipage}{0.33\linewidth}
    \centering
    \includegraphics[width=1\textwidth]{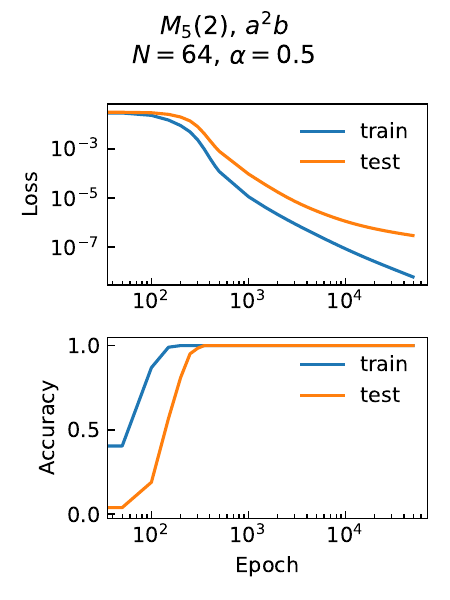} 
\end{minipage}
\begin{minipage}{0.32\linewidth}
    \centering
    \includegraphics[width=1.05\textwidth]{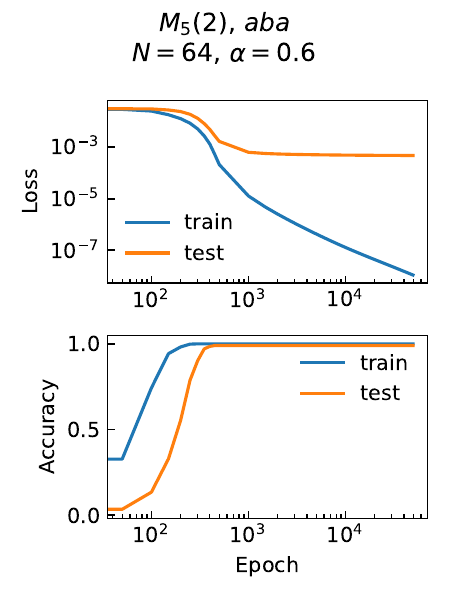} 
\end{minipage}

\begin{minipage}{0.33\linewidth}
    \centering
    \includegraphics[width=1\textwidth]{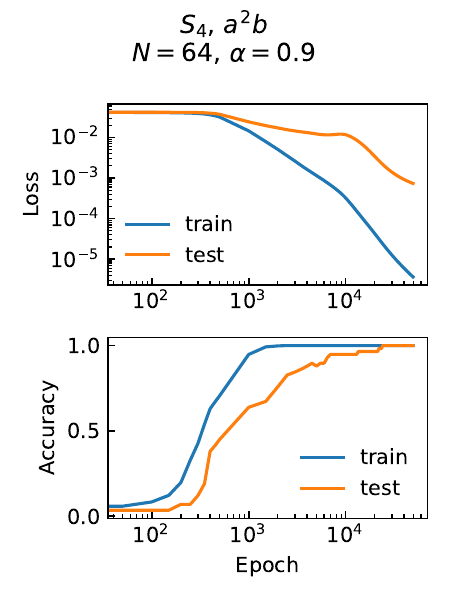} 
\end{minipage}
\begin{minipage}{0.33\linewidth}
    \centering
    \includegraphics[width=1\textwidth]{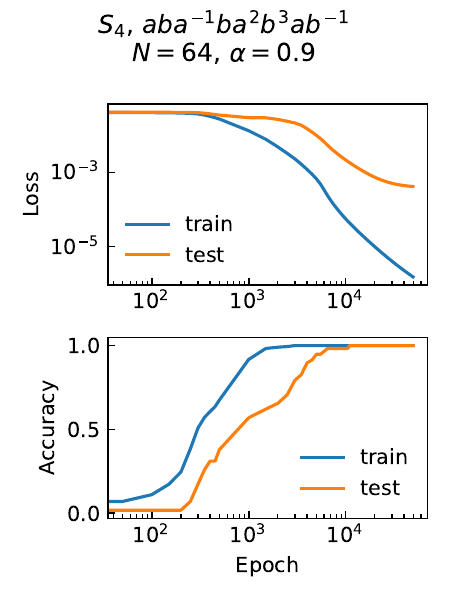} 
\end{minipage}
\end{figure}

\subsection{Additional heatmaps for the terminal weights of the TLP model}
\label{s:B.2}
Below are heatmaps of the rows of matrices $A$, $B$ and $C$ of the terminal weight configuration of one run of the TLP model with the ReLU activation function, for several groups and words. As in the case of the HD model, rows are projected onto the subspaces of the bscs of the group. The width of the model can be read off the maps. The non-trivial block structure is apparent, albeit with more noise compare to the heatmaps in the case of the HD model, as shiown in Subsection~\ref{s:B.1}. A careful examination of the bscs appearing in the support of the rows, suggest that an analogous principle to Suggested General Principle~\ref{gp:1} holds albeit under a reformulation of the notions of a box-cover and minimal box-cover from Section~\ref{s:4}. See also Subsection~\ref{s:7.3}.

\begin{figure}[H]
\tiny
\renewcommand\theadfont{\tiny}
    \includegraphics[width=\textwidth]{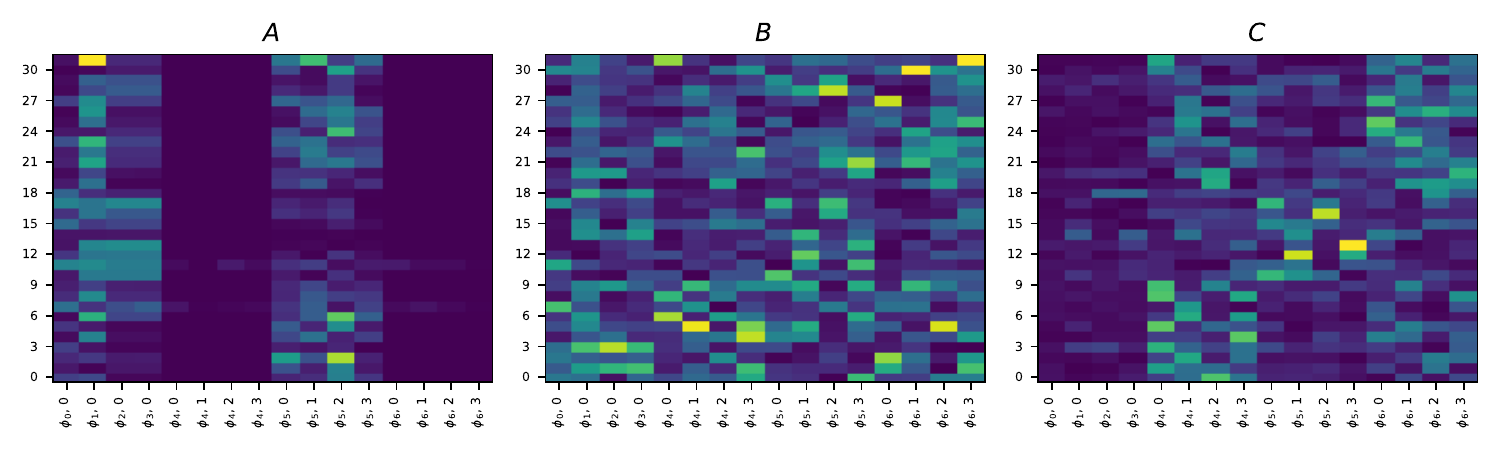}
\caption{$D_8$ and the word $a^2b$ with ReLU.}
\label{fig:d8_aab_relu} 
\end{figure}

\begin{figure}[H]
\tiny
\renewcommand\theadfont{\tiny}
    \includegraphics[width=\textwidth]{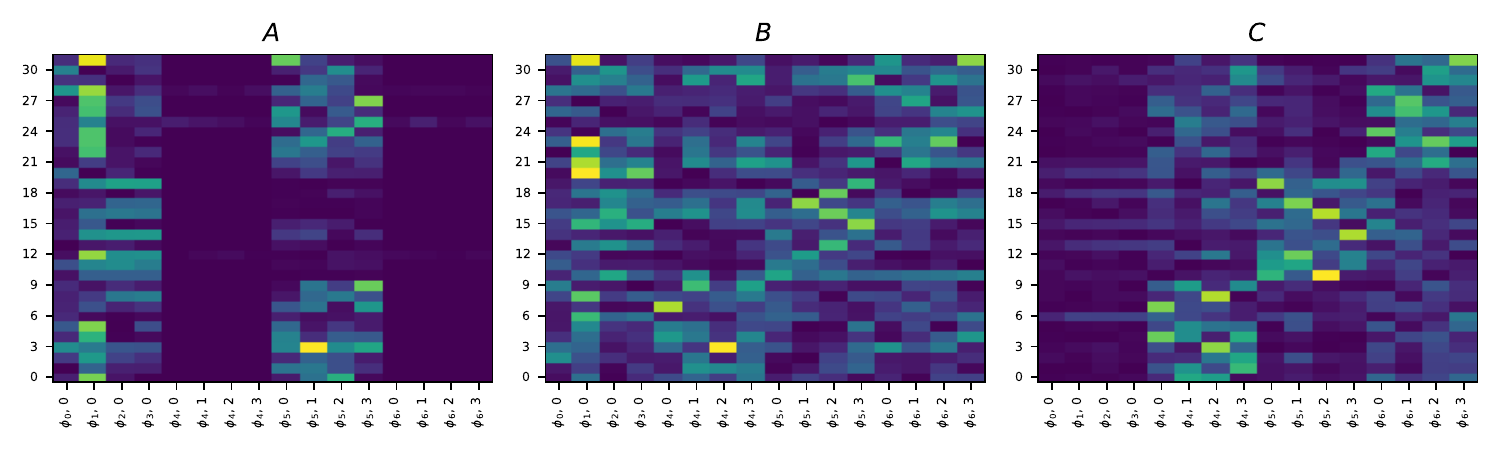}
\caption{$D_8$ and the word $aba$ with ReLU.}
\label{fig:d8_aba_relu} 
\end{figure}

\begin{figure}
\tiny
\renewcommand\theadfont{\tiny}
    \includegraphics[width=\textwidth]{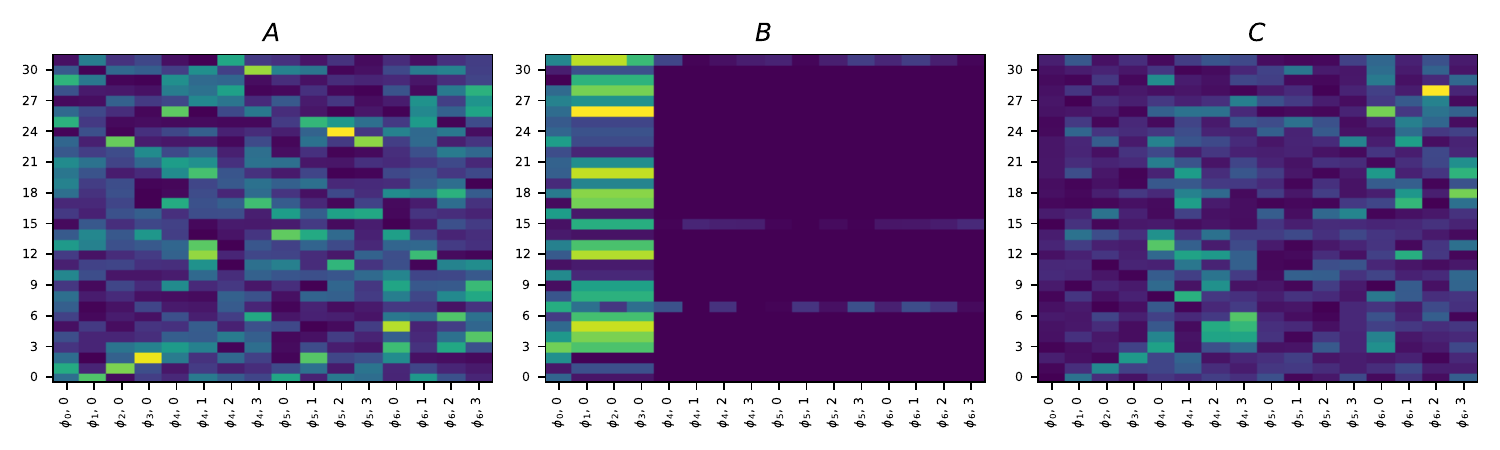}
\caption{$D_8$ and the word $aba^{-1}ba^2b^3ab^{-1}$ with ReLU.}
\label{fig:d8_long_relu} 
\end{figure}

\begin{figure}
\tiny
\renewcommand\theadfont{\tiny}
    \includegraphics[width=\textwidth]{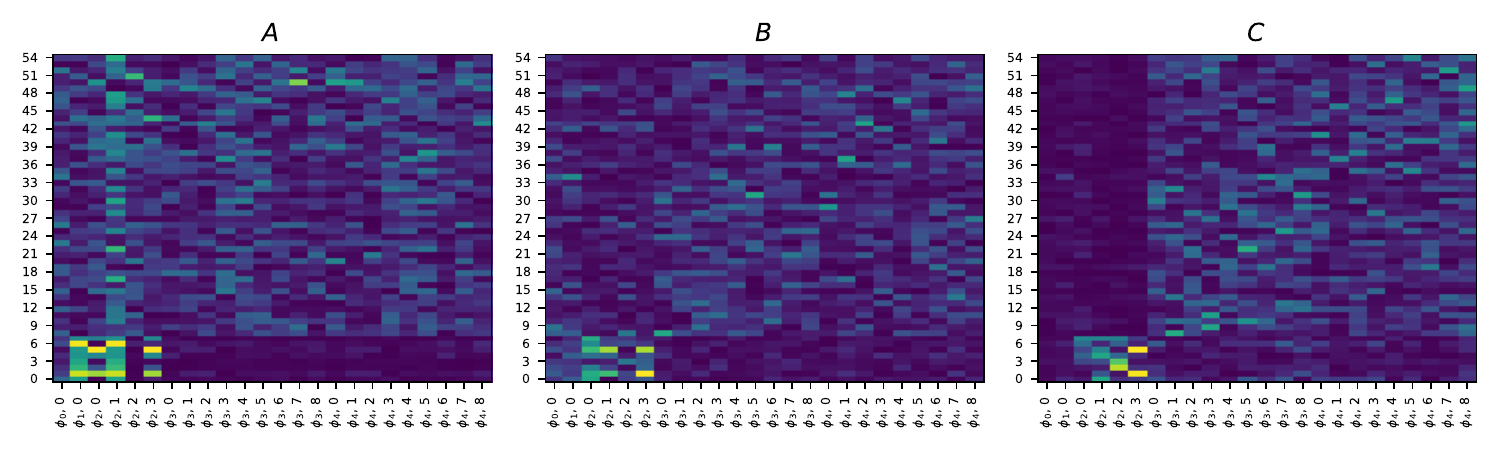}
\caption{$S_4$ and the word $a^2b$ with ReLU.}
\label{fig:s4_aab_relu} 
\end{figure}

\begin{figure}
\tiny
\renewcommand\theadfont{\tiny}
    \includegraphics[width=\textwidth]{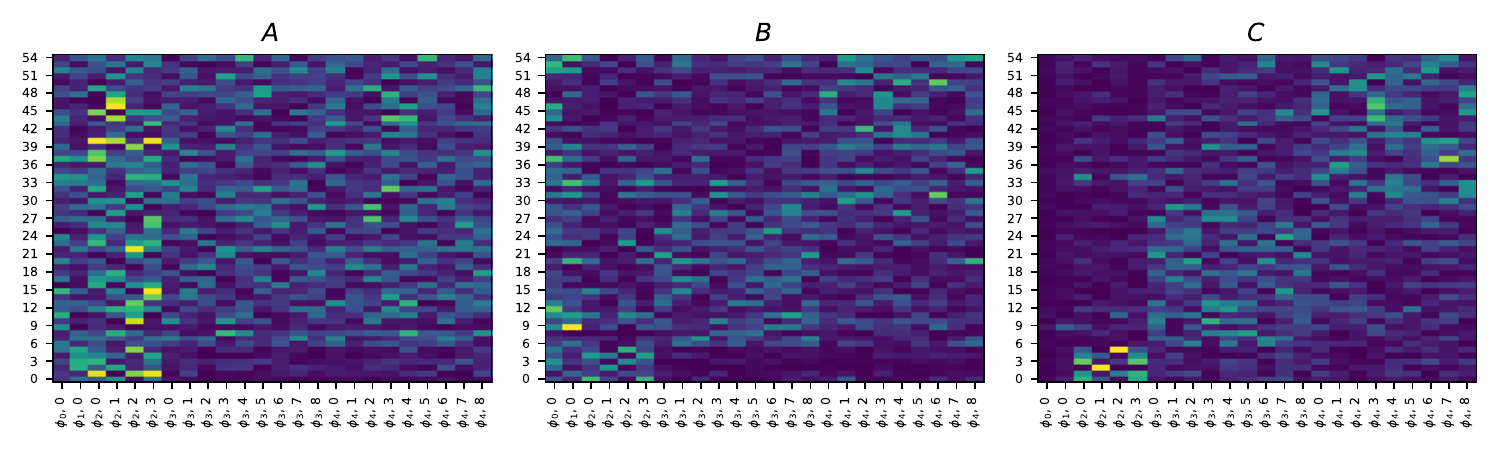}
\caption{$S_4$ and the word $aba$ with ReLU.}
\label{fig:s4_aba_relu} 
\end{figure}

\begin{figure}
\tiny
\renewcommand\theadfont{\tiny}
    \includegraphics[width=\textwidth]{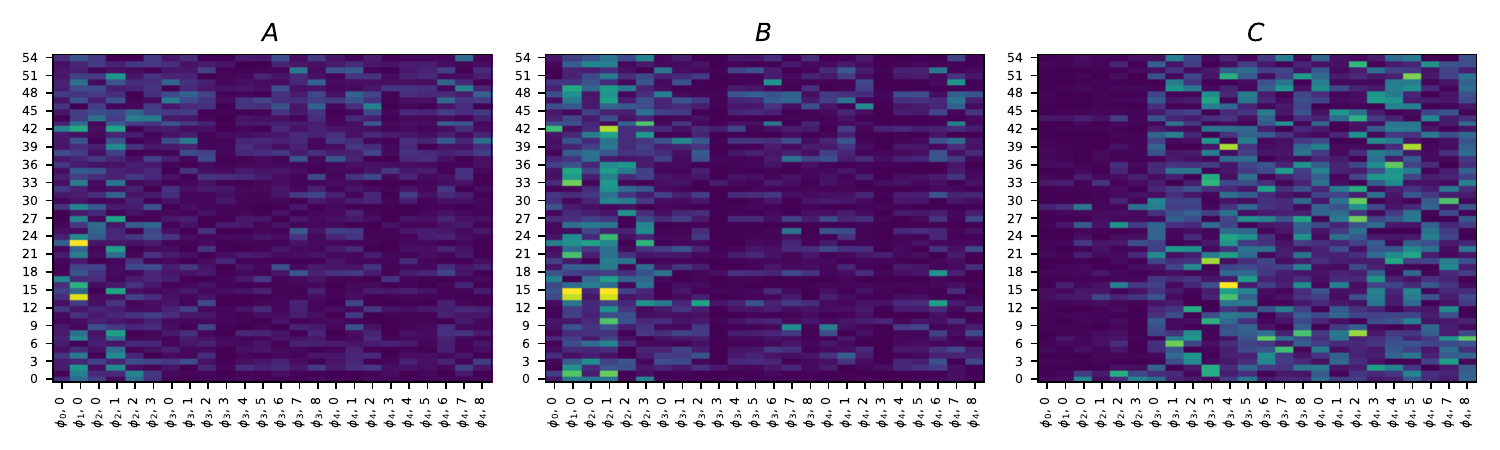}
\caption{$S_4$ and the word $aba^{-1}ba^2b^3ab^{-1}$ with ReLU.}
\label{fig:s4_long_relu} 
\end{figure}

\begin{figure}
\tiny
\renewcommand\theadfont{\tiny}
    \includegraphics[width=\textwidth]{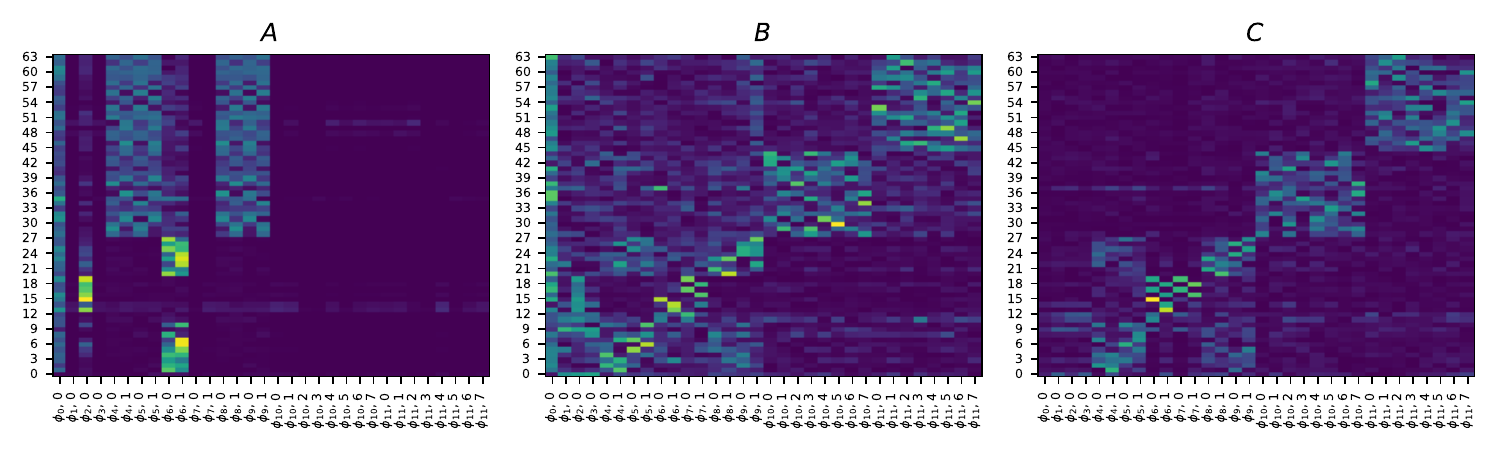}
\caption{$M_5(2)$ and the word $a^2b$ with ReLU.}
\label{fig:m52_aab_relu} 
\end{figure}

\begin{figure}
\tiny
\renewcommand\theadfont{\tiny}
    \includegraphics[width=\textwidth]{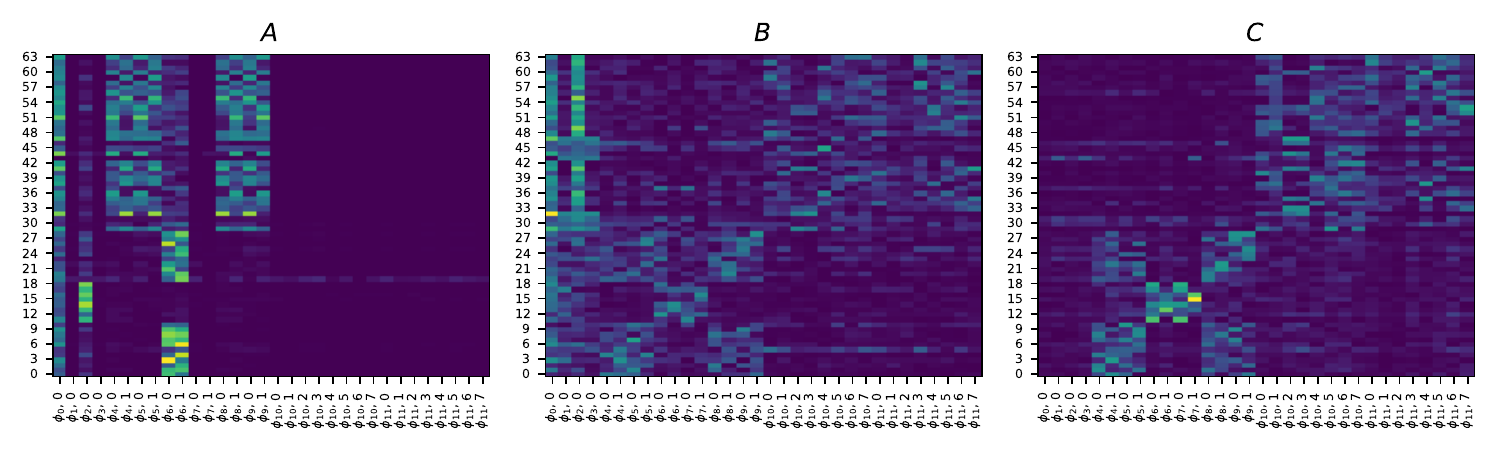}
\caption{$M_5(2)$ and the word $aba$ with ReLU.}
\label{fig:m52_aba_relu} 
\end{figure}

\begin{figure}
\tiny
\renewcommand\theadfont{\tiny}
    \includegraphics[width=\textwidth]{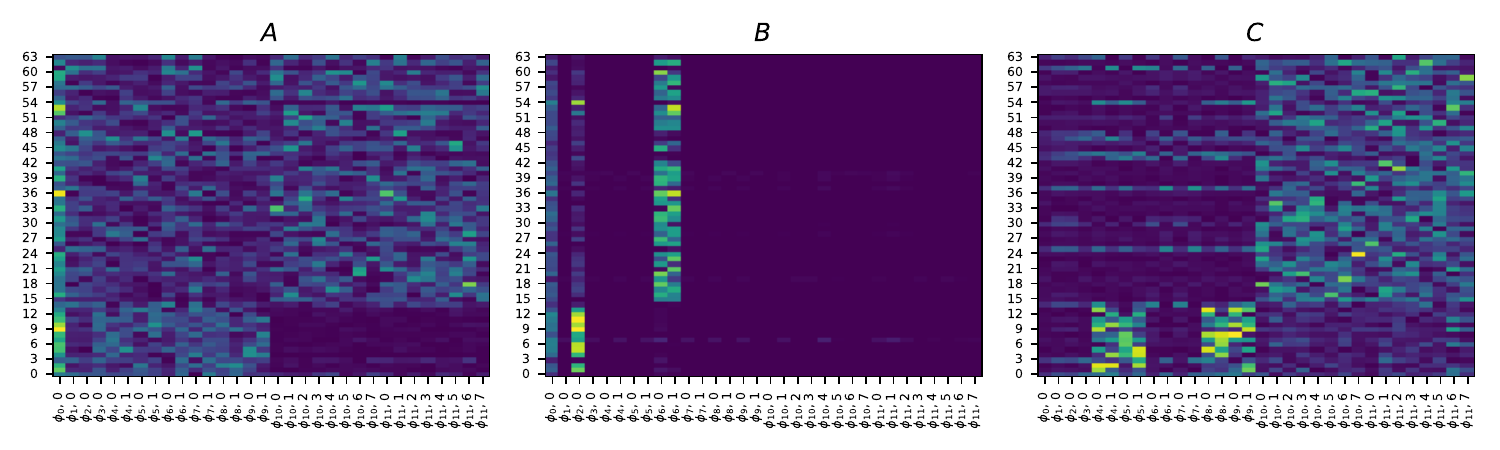}
\caption{$M_5(2)$ and the word $aba^{-1}ba^2b^3ab^{-1}$ with ReLU.}
\label{fig:m52_long_relu} 
\end{figure}

\newpage
\section{Additional material for Section~\ref{s:5}}
In this section of the appendix, we provide additional empirical results concerning the case of the simple multiplication word. 
\subsection{Terminal loss and accuracy under the TLP model with full datasets as train samples}
{\em Left:} Training results on various groups. 20 runs per group and activation function, using the AdamW optimizer without weight decay, batch size of 16, and 5000 epochs. 
{\em Right:} Projection (in absolute value) of the terminal weights of the rows of matrices $A$, $B$ and $C$ (Y-axis) on the matrix entries of all bscs of the group as $\bbR^G$-vectors (X-axis; entries of the same bsc are adjacent to each other) in one run of training for ($\bbZ_{56}^\times$, square), 
($S_4$, square), 
($S_4$, Relu) (top to bottom). The resulting blocks correspond exactly to the different bscs of the group.

\begin{figure}[H]
\begin{minipage}{0.5\linewidth}
\resizebox{\textwidth}{!}{
\begin{tabular}{|c|c|c|c|c|c|c|c|}
	\hline
	Group & $N$ & Activation & \thead{Learning \\ rate} &  \thead{Weight \\ init. STD} & \thead{Min. final \\ accuracy} & \thead{Median final\\accuracy} & \thead{Max. final \\ loss}  \tabularnewline 
	\hline
%	\hline
%	& 32 & ReLU & 0.001 & 0.18 & 0.972656 & 0.997559 & 0.02\tabularnewline
	%\hline
%	$\mathbb{Z}_{32}$ & 32 & sigmoid & 0.005 & 0.18 & 1 & 1 & 0.019\tabularnewline
	%\hline
%	 & 32 & square & 0.001 & 0.18 & 1 & 1 & 0.013\tabularnewline
%	\hline
%	& 32 & ReLU & 0.001 & 0.18 & 0.999023 & 1 & 0.017\tabularnewline
	%\hline
%	$\mathbb{Z}^{5}_{2}$ & 32 & sigmoid & 0.005 & 0.18 & 1 & 1 & 0.017\tabularnewline
	%\hline
%	 & 32 & square & 0.001 & 0.18 & 1 & 1 & 0.0049\tabularnewline
%	\hline
	& 24 & ReLU & 0.001 & 0.2 & 0.947917 & 0.998264 & 0.025\tabularnewline
		%\hline
	$\mathbb{Z}_{56}^{\times}$ & 24 & sigmoid & 0.005 & 0.2 & 0.972222 & 1 & 0.024\tabularnewline
	%\hline
	 & 24 & square & 0.001 & 0.2 & 1 & 1 & 0.013\tabularnewline
	\hline
	 & 48 & ReLU & 0.001 & 0.14 & 0.973765 & 0.98968 & 0.011\tabularnewline
	%\hline
	$\mathbb{Z}_{91}^{\times}$& 48 & sigmoid & 0.005 & 1 & 1 & 1 & 0.01\tabularnewline
	%\hline
	& 48 & square & 0.001 & 0.14 & 1 & 1 & 0.0085\tabularnewline
	\hline
	& 32 & ReLU & 0.001 & 0.18 & 0.972656 & 0.988281 & 0.022\tabularnewline
	%\hline
	$D_{16}$ & 32 & sigmoid & 0.005 & 0.18 & 0.990234 & 1 & 0.02\tabularnewline
	%\hline
	 & 32 & square & 0.001 & 0.18 & 1 & 1 & 0.014\tabularnewline
	\hline
	 & 32 & ReLU & 0.001 & 0.18 & 0.921875 & 0.983507 & 0.025\tabularnewline
	%\hline
	$S_{4}$ & 32 & sigmoid & 0.005 & 0.18 & 0.984375 & 0.998264 & 0.023\tabularnewline
	%\hline
	& 32 & square & 0.001 & 0.18 & 0.980903 & 1 & 0.02\tabularnewline
	\hline
	& 90 & ReLU & 0.001 & 0.11 & 0.996944 & 0.999583 & 0.011\tabularnewline
	%\hline
	$A_{5}$ & 90 & sigmoid & 0.005 & 1 & 0.998889 & 1 & 0.011\tabularnewline
	%\hline
	 & 90 & square & 0.001 & 0.11 & 0.999722 & 1 & 0.01\tabularnewline
	\hline
	 & 32 & ReLU & 0.001 & 0.18 & 1 & 1 & 0.013\tabularnewline
	%\hline
	$(\mathbb{Z}_{4}\times\mathbb{Z}_{2})$ & 32 & sigmoid & 0.005 & 0.18 & 1 & 1 & 0.016\tabularnewline
	%\hline
	 $\rtimes\mathbb{Z}_{2}$ & 32 & square & 0.001 & 0.18 & 1 & 1 & 1.3e-05\tabularnewline
	\hline
	 & 16 & ReLU & 0.005 & 0.25 & 0.96875 & 1 & 0.027\tabularnewline
	%\hline
	$Q_{8}$ & 16 & sigmoid & 0.005 & 0.25 & 1 & 1 & 0.017\tabularnewline
	%\hline
	 & 12 & square & 0.001 & 0.29 & 1 & 1 & 0.016\tabularnewline
	\hline
	 & 48 & ReLU & 0.001 & 0.14 & 0.992188 & 0.998047 & 0.016\tabularnewline
	%\hline
	$M_{5}(2)$ & 48 & sigmoid & 0.005 & 1 & 0.963867 & 0.998047 & 0.016\tabularnewline
	%\hline
	 & 48 & square & 0.001 & 0.14 & 1 & 1 & 0.0089\tabularnewline
	\hline
\end{tabular}}
\end{minipage}\hfill
\begin{minipage}{0.5\linewidth}
    \includegraphics[width=\textwidth]{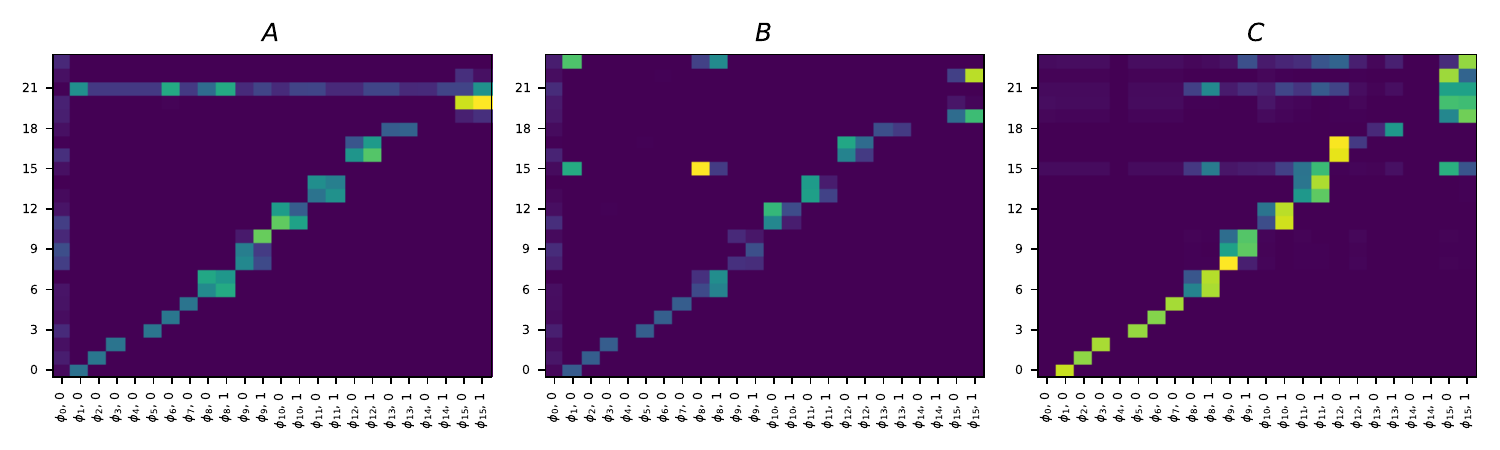} 
    \includegraphics[width=\textwidth]{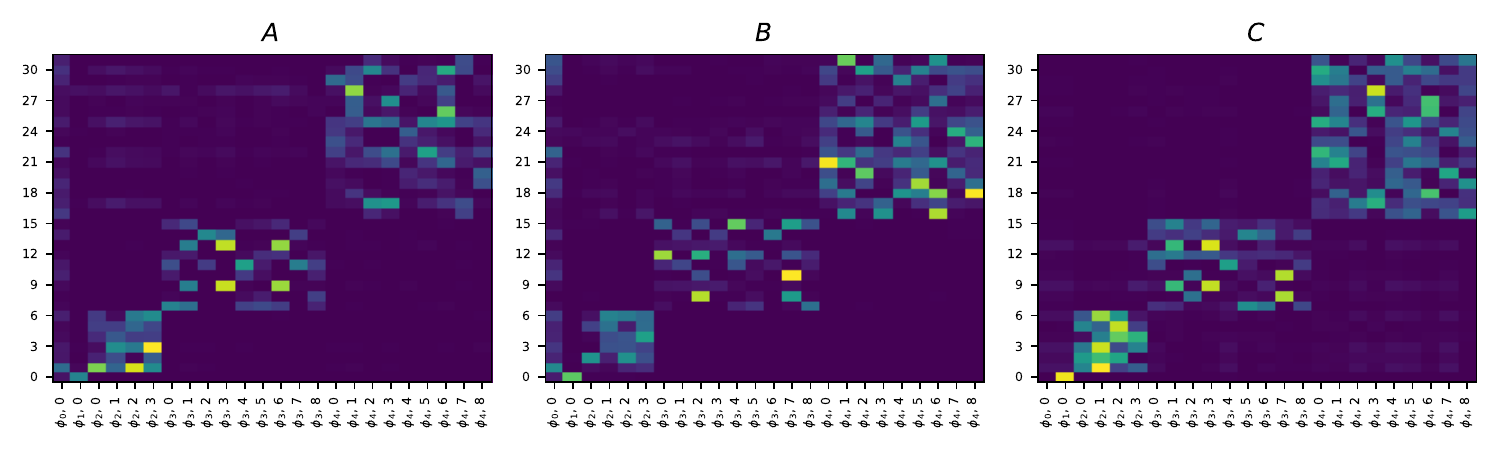} \\
    \includegraphics[width=\textwidth]{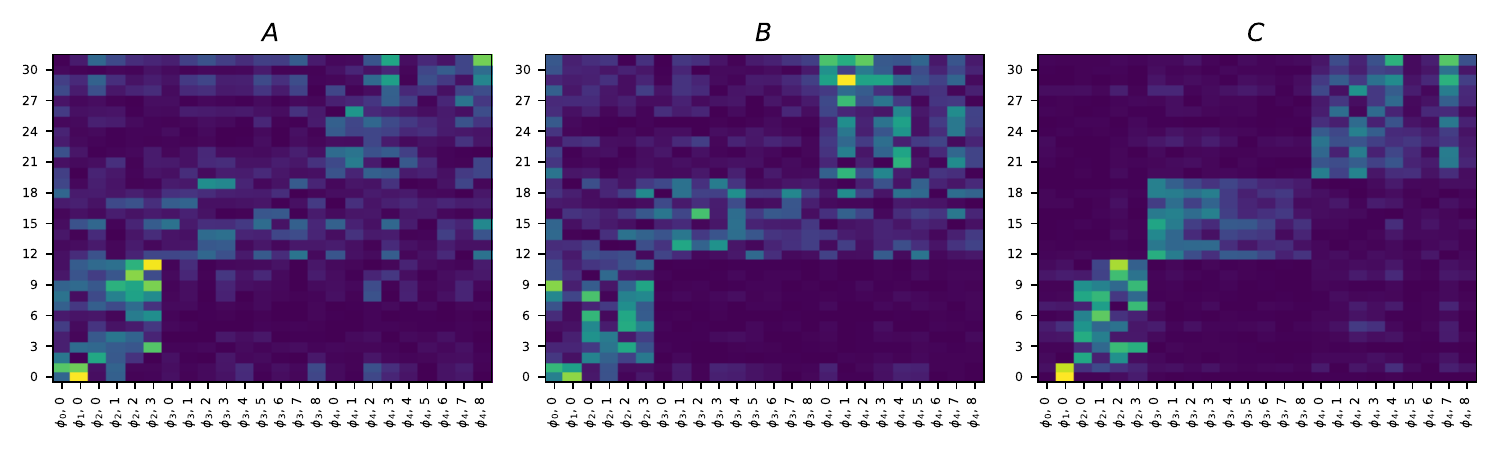} 
\end{minipage}	
\vspace{0.3cm}
\end{figure}

\subsection{Additional heatmaps for the terminal weights under the HD model}
Terminal weight configuration under the HD model for the simple multiplication word. Full dataset was used as train set. Rows are projected onto the subspaces of the bscs of the group. 
\begin{figure}[H]
    \centering
    \includegraphics[width=1\textwidth]{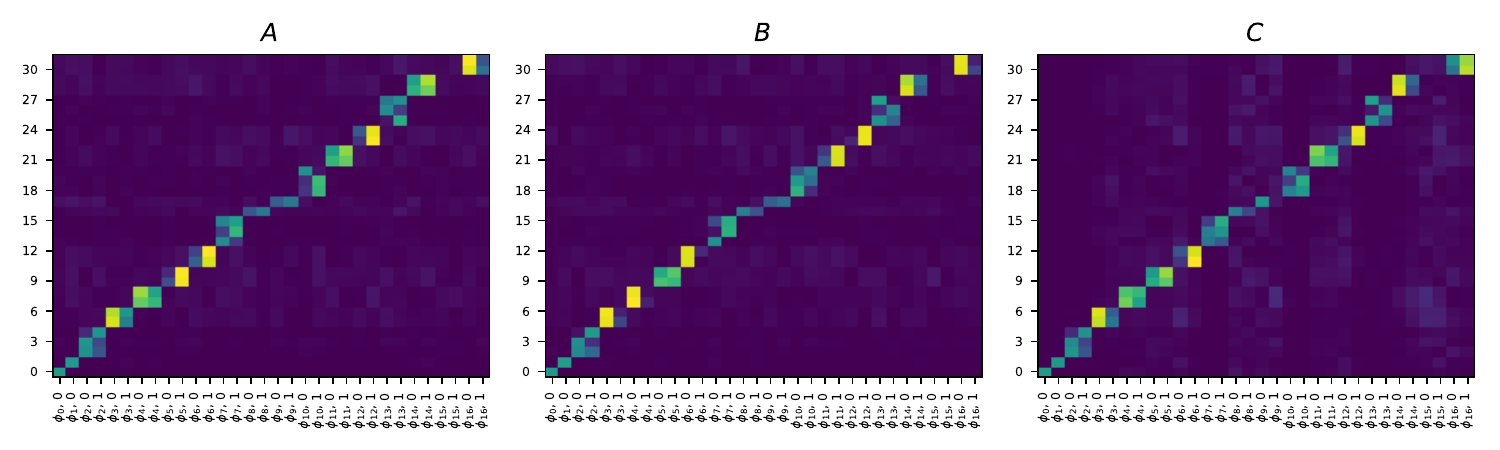}
    %{plots/general_abelian_32__N32_alph0.4_prod_ep50000_2025-02-26_05-25-21_mZC8LANvlg/abc_in_irrep_basis.pdf}     
    \caption{$\mathbb{Z}_{32}$.}
\end{figure}

\begin{figure}[H]
    \centering
    \includegraphics[width=1\textwidth]{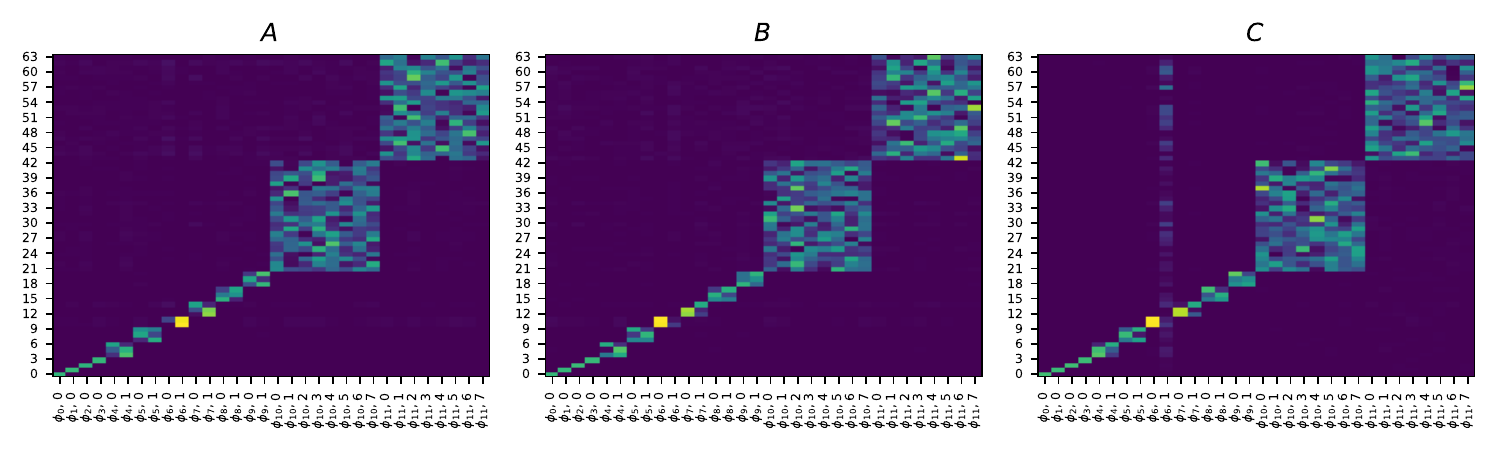} 
    \caption{$M_{5}(2)$.}
\end{figure}

\subsection{Terminal accuracy under the HD model with various widths and fraction of samples}
Final accuracy, for different groups, widths $N$ and train fractions, as average over $20$ runs of GD for the HD-model 
starting from a random initialization and using a random train-test split.
    Error bars mark one standard deviation.
\begin{figure}[H]
    \centering
    \includegraphics[width=1\textwidth]{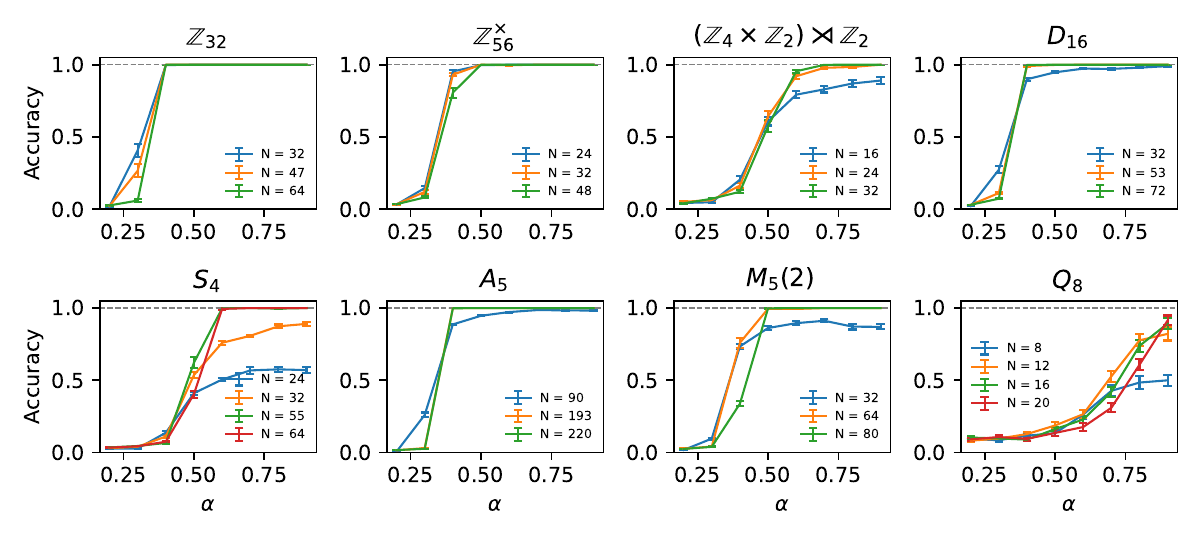} 
\end{figure}

\subsection{Additional train/test accuracy/loss evolution for various groups}
The evolution of train/test loss/accuracy during training in one run under the HD model for various groups and fraction of train samples.
\begin{figure}[H]
\begin{minipage}{0.33\linewidth}
    \centering
    \includegraphics[width=1\textwidth]{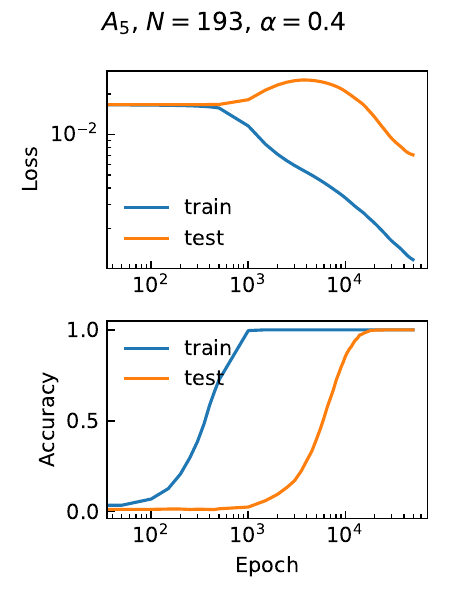} 
\end{minipage}
\begin{minipage}{0.32\linewidth}
    \centering
    \includegraphics[width=1.05\textwidth]{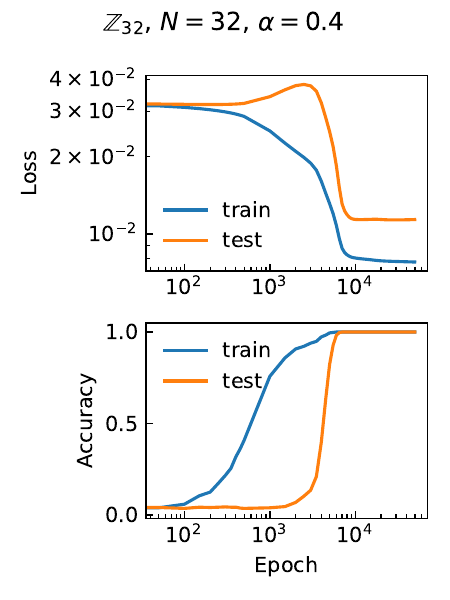} 
\end{minipage}
\begin{minipage}{0.33\linewidth}
    \centering
    \includegraphics[width=1\textwidth]{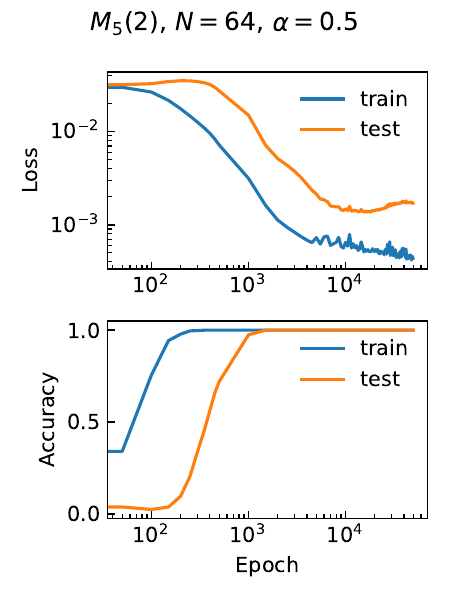} 
\end{minipage}
\end{figure}

\subsection{Loss decomposition for various groups and model widths}
Median terminal accuracy, total loss, bsc-loss and number of rows per bsc across 20 runs for the HD model with widths and groups.
\begin{table}[H]
%\begin{table}
\setlength\tabcolsep{6pt}
\begin{tabular}{|c|c|c|c|c|c|c|c|}
\hline
 $N$ & \thead{Final \\ accuracy} & \thead{Final \\ loss} & \thead{ $\phi_{0}$\\I, $d = 1$}&\thead{ $\phi_{1}$\\I, $d = 1$}&\thead{ $\phi_{2}$\\I, $d = 2$}&\thead{ $\phi_{3}$\\I, $d = 3$}&\thead{ $\phi_{4}$\\I, $d = 3$}\tabularnewline 
\hline 
$24$ &\small{$0.99045$} & \small{$0.023$} & \small{$1.3 \times 10^{-31}$} & \small{$1.2 \times 10^{-31}$} & \small{$0.0026$} & \small{$0.01$} & \small{$0.0093$}\tabularnewline
 & &  & \small{$1$} & \small{$1$} & \small{$5$} & \small{$9$} & \small{$10$}\tabularnewline 
\hline 
$32$ &\small{$1$} & \small{$0.016$} & \small{$1.2 \times 10^{-31}$} & \small{$1.2 \times 10^{-31}$} & \small{$0.00043$} & \small{$0.0072$} & \small{$0.0081$}\tabularnewline
 & &  & \small{$1$} & \small{$1$} & \small{$6$} & \small{$12$} & \small{$11$}\tabularnewline 
\hline 
$55$ &\small{$1$} & \small{$0.002$} & \small{$1.1 \times 10^{-31}$} & \small{$1.1 \times 10^{-31}$} & \small{$1.3 \times 10^{-9}$} & \small{$4.5 \times 10^{-6}$} & \small{$0.00088$}\tabularnewline
 & &  & \small{$1$} & \small{$1$} & \small{$8$} & \small{$23$} & \small{$22$}\tabularnewline 
\hline 
$64$ &\small{$1$} & \small{$3 \times 10^{-7}$} & \small{$5 \times 10^{-22}$} & \small{$5.2 \times 10^{-18}$} & \small{$8.9 \times 10^{-10}$} & \small{$1.6 \times 10^{-7}$} & \small{$1 \times 10^{-7}$}\tabularnewline
 & &  & \small{$1$} & \small{$1$} & \small{$9$} & \small{$27$} & \small{$26$}\tabularnewline 
\hline 

\end{tabular}\caption{$S_{4}$.}
\end{table}

\begin{table}[H]
\setlength\tabcolsep{6pt}
\begin{tabular}{|c|c|c|c|c|c|c|c|}
\hline
 $N$ & \thead{Final \\ accuracy} & \thead{Final \\ loss} & \thead{ $\phi_{0}$\\I, $d = 1$}&\thead{ $\phi_{1}$\\I, $d = 3$}&\thead{ $\phi_{2}$\\I, $d = 3$}&\thead{ $\phi_{3}$\\I, $d = 4$}&\thead{ $\phi_{4}$\\I, $d = 5$}\tabularnewline 
\hline 
$90$ &\small{$1$} & \small{$0.0093$} & \small{$2.5 \times 10^{-32}$} & \small{$0.00084$} & \small{$0.00083$} & \small{$0.0026$} & \small{$0.005$}\tabularnewline
 & &  & \small{$1$} & \small{$15$} & \small{$15$} & \small{$25$} & \small{$34$}\tabularnewline 
\hline 
$193$ &\small{$1$} & \small{$0.0017$} & \small{$2.4 \times 10^{-32}$} & \small{$1.8 \times 10^{-8}$} & \small{$1.9 \times 10^{-8}$} & \small{$7.6 \times 10^{-5}$} & \small{$0.0015$}\tabularnewline
 & &  & \small{$1$} & \small{$28$} & \small{$27$} & \small{$54$} & \small{$83$}\tabularnewline 
\hline 
$220$ &\small{$1$} & \small{$0.00054$} & \small{$2.5 \times 10^{-32}$} & \small{$2.2 \times 10^{-8}$} & \small{$2.3 \times 10^{-8}$} & \small{$1.5 \times 10^{-7}$} & \small{$0.00054$}\tabularnewline
 & &  & \small{$1$} & \small{$30$} & \small{$30$} & \small{$61$} & \small{$98$}\tabularnewline 
\hline 
$270$ &\small{$1$} & \small{$3.8 \times 10^{-7}$} & \small{$3.8 \times 10^{-15}$} & \small{$1.5 \times 10^{-8}$} & \small{$1.2 \times 10^{-8}$} & \small{$1.5 \times 10^{-7}$} & \small{$2 \times 10^{-7}$}\tabularnewline
 & &  & \small{$1$} & \small{$33$} & \small{$34$} & \small{$72$} & \small{$127$}\tabularnewline 
\hline 

\end{tabular}\caption{$A_{5}$.}
\end{table}

\begin{table}[H]
\resizebox{0.95\linewidth}{!}{
\setlength\tabcolsep{2pt}
\begin{tabular}{|c|c|c|c|c|c|c|c|c|c|c|c|c|c|c|c|c|c|c|}
\hline
 $N$ & \thead{Final \\ accuracy} & \thead{Final \\ loss} & \thead{ $\phi_{0}$\\I, $d = 1$}&\thead{ $\phi_{1}$\\I, $d = 1$}&\thead{ $\phi_{2}$\\I, $d = 1$}&\thead{ $\phi_{3}$\\I, $d = 1$}&\thead{ $\phi_{4}$\\I, $d = 1$}&\thead{ $\phi_{5}$\\I, $d = 1$}&\thead{ $\phi_{6}$\\I, $d = 1$}&\thead{ $\phi_{7}$\\I, $d = 1$}&\thead{ $\phi_{8}$\\II, $d = 2$}&\thead{ $\phi_{9}$\\II, $d = 2$}&\thead{ $\phi_{10}$\\II, $d = 2$}&\thead{ $\phi_{11}$\\II, $d = 2$}&\thead{ $\phi_{12}$\\II, $d = 2$}&\thead{ $\phi_{13}$\\II, $d = 2$}&\thead{ $\phi_{14}$\\II, $d = 2$}&\thead{ $\phi_{15}$\\II, $d = 2$}\tabularnewline 
\hline 
$24$ &\small{$1$} & \small{$0.01$} & \small{$1.4 \times 10^{-31}$} & \small{$1.5 \times 10^{-31}$} & \small{$1.4 \times 10^{-31}$} & \small{$1.7 \times 10^{-31}$} & \small{$1.4 \times 10^{-31}$} & \small{$1.9 \times 10^{-31}$} & \small{$1.4 \times 10^{-31}$} & \small{$1.5 \times 10^{-31}$} & \small{$0.00087$} & \small{$0.00087$} & \small{$0.00087$} & \small{$0.00087$} & \small{$0.00087$} & \small{$0.00087$} & \small{$0.00087$} & \small{$0.00087$}\tabularnewline
 & &  & \small{$1$} & \small{$1$} & \small{$1$} & \small{$1$} & \small{$1$} & \small{$1$} & \small{$1$} & \small{$1$} & \small{$2$} & \small{$2$} & \small{$2$} & \small{$2$} & \small{$2$} & \small{$2$} & \small{$2$} & \small{$2$}\tabularnewline 
\hline 
$32$ &\small{$1$} & \small{$0.0017$} & \small{$7.6 \times 10^{-23}$} & \small{$8.7 \times 10^{-20}$} & \small{$9.5 \times 10^{-23}$} & \small{$1.1 \times 10^{-24}$} & \small{$5.2 \times 10^{-25}$} & \small{$4.7 \times 10^{-23}$} & \small{$3.3 \times 10^{-23}$} & \small{$2.9 \times 10^{-23}$} & \small{$4.1 \times 10^{-19}$} & \small{$4.9 \times 10^{-17}$} & \small{$4.6 \times 10^{-13}$} & \small{$3.4 \times 10^{-21}$} & \small{$3.7 \times 10^{-14}$} & \small{$3.6 \times 10^{-12}$} & \small{$1.1 \times 10^{-14}$} & \small{$1.4 \times 10^{-12}$}\tabularnewline
 & &  & \small{$1$} & \small{$1$} & \small{$1$} & \small{$1$} & \small{$1$} & \small{$1$} & \small{$1$} & \small{$1$} & \small{$3$} & \small{$3$} & \small{$3$} & \small{$3$} & \small{$3$} & \small{$3$} & \small{$3$} & \small{$3$}\tabularnewline 
\hline 
$48$ &\small{$1$} & \small{$2.8 \times 10^{-8}$} & \small{$3.8 \times 10^{-10}$} & \small{$4.9 \times 10^{-10}$} & \small{$2.8 \times 10^{-10}$} & \small{$3.7 \times 10^{-10}$} & \small{$2.5 \times 10^{-10}$} & \small{$2.9 \times 10^{-10}$} & \small{$4.6 \times 10^{-10}$} & \small{$4.6 \times 10^{-10}$} & \small{$2.1 \times 10^{-9}$} & \small{$1.9 \times 10^{-9}$} & \small{$2.1 \times 10^{-9}$} & \small{$1.8 \times 10^{-9}$} & \small{$1.9 \times 10^{-9}$} & \small{$1.5 \times 10^{-9}$} & \small{$2.6 \times 10^{-9}$} & \small{$1.4 \times 10^{-9}$}\tabularnewline
 & &  & \small{$1$} & \small{$1$} & \small{$1$} & \small{$2$} & \small{$1$} & \small{$2$} & \small{$1$} & \small{$1$} & \small{$4$} & \small{$4$} & \small{$4$} & \small{$4$} & \small{$5$} & \small{$4$} & \small{$5$} & \small{$4$}\tabularnewline 
\hline 
\end{tabular}}
\caption{$\mathbb{Z}_{56}^{\times}\simeq\mathbb{Z}_{2}\times\mathbb{Z}_{2}\times\mathbb{Z}_{6}$.}
\end{table}

\begin{table}[H]
\setlength\tabcolsep{6pt}
\resizebox{0.95\linewidth}{!}{
\begin{tabular}{|c|c|c|c|c|c|c|c|c|c|c|c|c|c|c|c|c|c|c|c|}
\hline
 $N$ & \thead{Final \\ accuracy} & \thead{Final \\ loss} & \thead{ $\phi_{0}$\\I, $d = 1$}&\thead{ $\phi_{1}$\\I, $d = 1$}&\thead{ $\phi_{2}$\\II, $d = 2$}&\thead{ $\phi_{3}$\\II, $d = 2$}&\thead{ $\phi_{4}$\\II, $d = 2$}&\thead{ $\phi_{5}$\\II, $d = 2$}&\thead{ $\phi_{6}$\\II, $d = 2$}&\thead{ $\phi_{7}$\\II, $d = 2$}&\thead{ $\phi_{8}$\\II, $d = 2$}&\thead{ $\phi_{9}$\\II, $d = 2$}&\thead{ $\phi_{10}$\\II, $d = 2$}&\thead{ $\phi_{11}$\\II, $d = 2$}&\thead{ $\phi_{12}$\\II, $d = 2$}&\thead{ $\phi_{13}$\\II, $d = 2$}&\thead{ $\phi_{14}$\\II, $d = 2$}&\thead{ $\phi_{15}$\\II, $d = 2$}&\thead{ $\phi_{16}$\\II, $d = 2$}\tabularnewline 
\hline 
$32$ &\small{$1$} & \small{$0.0098$} & \small{$4.8 \times 10^{-31}$} & \small{$4.7 \times 10^{-31}$} & \small{$0.00049$} & \small{$0.00049$} & \small{$0.00049$} & \small{$0.00049$} & \small{$0.00049$} & \small{$0.00049$} & \small{$0.00049$} & \small{$0.00049$} & \small{$0.00049$} & \small{$0.00049$} & \small{$0.00049$} & \small{$0.00049$} & \small{$0.00049$} & \small{$0.00049$} & \small{$0.00049$}\tabularnewline
 & &  & \small{$1$} & \small{$1$} & \small{$2$} & \small{$2$} & \small{$2$} & \small{$2$} & \small{$2$} & \small{$2$} & \small{$2$} & \small{$2$} & \small{$2$} & \small{$2$} & \small{$2$} & \small{$2$} & \small{$2$} & \small{$2$} & \small{$2$}\tabularnewline 
\hline 
$47$ &\small{$1$} & \small{$0.0015$} & \small{$1 \times 10^{-19}$} & \small{$3.3 \times 10^{-18}$} & \small{$1.9 \times 10^{-16}$} & \small{$1.5 \times 10^{-14}$} & \small{$2 \times 10^{-12}$} & \small{$1.9 \times 10^{-17}$} & \small{$3.6 \times 10^{-16}$} & \small{$2 \times 10^{-15}$} & \small{$8.3 \times 10^{-15}$} & \small{$2.2 \times 10^{-14}$} & \small{$8 \times 10^{-18}$} & \small{$9.9 \times 10^{-19}$} & \small{$3 \times 10^{-16}$} & \small{$5.8 \times 10^{-14}$} & \small{$4.9 \times 10^{-17}$} & \small{$8.3 \times 10^{-16}$} & \small{$5.2 \times 10^{-17}$}\tabularnewline
 & &  & \small{$1$} & \small{$1$} & \small{$3$} & \small{$3$} & \small{$3$} & \small{$3$} & \small{$3$} & \small{$3$} & \small{$3$} & \small{$3$} & \small{$3$} & \small{$3$} & \small{$3$} & \small{$3$} & \small{$3$} & \small{$3$} & \small{$3$}\tabularnewline 
\hline 
$64$ &\small{$1$} & \small{$5 \times 10^{-8}$} & \small{$5.3 \times 10^{-10}$} & \small{$3.5 \times 10^{-10}$} & \small{$2.4 \times 10^{-9}$} & \small{$3.4 \times 10^{-9}$} & \small{$1.3 \times 10^{-9}$} & \small{$2.4 \times 10^{-9}$} & \small{$2 \times 10^{-9}$} & \small{$2.1 \times 10^{-9}$} & \small{$2.5 \times 10^{-9}$} & \small{$3.1 \times 10^{-9}$} & \small{$1.7 \times 10^{-9}$} & \small{$2.2 \times 10^{-9}$} & \small{$3.9 \times 10^{-9}$} & \small{$2.8 \times 10^{-9}$} & \small{$2.6 \times 10^{-9}$} & \small{$2.7 \times 10^{-9}$} & \small{$2.1 \times 10^{-9}$}\tabularnewline
 & &  & \small{$1$} & \small{$1$} & \small{$4$} & \small{$4$} & \small{$4$} & \small{$4$} & \small{$4$} & \small{$4$} & \small{$4$} & \small{$4$} & \small{$4$} & \small{$4$} & \small{$4$} & \small{$4$} & \small{$4$} & \small{$4$} & \small{$4$}\tabularnewline 
\hline 
\end{tabular}}
\caption{$\mathbb{Z}_{32}$.}
\end{table}

\begin{table}[H]
\setlength\tabcolsep{2pt}
\resizebox{0.95\textwidth}{!}{
\begin{tabular}{|c|c|c|c|c|c|c|c|c|c|c|}
\hline
 $N$ & \thead{Final \\ accuracy} & \thead{Final \\ loss} & \thead{ $\phi_{0}$\\I, $d = 1$}&\thead{ $\phi_{1}$\\I, $d = 1$}&\thead{ $\phi_{2}$\\I, $d = 1$}&\thead{ $\phi_{3}$\\I, $d = 1$}&\thead{ $\phi_{4}$\\I, $d = 2$}&\thead{ $\phi_{5}$\\I, $d = 2$}&\thead{ $\phi_{6}$\\II, $d = 2$}&\thead{ $\phi_{7}$\\II, $d = 2$}\tabularnewline 
\hline 
$16$ &\small{$1$} & \small{$0.02$} & \small{$2.9 \times 10^{-32}$} & \small{$3 \times 10^{-32}$} & \small{$3.1 \times 10^{-32}$} & \small{$3.5 \times 10^{-32}$} & \small{$0.0088$} & \small{$0.0043$} & \small{$0.002$} & \small{$0.002$}\tabularnewline
 & &  & \small{$1$} & \small{$1$} & \small{$1$} & \small{$1$} & \small{$3$} & \small{$5$} & \small{$2$} & \small{$2$}\tabularnewline 
\hline 
$24$ &\small{$1$} & \small{$0.002$} & \small{$3.3 \times 10^{-32}$} & \small{$3.2 \times 10^{-32}$} & \small{$3.1 \times 10^{-32}$} & \small{$3 \times 10^{-32}$} & \small{$3.5 \times 10^{-18}$} & \small{$1.2 \times 10^{-13}$} & \small{$2.8 \times 10^{-31}$} & \small{$5.1 \times 10^{-31}$}\tabularnewline
 & &  & \small{$1$} & \small{$1$} & \small{$1$} & \small{$1$} & \small{$7$} & \small{$7$} & \small{$3$} & \small{$3$}\tabularnewline 
\hline 
$32$ &\small{$1$} & \small{$6.3 \times 10^{-9}$} & \small{$1.5 \times 10^{-11}$} & \small{$1.5 \times 10^{-11}$} & \small{$2.1 \times 10^{-11}$} & \small{$1.4 \times 10^{-10}$} & \small{$1.2 \times 10^{-9}$} & \small{$1.7 \times 10^{-9}$} & \small{$3 \times 10^{-10}$} & \small{$3.2 \times 10^{-10}$}\tabularnewline
 & &  & \small{$1$} & \small{$1$} & \small{$1$} & \small{$1$} & \small{$9$} & \small{$9$} & \small{$4$} & \small{$3$}\tabularnewline 
\hline 
\end{tabular}
}
\caption{$(\mathbb{Z}_{4}\times\mathbb{Z}_{2})\rtimes\mathbb{Z}_{2}$. }
\end{table}

\subsection{Single-bsc dynamics}
Terminal bsc-loss in repeated (20-100) runs of the model for various groups and bscs as a function of the width of the network, with initial weights chosen randomly from $R_\phi$. The minimal loss is marked with a blue diamond. 
\begin{figure}[H]
\tiny
\renewcommand\theadfont{\tiny}
\begin{subfigure}{0.33\linewidth}
\includegraphics[width=\textwidth]{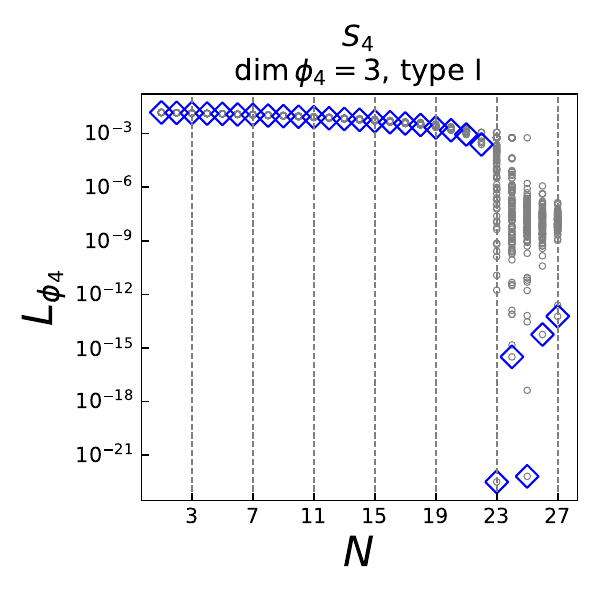} 
\end{subfigure}
\begin{subfigure}{0.33\linewidth}
\includegraphics[width=\textwidth]{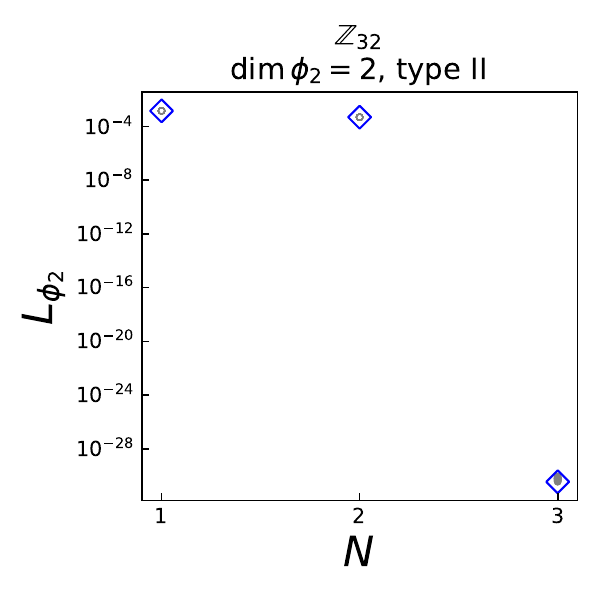} 
\end{subfigure}
\begin{subfigure}{0.33\linewidth}
\includegraphics[width=\textwidth]{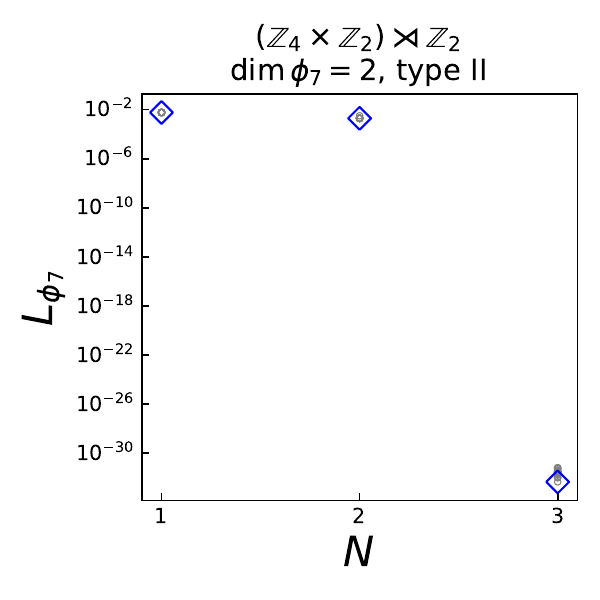} 
\end{subfigure}
\end{figure}

\end{document}